\documentclass[nohyperref]{article}

\usepackage{microtype}
\usepackage{graphicx}
\usepackage{booktabs} 
\usepackage{multirow}
\usepackage{subfig}
\usepackage{tablefootnote}
\usepackage{algorithm}
\usepackage{algorithmic}
\usepackage[dvipsnames]{xcolor}
\usepackage{xurl}

\definecolor{darkgreen}{rgb}{0.0, 0.25, 0.08}

\usepackage{hyperref}

\usepackage[accepted]{icml2022}

\usepackage{amsmath}
\usepackage{amssymb}
\usepackage{mathtools}
\usepackage{amsthm}

\usepackage[capitalize,noabbrev]{cleveref}

\theoremstyle{plain}
\newtheorem{theorem}{Theorem}[section]
\newtheorem{proposition}[theorem]{Proposition}
\newtheorem{lemma}[theorem]{Lemma}

\theoremstyle{definition}
\newtheorem{definition}[theorem]{Definition}

\theoremstyle{remark}

\usepackage[textsize=tiny]{todonotes}

\usepackage[resetlabels]{multibib}
\newcites{supp}{References -- Supplementary}

\usepackage{graphicx,calc}
\newlength\myheight
\newlength\mydepth
\settototalheight\myheight{Xygp}
\newcommand*\inlinegraphics[1]{%
  \settototalheight\myheight{Xygp}%
  \settodepth\mydepth{Xygp}%
  \raisebox{-\mydepth}{\includegraphics[height=\myheight]{#1}}%
}

\icmltitlerunning{Fair Generalized Linear Models with a Convex Penalty}

\begin{document}

\twocolumn[
\icmltitle{Fair Generalized Linear Models with a Convex Penalty}




\begin{icmlauthorlist}
    \icmlauthor{Hyungrok Do}{nyu}
    \icmlauthor{Preston Putzel}{uci}
    \icmlauthor{Axel Martin}{nyu}
    \icmlauthor{Padhraic Smyth}{uci}
    \icmlauthor{Judy Zhong}{nyu}
\end{icmlauthorlist}

\icmlaffiliation{nyu}{Department of Population Health, NYU Grossman School of Medicine, New York, NY, USA}
\icmlaffiliation{uci}{Department of Computer Science, University of California, Irvine, CA, USA}

\icmlcorrespondingauthor{Judy Zhong}{judy.zhong@nyulangone.org}

\icmlkeywords{algorithmic fairness, algorithmic bias, generalized linear models}

\vskip 0.3in
]



\printAffiliationsAndNotice{}  

\begin{abstract}
    Despite recent advances in algorithmic fairness, methodologies for achieving fairness with generalized linear models (GLMs) have yet to be explored in general, despite GLMs being widely used in practice. In this paper we introduce two fairness criteria for GLMs based on equalizing expected outcomes or log-likelihoods. We prove that for GLMs both criteria can be achieved via a convex penalty term based solely on the linear components of the GLM, thus permitting efficient optimization. We also derive theoretical properties for the resulting fair GLM estimator. To empirically demonstrate the efficacy of the proposed fair GLM, we compare it with other well-known fair prediction methods on an extensive set of benchmark datasets for binary classification and regression. In addition, we demonstrate that the fair GLM can generate fair predictions for a  range of response variables, other than binary and continuous outcomes.
\end{abstract}


\section{Introduction}\label{sec:introduction}

Though machine learning is increasingly being used to support and perform crucial decision making tasks, recent research has clearly demonstrated that data-driven predictive models can often retain systematic biases that are present in the underlying data and  can propagate these inequalities to their predictions. For example, large biases in prediction performance have been detected for machine learning models in areas such as criminal recidivism prediction relative to race \citep{angwin2016machine}, ranking of job candidates relative to gender \citep{DBLP:journals/corr/abs-1806-01059}, face recognition relative to both race and gender \citep{ryu2017inclusivefacenet,pmlr-v81-buolamwini18a}, and in multiple healthcare applications relative to gender, race, and insurance status \citep{char2018implementing,larrazabal2020gender,seyyed2020chexclusion}. 

To address these issues there has recently been a significant body of work in the machine learning community on algorithmic fairness in the context of predictive modeling, including (i) data preprocessing methods that try to reduce disparities, (ii) in-process approaches which enforce fairness during model training, and (iii) post-process approaches which adjust a model's predictions to achieve fairness after training is completed. 
However, the majority of this work has focused on classification problems with binary outcome variables, and to a lesser extent on regression. There has been little to no investigation of fairness in contexts that include other types of outcome variables such as multiclass or count outputs.

\begin{table*}[!ht]
\setlength{\tabcolsep}{7pt}
\renewcommand{\arraystretch}{1.02}
\caption{In-process fair prediction methods and outcome types they can handle (\checkmark: demonstrated in the paper).}
\label{table:competitors}
\centering
    \renewcommand{\arraystretch}{1.5}
    {\small\begin{tabular}{lcccc}
    \hline\hline
    \multirow{2}{*}{Methods} & \multicolumn{4}{c}{Outcome Types} \\
    &
    ~Binary~ &
    Continuous &
    Multiclass &
    ~~Count~~ \\
    \hline
    Fair Constraints \citep{zafar2017fc} & \checkmark       &&&\\ 
    Disparate Mistreatment \citep{zafar2017dm} & \checkmark &&&\\ 
    Absolute/Squared Difference \citep{bechavod2017penalizing} & \checkmark &&&\\
    Group/Individual Fairness \citep{berk2017convex} & \checkmark & \checkmark &  &    \\
    Independence measured by HSIC \citep{perez2017fair} &  & \checkmark &  &    \\
    Fair ERM \citep{donini2018} & \checkmark &&& \\
    Statistical Parity \citep{pmlr-v80-agarwal18a,pmlr-v97-agarwal19d} & \checkmark & \checkmark &  &   \\
    Bounded Group Loss \citep{pmlr-v97-agarwal19d} & \checkmark & \checkmark &  &   \\
    General FERM \citep{oneto2020general} &  & \checkmark &  &     \\
    Fair GLM (Ours) & \checkmark & \checkmark & \checkmark & \checkmark  \\
    \hline\hline
    \end{tabular}}
    \\
\end{table*}

Generalized linear models (GLMs) provide a natural and systematic approach to handle a variety of different types of response variables $Y$ including real-valued, binary, categorical, ordinal, and count outcomes \citep{nelder1972generalized,mccullagh1989generalized,hilbe1994generalized}. More specifically, GLMs can be viewed as a generalization of standard linear regression, where the normality assumption on the conditional distribution of $Y$ is relaxed to allow for a range of distributional forms, including binomial, multinomial, and Poisson. While GLMs have a significantly simpler functional form compared to flexible modern machine learning models (such as tree-based models and deep neural networks), they nonetheless are a frequent method of choice for building predictive models across many applications areas such as biology, medicine, social science, engineering, climate analysis, and risk analysis \citep{lindsey2000applying}.  

Thus, there is a gap between the practical use of GLMs and the development of fairness-aware methodologies for GLMs in the research literature. This paper addresses the gap by:
\vspace{-1em}
\begin{itemize}
    \item Developing a new framework for GLMs to achieve fair predictions for under-represented groups; \vspace{-0.5em}
    \item Providing theoretical performance properties and optimization guarantees for fair GLMs; \vspace{-0.5em}
    \item Demonstrating that the proposed fair GLM can improve prediction parities for a variety of outcomes including the less-studied count and multinomial outcomes, and investigating, via a systematic empirical study across 11 datasets, the accuracy-disparity trade-offs of our fair GLM compared with existing alternative approaches. 
\end{itemize} 
Full code and datasets for our experiments are available on { \url{https://github.com/hyungrok-do/fair-glm-cvx}}. The proofs for all of the theoretical results in the paper can be found in Appendix \ref{appendix:proof}.

\section{Related Work}\label{sec:related}
We focus on approaches that consider \emph{group fairness}, which require models to have similar predictive performances across groups. Among the general class of group fairness methods, we focus on the widely-used \emph{in-process} approach, where fairness criteria are introduced during the training process, typically by adding a fairness constraint or penalty to the formulation of their objective function. We limit our scope in this paper to methods that do not include sensitive features as inputs to the prediction model but instead use them during training as part of a penalty or constraint on disparity. 

Below we discuss related work organized by the types of outcome variables that are handled by each approach, as summarized in
Table \ref{table:competitors}.

For tackling fair binary classification tasks, \citet{zafar2017fc} and \citet{zafar2017dm} proposed the fair constraint (FC) and the constraint for avoiding disparate mistreatment (DM), respectively. Furthermore, \citet{donini2018} formalized fair empirical risk minimization (FERM) as a constrained risk minimization problem and applied it to linear and nonlinear support vector machines (SVMs). \citet{pmlr-v80-agarwal18a} proposed a reductions approach which can transform binary classification problems under statistical parity (SP) or under equalized odds constraints into  unconstrained cost-sensitive classification problems. 

For fair regression tasks, \citet{perez2017fair} proposed a penalty based on the Hilbert-Schmidt independence criterion (HSIC) to encourage independence between predicted values and sensitive attributes, and applied this approach to both linear and kernel regression.

There have also been attempts to develop frameworks that can be generalized to multiple types of outcomes. \citet{berk2017convex} proposed individual fairness (IF) and group fairness (GF) penalties and applied them to binary logistic and linear regression models. \citet{oneto2020general} proposed a generalized FERM (GFERM) framework, extending the FERM idea to regression models. \citet{pmlr-v97-agarwal19d} extended their reductions approach for a general class of problems defined by Lipschitz loss functions and applied the approach to regression and binary classification. They considered statistical parity (SP) and bounded group loss (BGL) as fairness criteria.

Fairness for  multiclass classification has also seen relatively little investigation despite its potential utility, and in particular  in-process frameworks that cover  multiclass classification do not appear to have been investigated in prior work. A likely reason is that extending the fairness approaches used for other problems, such as enforcing equalized odds, is not trivial to extend to the multiclass case. \citet{ye2020unbiased} proposed to use one-versus-rest SVMs with a penalty on misclassification rates,  while \citet{putzel2022blackbox} have investigated several different extensions of the demographic parity and equalized odds criteria in the multiclass setting. \citet{denis2021fairness} proposed a plug-in estimator that guarantees  demographic parity for multiclass classification. 

Thus, overall, to the best of our knowledge, there has been no prior work providing a unified framework for fairness methods with GLMs.

\section{Problem Formulation}\label{sec:problem}

\begin{table*}[!ht]
    \renewcommand{\arraystretch}{2}
    \caption{Common GLM Distributions with Canonical Link Functions}
    \label{table:glm-links}
    \centering
        {\begin{tabular}{lcccccc}
        \hline\hline
        Distribution             & Link      & Support &$\mu = g^{-1}(\mathbf{X}\boldsymbol{\beta})$      & $\phi$ & $\mu'$ & $\ell'$ \\  \hline
        Bernoulli$(\mu)$         & Logit     & $\{0,1\}$ & $\displaystyle \frac{1}{1+\exp(-\mathbf{X}\boldsymbol{\beta})}$     & 1 & $\mu (1-\mu)$ & $y-\mu$\\ 
        Multinomial$(\mu_{i})$   & Logit     & $\{0,1\}$ & $\displaystyle \frac{\exp(\mathbf{X}\boldsymbol{\beta}_{i})}{1+\sum_{j\neq  i}\exp(\mathbf{X}\boldsymbol{\beta}_{j})}$ & 1 & $\mu_{i}(1-\mu_{i})$ & $y_{i}-\mu_{i}$\\
        Normal$(\mu,\sigma^{2})$ & Identity  & $\mathbb{R}$ &$\mathbf{X}\boldsymbol{\beta}$                    & $\sigma^{2}$ & 1 & $\displaystyle \frac{(y-\mu)}{\sigma^{2}}$\\
        Poisson$(\mu)$           & Log       & $\{0\} \cup \mathbb{Z}_{+}$ & $\exp(\mathbf{X}\boldsymbol{\beta})$              & 1 & $\mu$ & $y-\mu$\\
        \hline\hline
    \end{tabular}}
\end{table*}

\subsection{Notation and Problem Definition}
Throughout the paper, we consider the generalized linear model framework
\begin{equation}\label{eqn:glm}
    \mathbb{E}\left[Y|\mathbf{X}\right]  = g^{-1}(\mathbf{X}\boldsymbol{\beta}) = \mu(\theta),
\end{equation}
where $\mathbf{X}$ and $Y$ are predictor and response variables distributed over $\mathbb{R}^{p}$ and $\mathcal{Y}$, respectively, $\boldsymbol{\beta} \in \mathbb{R}^{p}$ is a regression coefficient vector, $\theta$ is a function of $\mathbf{X}\boldsymbol{\beta}$, $g$ is a link function, and $\mu=g^{-1}$ is a mean function. We provide representative examples of GLMs in Table \ref{table:glm-links}. The probability density/mass function of $Y$ given $\mathbf{X}$ has the following form, which is parameterized by $\theta$ and $\phi$ as
\begin{equation}\nonumber
    f(y|\theta,\phi) = \exp\left(\frac{y\theta - b(\theta)}{a(\phi)} + c(y,\phi)\right),
\end{equation}
where $a,b,$ and $c$ are functions that vary depending on the choice of link functions or distributions. In this paper, we limit our scope to  canonical link functions so that $\theta = \mathbf{X}\boldsymbol{\beta}$.

Suppose that we are given $K$ groups, defined as the possible values $\mathcal{A} = \{a_{1},\cdots,a_{K}\}$ of a sensitive attribute $A$ such as race/ethnicity or gender. We denote the predictor and response variables of group $k$ as $\mathbf{X}^{k}$ and $Y^{k}$ and their probability distributions as $p^{k}_{\mathbf{X}}$ and $p^{k}_{Y}$, respectively. In turn, $p^{k}_{\mathbf{X}|Y=y}$ represents the conditional distribution of $\mathbf{X}^{k}$ given that $Y^{k}=y$. 

Our primary goal is to build a prediction model that learns the relationship between $\mathbf{X}$ and $Y$ well. A conventional GLM approach is to estimate a parameter vector  $\widehat{\boldsymbol{\beta}}_{\text{GLM}}$ that maximizes the expected log-likelihood, or equivalently minimizes expected negative log-likelihood, that is, 
\begin{equation}\nonumber
    \widehat{\boldsymbol{\beta}}_{\text{GLM}}=\underset{\boldsymbol{\beta}}{\text{argmin}} ~-\mathbb{E}[\ell(\boldsymbol{\beta};\mathbf{X},Y)],
\end{equation}    
where the expectation is with respect to the joint distribution $p_{\mathbf{X}Y}$ of $(\mathbf{X},Y)$. To relieve prediction disparity, a fairness penalty term $\mathcal{D}(\boldsymbol{\beta}; \mathbf{X}, Y, A)$ is included to encourage fair prediction performance between groups.  
\begin{equation}\label{eqn:fair-glm}
    \widehat{\boldsymbol{\beta}}_{\text{FGLM}}=\underset{\boldsymbol{\beta}}{\text{argmin}} ~ -\mathbb{E}[\ell(\boldsymbol{\beta};\mathbf{X},Y)]+\lambda\mathcal{D}(\boldsymbol{\beta}; \mathbf{X}, Y, A),
\end{equation}
where $\lambda$ is a hyperparameter.
As mentioned in the previous section, several versions of fairness penalties have been proposed. In this paper, we investigate a general framework for fairness criteria for GLMs, a framework that is theoretically  applicable to all types of outcomes as well as being computationally efficient. 

\subsection{Fairness Criteria for GLMs}
\begin{definition}[Equalized Expected Outcomes]\label{def:equal-expected-outcomes}
A GLM, parameterized by $\boldsymbol{\beta}$, satisfies the criterion of \emph{equalized expected outcomes}, with respect to a sensitive attribute $A$ and response variable $Y$, if the GLM's expected outcomes are identical for all possible outcomes for every pair of groups, that is,
\begin{equation}
    \mathbb{E}\left[\mu(\mathbf{X}^{ky}\boldsymbol{\beta})\right] = \mathbb{E}\left[\mu(\mathbf{X}^{ly}\boldsymbol{\beta})\right],
\end{equation}
where $\mathbf{X}^{ky} \sim p^{k}_{\mathbf{X}|Y=y}$ and $\mathbf{X}^{ly} \sim p^{l}_{\mathbf{X}|Y=y}$ for all $k, l \in \mathcal{A}$ and $y \in \mathcal{Y}$. Here, $\mathcal{Y}$ is a set of all possible outcomes of $Y$. In practice, for real-valued or unbounded outcomes, $\mathcal{Y}$ is discretized in the fairness penalty (as in prior work, \citet{donini2018}), which we will discuss later in Section \ref{sec:discretization}.
\end{definition}
Note that the set of ${\boldsymbol{\beta}}$ satisfying the equalized expected outcome is non-empty as it can be achieved by the trivial solution $\widehat{\boldsymbol{\beta}}=\mathbf{0}$. This criterion is a natural extension of \emph{equalized odds} \citep{hardt2016}, which requires the predicted values for each group to be the same for each true outcome $y \in \{0,1\}$ or equivalently, requires that the predicted values and the sensitive attribute  be conditionally independent given the true outcomes. However, statistical independence is not equivalent to equalized expected outcomes for types of prediction tasks other than binary classification; indeed, the equalized expected outcomes is a weaker condition than statistical independence. Even though the equalized expected outcomes does not guarantee statistical independence, producing the same expected prediction for the same true outcome is still a meaningful criterion. 

Previous work in \citet{donini2018} considered equalizing loss functions across sensitive attributes for the case of binary classification. This was further extended to the case of regression tasks in \citet{oneto2020general}. Here we introduce an extension of this approach to the case of GLMs. For GLMs, negative log-likelihoods are used as loss functions, thus, we consider the conditional expected log-likelihood
\begin{align}\nonumber
    \mathbb{E}[\ell(\boldsymbol{\beta};\mathbf{X},y)] = \int_{\mathbf{x} \in \mathcal{X}} \left( \frac{y\theta - b(\theta)}{a(\phi)} + c(y,\phi) \right)dp_{\mathbf{X}|Y=y}(\mathbf{x}),
\end{align}
which measures the expected log-likelihood given $y$.

\begin{definition}[Equalized Expected Log-likelihoods]\label{def:equal-expected-log-likelhoods}
A GLM satisfies equalized expected log-likelihoods with respect to a sensitive attribute $A$ and response variable $Y$ if 
\begin{equation}\label{eqn:equal-expected-log-likelihoods}
    \mathbb{E}[\ell(\boldsymbol{\beta};\mathbf{X}^{ky},y)] = \mathbb{E}[\ell(\boldsymbol{\beta};\mathbf{X}^{ly},y)],
\end{equation}
for all $k,l \in \mathcal{A}$ and $y \in \mathcal{Y}$.
\end{definition}

The trivial solution $\widehat{\boldsymbol{\beta}}=\mathbf{0}$ also satisfies the criterion of equalized expected log-likelihoods, so Definition \ref{def:equal-expected-log-likelhoods} is always achievable. However, we emphasize that Definition \ref{def:equal-expected-outcomes} and \ref{def:equal-expected-log-likelhoods} are different. Equalizing expected outcomes need not result in equalizing expected log-likelihoods and vice versa.

Based on Definition \ref{def:equal-expected-outcomes} and \ref{def:equal-expected-log-likelhoods}, we define a measure of disparity between groups by summing up the squared differences of the pairwise expected outcomes or log-likelihoods for all possible true outcomes:
\begin{align}\nonumber
    &\mathcal{D}_{\text{EO}} = \sum_{k,l \in \mathcal{A}}\sum_{y \in \mathcal{Y}} \left(\mathbb{E}[\mu(\mathbf{X}^{ky}\boldsymbol{\beta})] -\mathbb{E}[\mu(\mathbf{X}^{ly}\boldsymbol{\beta})] \right)^{2}, \\ \nonumber
    &\mathcal{D}_{\text{ELL}} = \sum_{k,l \in \mathcal{A}}\sum_{y \in \mathcal{Y}} \left(\mathbb{E}[\ell(\boldsymbol{\beta};\mathbf{X}^{ky},y)] -\mathbb{E}[\ell(\boldsymbol{\beta};\mathbf{X}^{ly},y)] \right)^{2}.
\end{align}

Here we assume $y$ is discretized into a finite number of regions causing the summation $\sum_{y \in \mathcal{Y}}$ to be finite. Both disparities above can directly be plugged into \eqref{eqn:fair-glm} as a penalty term to get a fair GLM estimator. However, they are not convex in general, depending on the choice of link functions, and thus, it will be hard to attain globally optimal solutions. In the following section, we introduce a new convex penalty, an upper bound on each of $\mathcal{D}_{\text{EO}}$ and $\mathcal{D}_{\text{ELL}}$.

\subsection{A Linear Component Fairness Penalty for GLMs}

\begin{lemma}\label{lem:key-lemma}
Let $h:\mathbb{R} \to \mathbb{R}$ be a differentiable function. Then, for any random variables $\theta^{k}$ and $\theta^{l}$ distributed over $\mathbb{R}$, the following inequality holds:
\begin{equation}\label{eqn:key-lemma}
        \left(\mathbb{E}[h(\theta^{k})] - \mathbb{E}[h(\theta^{l})]\right)^{2} \leq \mathbb{E}[h'(\theta^{m})^{2}]\mathbb{E}[(\theta^{k} - \theta^{l})^{2}],
    \end{equation} 
where $\theta^{m} = \alpha \theta^{k} + (1-\alpha) \theta^{l}$, for some $\alpha \in [0,1]$. 
\end{lemma}

The lemma implies that the squared difference of the expected value of the function $h$ is bounded by the second-order moment of the difference of the $\theta$s, provided that the expectation of $h'$ is finite. Based on this lemma, we provide two key results of our work.

\begin{proposition}\label{prop:outcome}
Given $y \in \mathcal{Y}$, let $\mu$ be an inverse link function and $\boldsymbol{\beta}$ be the coefficient vector of GLM defined in \eqref{eqn:glm}. Let $h = \mu$, $\theta^{k}=\mathbf{X}^{ky}\boldsymbol{\beta}$, and $\theta^{l}=\mathbf{X}^{ly}\boldsymbol{\beta}$. Then,
\begin{align}\nonumber
    &\left(\mathbb{E}[\mu(\mathbf{X}^{ky}\boldsymbol{\beta})] - \mathbb{E}[\mu(\mathbf{X}^{ly}\boldsymbol{\beta})]\right)^{2} \\ \label{eqn:prop-outcome} &~~\leq \mathbb{E}[\mu'(\mathbf{X}^{m}\boldsymbol{\beta})^{2}]\mathbb{E}[(\mathbf{X}^{ky}\boldsymbol{\beta} - \mathbf{X}^{ly}\boldsymbol{\beta})^{2}], 
\end{align} 
where $\mathbf{X}^{m} = \alpha \mathbf{X}^{ky} + (1-\alpha) \mathbf{X}^{ly}$, for some $\alpha \in [0,1]$.
\end{proposition}

The implication of the proposition is that the squared difference of expected GLM outcomes of group $k$ and $l$ is bounded by the second-order moment of the difference of linear components of group $k$ and $l$. Therefore, we can minimize the left-hand side of \eqref{eqn:prop-outcome} by minimizing the second-order moment provided $\mathbb{E}[\mu'(\mathbf{X}^{m}\boldsymbol{\beta})^{2}]$ is bounded above. 

This result motivates us to use $\mathbb{E}[(\mathbf{X}^{ky}\boldsymbol{\beta} - \mathbf{X}^{ly}\boldsymbol{\beta})^{2}]$ as a penalty to encourage fairness. As shown in Appendix \ref{appendix:proof}, $\mathbb{E}[\mu'(\mathbf{X}^{m}\boldsymbol{\beta})^{2}]$  is bounded above for some outcome distributions including normal, binomial and multinomial, but not for all distributions. To address this issue, later in this section we introduce results that support the usage of the term as a fairness penalty for broader classes of distributions.

The next proposition states that $\mathbb{E
}[(\mathbf{X}^{ky}\boldsymbol{\beta} - \mathbf{X}^{ly}\boldsymbol{\beta})^{2}]$ bounds the difference of expected log-likelihoods as well.

\begin{proposition}\label{prop:likelihood}
Given $y \in \mathcal{Y}$, let 
\begin{equation}\nonumber
    h(\theta) = \ell(\boldsymbol{\beta};\mathbf{X},y) = \frac{y\theta - b(\theta)}{a(\phi)} + c(y,\phi),
\end{equation}
where $\theta = \mathbf{X}\boldsymbol{\beta}$ and let $\theta^{k}=\mathbf{X}^{ky}\boldsymbol{\beta}$, and $\theta^{l}=\mathbf{X}^{ly}\boldsymbol{\beta}$. Then, we have 
\begin{align}\nonumber
    &\left(\mathbb{E}[\ell(\boldsymbol{\beta};\mathbf{X}^{ky},y)] - \mathbb{E}[\ell(\boldsymbol{\beta};\mathbf{X}^{ly},y)]\right)^{2} \\ \label{eqn:main-thm} &~~\leq \mathbb{E}[\ell'(\boldsymbol{\beta};\mathbf{X}^{m},y)^{2}]\mathbb{E}[(\mathbf{X}^{ky}\boldsymbol{\beta} - \mathbf{X}^{ly}\boldsymbol{\beta})^{2}],
\end{align} 
where $\mathbf{X}^{m} = \alpha \mathbf{X}^{ky} + (1-\alpha) \mathbf{X}^{ly}$, for some $\alpha \in [0,1]$.
\end{proposition}

As in the previous proposition, for some distributions,  $\mathbb{E}[\ell'(\boldsymbol{\beta};\mathbf{X}^{m},y)^{2}]$, is bounded above but not for all. More rigorous theoretical arguments are presented Section \ref{sec:fglm}.

The two propositions above motivate the definition of fairness using the linear components:
\begin{equation}\label{eqn:dlc}
    \mathcal{D}_{\text{LC}} = \sum_{k,l \in A}\sum_{y \in \mathcal{Y}} \mathbb{E}[(\mathbf{X}^{ky}\boldsymbol{\beta} - \mathbf{X}^{ly}\boldsymbol{\beta})^{2}].
\end{equation}
We can rewrite the penalty term by applying the bias-variance trade-off as
\begin{align}\nonumber
    & \mathbb{E}[(\mathbf{X}^{ky}\boldsymbol{\beta} - \mathbf{X}^{ly}\boldsymbol{\beta})^{2}] \\ \nonumber & =
    \mathbb{V}\left[ \mathbf{X}^{ky}\boldsymbol{\beta} - \mathbf{X}^{ly}\boldsymbol{\beta} \right] + \left(\mathbb{E}\left[\mathbf{X}^{ky}\boldsymbol{\beta}\right] -\mathbb{E}\left[ \mathbf{X}^{ly}\boldsymbol{\beta}\right]\right)^{2}.
\end{align}
That is, our penalty term consists of the variance of the difference of linear components and the expected difference of linear components. The second term, the difference of linear components, has been considered as a fairness penalty previously in \citet{bechavod2017penalizing} for binary classification and \citet{berk2017convex} for both binary classification and linear regression. Our analysis shows that adding the variance term can effectively bound the difference of the expected log-likelihoods. In addition, computationally, $\mathcal{D}_{\text{LC}}$ is still convex in $\boldsymbol{\beta}$, thus permitting efficient optimization.

\subsection{The Fair Generalized Linear Model (F-GLM)}\label{sec:fglm}
The above results provide the basis for our F-GLM estimator: $\widehat{\boldsymbol{\beta}}_{\text{FGLM}}$ is estimated by minimizing the function
\begin{align}\label{eqn:expected-objective}
    -\mathbb{E}[\ell(\boldsymbol{\beta};\mathbf{X},Y)] + \frac{\lambda}{\kappa}\sum_{k,l \in \mathcal{A}}\sum_{y \in \mathcal{Y}}\mathbb{E}[(\mathbf{X}^{ky}\boldsymbol{\beta} - \mathbf{X}^{ly}\boldsymbol{\beta})^{2}]
\end{align}
where $\lambda \geq 0$ is a tuning parameter that controls the trade-off between fairness and log-likelihood and $\kappa = {|\mathcal{Y}| K(K-1)}/2$ which is the number of all possible combinations. 

We now provide two theorems proving that, at the value of $\widehat{\boldsymbol{\beta}}_{\text{FGLM}}$ which minimizes \eqref{eqn:expected-objective}, the corresponding $\mathcal{D}_{\text{EO}}$ and $\mathcal{D}_{\text{ELL}}$ are bounded by $\mathcal{D}_{\text{LC}}$ multiplied by respective constants $C_{\mu}$ and $C_{\ell}$ independent of $\lambda$. Thus, even if we impose a greater value of $\lambda$ to get $\widehat{\boldsymbol{\beta}}_{\text{FGLM}}$ with smaller $\mathcal{D}_{\text{LC}}$, the constant does not change, and thus, the upper bound of $\mathcal{D}_{\text{EO}}$ gets smaller. That is, we can obtain $\widehat{\boldsymbol{\beta}}_{\text{FGLM}}$ with $\mathcal{D}_{\text{EO}}$ as small as we want, by increasing $\lambda$. This applies to $\mathcal{D}_{\text{ELL}}$ equivalently, provided by Theorem \ref{thm:likelihood}.

\begin{theorem}\label{thm:outcome}
Given $y \in \mathcal{Y}$, let $\mu$ be an inverse link function and $\widehat{\boldsymbol{\beta}}_{\text{\normalfont FGLM}}$ be the minimizer of \eqref{eqn:expected-objective}, there exists $C_{\mu} > 0$, which is independent of $\lambda$, satisfying
    \begin{align}\label{eqn:main-thm1}
        &\left(\mathbb{E}[\mu(\mathbf{X}^{ky}\widehat{\boldsymbol{\beta}}_{\text{\normalfont FGLM}})] - \mathbb{E}[\mu(\mathbf{X}^{ly}\widehat{\boldsymbol{\beta}}_{\text{\normalfont FGLM}})]\right)^{2} \\\nonumber  &~~\leq C_{\mu}\mathbb{E}[(\mathbf{X}^{ky}\widehat{\boldsymbol{\beta}}_{\text{\normalfont FGLM}} - \mathbf{X}^{ly}\widehat{\boldsymbol{\beta}}_{\text{\normalfont FGLM}})^{2}].
    \end{align} 
\end{theorem}
The proof makes use of the fact that $\mu'$ is either bounded or monotonically increasing for the case of canonical GLMs. Full details of the proof can be found in Appendix \ref{proof:thm-outcome} 
\begin{theorem}\label{thm:likelihood}
    Similar to the previous theorem, we also have $C_{\ell} > 0$, independent of $\lambda$, satisfying
    \begin{align}\label{eqn:main-thm2}
        &\left(\mathbb{E}[\ell(\widehat{\boldsymbol{\beta}}_{\text{\normalfont FGLM}};\mathbf{X}^{ky},y)] - \mathbb{E}[\ell(\widehat{\boldsymbol{\beta}}_{\text{\normalfont FGLM}};\mathbf{X}^{ly},y)]\right)^{2} \\\nonumber  &~~\leq C_{\ell}\mathbb{E}[(\mathbf{X}^{ky}\widehat{\boldsymbol{\beta}}_{\text{\normalfont FGLM}} - \mathbf{X}^{ly}\widehat{\boldsymbol{\beta}}_{\text{\normalfont FGLM}})^{2}].
    \end{align}
\end{theorem}

While we do not have theoretical results on the tightness of the bounds, our empirical results later in the paper suggest that optimizing the bounds produces models with useful prediction-fairness trade-offs.

The corresponding empirical objective function for \eqref{eqn:expected-objective}, given a training dataset $\{(y_{i},\mathbf{x}_{i},A_{i}) \in \mathcal{Y} \times \mathbb{R}^{1 \times p} \times \mathcal{A}:i=1,\cdots,n\}$ is 
\begin{align}\label{eqn:fair-glm-problem}
    &-\frac{1}{n}\sum_{i=1}^{n}\ell(\boldsymbol{\beta};\mathbf{x}_{i},y_{i}) \\ \nonumber & \quad+\frac{\lambda}{\kappa} \sum_{k,l \in \mathcal{A}}\sum_{y \in \mathcal{Y}} \frac{1}{n^{kly}}\sum_{(i,j) \in \mathcal{S}^{kly}}(\mathbf{x}_{i}\boldsymbol{\beta} - \mathbf{x}_{j}\boldsymbol{\beta})^{2},
\end{align}  
where $$\ell(\boldsymbol{\beta};\mathbf{x}_{i},y_{i}) = \frac{y_{i}\mathbf{x}_{i}\boldsymbol{\beta} - b(\mathbf{x}_{i}\boldsymbol{\beta})}{a(\phi)} + c(y_{i}, \phi),$$ $\mathcal{S}^{kly} = \{(i,j): y_{i} = y_{j} = y, A_{i} = k, A_{j} = l\},$ and $n^{kly} = |\mathcal{S}^{kly}|$. 

Note that our fairness definitions are formulated conditional on $Y=y$ to equalize the terms (log-likelihood, expected outcomes, or linear expectations) within each group of subjects with outcomes equal to the discretized value $y$ or included within the discretized region. For binary outcomes, this means enforcing parity separately for $y=0$ and $y=1$, which has previously been reported to achieve better fairness \citep{hardt2016} rather than encouraging demographic parity, which does not condition on $y$. For continuous or unbounded outcomes, define a discretization mapping that maps $\mathcal{Y}$ into a discretized set $\left\{[\delta_{i}, \delta_{i+1})\right\}$, rendering finite numbers of regions. We denote the number of $[\delta_{i}, \delta_{i+1})$ by $|\mathcal{Y}|$.

We emphasize that the discretization is only applied to the fairness penalty term, but not to the log-likelihood term. Additional details about the discretization process can be found in Section \ref{sec:discretization}. In principle, the proposed fairness definition of equalized log-likelihoods can also be formulated without conditioning on $y$ to achieve marginal fairness between the sensitive attributes. 

We further note that our penalty term is similar to the individual fairness penalty of \citet{berk2017convex} defined as the sum of squared difference of linear components $\mathbf{x}_{i}\boldsymbol{\beta}$ and $\mathbf{x}_{j}\boldsymbol{\beta}$ of two individuals sampled from different groups weighted by a distance function $d(y_{i},y_{j})$. \citet{berk2017convex} used $d(y_{i},y_{j}) = \mathbb{I}(y_{i}=y_{j})$ and $d(y_{i},y_{j}) = \exp(-(y_{i}-y_{j})^{2})$ for binary classification and regression tasks, respectively. The choice of the identity function as the distance function seems to yield a formulation which bears similarity to ours. However, two penalty terms have different denominators, and the Individual Fairness penalty is not a finite sample estimation of the expectation \eqref{eqn:dlc}. 
In addition, our work is the first to provide theoretical support for learning fair GLMs; we combine these theoretical results with  extensive empirical results in Section \ref{sec:experiments}, considerably broadening the scope of \citet{berk2017convex}.

\section{Consistency}\label{sec:properties}

Here we present the $\sqrt{n}$-consistency of the F-GLM estimator. The full proofs can be found in Appendix \ref{appendix:proof}.

\begin{lemma}\label{lem:consistency}
As $\min_{k}n^{k} = n \to \infty$, we have,
\begin{align}\nonumber
    \mathbf{D} &\to 
    \frac{1}{\kappa}\sum_{k,l \in \mathcal{A}}\sum_{y\in \mathcal{Y}}\mathbb{E}\left[(\mathbf{X}^{ky} - \mathbf{X}^{ly})^{2}\right]
    = \Delta
\end{align}
\end{lemma}
One can easily confirm that if all the pairs $\mathbf{X}^{ky}$ and $\mathbf{X}^{ly}$ are identically distributed, then $\Delta$ becomes $\mathbf{0}$.

\begin{theorem}[$\sqrt{n}$-consistency]\label{thm:consistency}
Let $\boldsymbol{\beta}^{*}$ be the true GLM coefficient. Assuming the two regularity conditions specified in \citet{zou2006adaptive}:

If $\sqrt{n}\lambda_{n} \to \lambda_{0} \geq 0$ and $\mathcal{I}\left({\boldsymbol{\beta}}^{*}\right) = \mathbf{\Sigma}^{-1}$, as $\min_{k}n^{k} = n \to \infty$, then,
\begin{equation}\label{eqn:asymptotic}
    \sqrt{n}\left(\widehat{\boldsymbol{\beta}}_{\text{\normalfont FGLM}} - \boldsymbol{\beta}^{*}\right) \overset{d}{\to} -2\mathbf{\Sigma}^{-1}(\mathbf{W} +\lambda_{0}\Delta\boldsymbol{\beta}^{*}),
\end{equation}
as $\min_{k}n^{k} = n \to \infty$, where $\mathbf{W} \sim \mathcal{N}(\mathbf{0}, \mathbf{\Sigma})$ and $\mathcal{I}(\boldsymbol{\beta}^{*})$ is the Fisher information matrix. Thus,
\begin{equation}
    \widehat{\boldsymbol{\beta}}_{\text{\normalfont FGLM}} \to \underset{{\boldsymbol{\beta}}}{\text{\normalfont argmin}} -\mathbb{E}[\ell(\boldsymbol{\beta};\mathbf{X},Y)] + \lambda_{0}\boldsymbol{\beta}^{T}\Delta\boldsymbol{\beta}.
\end{equation}
\end{theorem}
The introduction of the penalty term does not alter the asymptotic variance of the standard GLM estimator; however, it does introduce bias. The amount of bias depends on the variance and difference of expectations between the pairs $\mathbf{X}^{ky}$ and $\mathbf{X}^{ly}$.

\section{Optimization}\label{sec:optimization}
In this section we outline our approach for optimizing \eqref{eqn:fair-glm-problem}. We first introduce the matrix form of the penalty term in \eqref{eqn:fair-glm-problem}:
\begin{equation}\nonumber
    \frac{\lambda}{\kappa} \boldsymbol{\beta}^{T}\mathbf{D}^{kly}\boldsymbol{\beta}
\end{equation}
where
\begin{equation}\nonumber
    \mathbf{D}^{kly} = \frac{1}{n^{kly}}\sum_{(i,j) \in \mathcal{S}^{kly}} \left(\mathbf{x}_{i} - \mathbf{x}_{j}\right)^{T}\left(\mathbf{x}_{i} - \mathbf{x}_{j}\right),
\end{equation}
with $\mathbf{x}_{i} \in \mathbb{R}^{1 \times p}$. Therefore, we can rewrite the objective function \eqref{eqn:fair-glm-problem} as
\begin{equation}\nonumber
    -\frac{1}{n} \sum_{i=1}^{n}\ell(\boldsymbol{\beta};\mathbf{x}_{i},y_{i}) + \lambda \boldsymbol{\beta}^{T}\Bigg(\underbrace{\frac{1}{\kappa} \sum_{{y} \in \mathcal{Y}}   \mathbf{D}^{kl{y}}}_{\mathbf{D}}\Bigg)\boldsymbol{\beta},
\end{equation}
since $\mathbf{D}$ is positive semi-definite the objective function \eqref{eqn:fair-glm-problem} is convex; thus it can efficiently be solved with first or second-order methods.

\subsection{Newton-Raphson Optimization}
Below we describe a Newton-Raphson method, a widely used second-order approach, for minimizing the objective function defined in \eqref{eqn:fair-glm-problem}.
Starting with an initial solution $\boldsymbol{\beta}^{(0)}=\mathbf{0}$, the Newton-Raphson method iteratively improves the current solution with the following update rule:
\begin{equation}
    \boldsymbol{\beta}^{(t+1)} = \boldsymbol{\beta}^{(t)} - [\nabla^{2} F(\boldsymbol{\beta}^{(t)})]^{-1} \nabla F(\boldsymbol{\beta}^{(t)}),
\end{equation}
where $$\nabla F(\boldsymbol{\beta}) = -\frac{1}{n}\mathbf{X}^{T}(\mathbf{y} - \boldsymbol{\mu}) + \lambda\mathbf{D}\boldsymbol{\beta}$$ and $$\nabla^{2} F(\boldsymbol{\beta}) = \frac{1}{n}\mathbf{X}^{T}\mathbf{W}\mathbf{X} + \lambda \mathbf{D},$$ where $\boldsymbol{\mu} = g^{-1}(\mathbf{X}\boldsymbol{\beta})$ and $\mathbf{W} = \text{diag}(\boldsymbol{\mu})$. 

Convergence of the algorithm is guaranteed if line search is applied to determine step sizes, provided \eqref{eqn:fair-glm-problem} is convex in $\boldsymbol{\beta}$. It is well known that the Newton-Raphson method is less sensitive to the choice of step sizes/learning rates than first-order methods. 

In addition, it has been noted that convergence of first-order methods is in general not guaranteed for maximizing the Poisson log-likelihood \citep{he2016fast}, due to its non-global Lipschitz continuity. Second-order methods can be more efficient and effective for finding the global optimum solution for Poisson log-likelihood loss functions which are not globally Lipschitz-continuous. 

In the results reported in this paper we focus on second-order methods for the reasons above, but we emphasize that other optimization approaches (e.g., first-order stochastic gradient methods) can  be used in practice. 

\begin{table*}[!ht]
\setlength{\tabcolsep}{6pt}
\renewcommand{\arraystretch}{1.1}
\caption{Real-world datasets categorized by their outcome types. Sensitive attribute and number of its unique categories $(K)$, sample size $(n)$ and number of the predictor variables $(p)$ are those of the \emph{preprocessed} datasets. Preprocessing details can be found in Appendix \ref{sec:dataset}.}
\label{table:datasets}
\small
\centering
    {\begin{tabular}{ccccrr}
    \hline\hline
    Outcome Type & Dataset & Outcome &Sensitive Attribute $(K)$ & $n$ &  $p$ \\
    \hline 
    \multirow{5}{*}{Binary}      & {\scriptsize Adult \citep{dataset_adult}}           & Income $\geq$ 50K            & Gender (2)      & 45,222 &    34 \\ 
                                 & {\scriptsize Arrhythmia \citep{dataset_arrhythmia}}      & Presence of Arrhythmia    & Gender (2)      & 418 &    80 \\
                                 & {\scriptsize COMPAS \citep{dataset_compas}}          & Recidivism in 2Y          & Race (4)        &  6,172 &    11 \\
                                 & {\scriptsize Drug Consumption \citep{dataset_drug}} & Methadone usage (Y vs N)   & Race (2)        &  1,885 &    25 \\ 
                                 & {\scriptsize German Credit \citep{uci_repo}}  & Credit (Good vs Bad)         & Gender (2)      &  1,000 &    46 \\ 
    \hline
    \multirow{4}{*}{Continuous}  & {\scriptsize Communities and Crime \citep{dataset_communities}}     & Violent Crimes per Capita & Race (3)        &    1,993 &   97 \\ 
                                 & {\scriptsize Law School (LSAC)   \citep{dataset_lsac}}       & GPA         & Race (5)        &  20,715 &    7 \\ 
                                 & {\scriptsize Parkinsons Telemonitoring \citep{dataset_parkinsons}} & UPDRS Score & Gender (2)      &   5,875 &   25 \\ 
                                 & {\scriptsize Student Performance \citep{dataset_student}}       & Final Grade & Gender (2)      &     649 &   39 \\ 
    \hline
    \multirow{2}{*}{Count}       & {\scriptsize Health \& Retirement Survey (HRS)}   & \# of dependencies & \multirow{2}{*}{Race (4)}        &  \multirow{2}{*}{12,774} &   \multirow{2}{*}{23} \\ 
                                 & \scriptsize{(\url{https://hrs.isr.umich.edu/about})}  & in daily activities &         &   &  \\
    \hline
    \multirow{3}{*}{Multiclass} & \multirow{2}{*}{\scriptsize{Drug Consumption \citep{dataset_drug}}}     &  Meth usage: never used vs& \multirow{2}{*}{Race (2)}        &   \multirow{2}{*}{1,885} &   \multirow{2}{*}{25} \\ 
                                 &      &within 1Y vs over 1Y ago& &   &   \\ 
                                 & {\scriptsize Obesity \citep{dataset_obesity}}                & Obesity Levels (6)    & Gender (2)      &   2,111 &   23 \\
    \hline\hline
    \end{tabular}}
    \\
\end{table*}

\subsection{Discretization for Continuous and Unbounded Outcomes}\label{sec:discretization}

We consider a discretization mapping such that each segment $[\delta_{i}, \delta_{i+1})$ contains at least one instance from each group. In particular, for continuous outcomes we use an \emph{equal counts} strategy where each segment contains the same number of samples. We achieve smoother approximations if we use more segments $t$. However, because of the constraint that at least one sample from each group has to be included in each segment it is not always possible to increase $t$ as large as we would like. Instead, we start from a large desired value of $t$ and check if the constraint is satisfied, then if not, we continually decrease $t$ until we get a proper mapping. Empirically, we find that the performance is relatively robust as a function of discretization strategy and number of segments. For count outcomes, which are discrete but unbounded, we choose integer maximum and minimum thresholds. Then we set any values greater than the maximum threshold equal to the threshold while keeping the other values the same and vice versa. Additional details are in Appendix \ref{appendix:discretization}. 

\subsection{Computational Complexity}
Our estimation procedure can be divided into two stages: (i) preparing $\mathbf{D}$ and (ii) Newton-Raphson iterations.
The complexity of computing $\mathbf{D}^{kly}$ is $\mathcal{O}\left(n^{kly}p^{2}\right)$. As a reminder $n^{kly}$ represents the number of pairs of individuals $(i, j)$ with $A_i=k$, $A_j=l$, and $y_i=y_j=y$. Since $n^{kly}$ is a subset of the total number of pairs, it is bounded by $n^2$. There are $K^2|Y|$ total terms $\mathbf{D}^{kly}$ that make up the full matrix $D$. Thus in total computing $\mathbf{D}$ takes $O(n^2K^2|\mathcal{Y}|p^2)$
The per iteration complexity of the Newton-Raphson algorithm is $\mathcal{O}(np^2 + p^3)$ for all outcomes besides multiclass. The complexity in the case of multiclass outcomes becomes $\mathcal{O}(mnp^2+mp^{3})$ since there are an additional $\mathcal{O}(m)$ set of parameters $\boldsymbol{\beta}$. However, the complexity of computing  $\mathbf{D}$ remains the same in all cases. The complete derivation can be found in Appendix \ref{appendix:complexity}.

Thus, the largest contribution to the computational complexity of our approach is from pre-computing $\mathbf{D}$. With very large dataset sizes (large $n$), this bottleneck could be mitigated by using a subsample of the dataset to compute $\mathbf{D}$. For a dataset which also is high-dimensional, the  Newton Raphson optimization could also be replaced by faster (per iteration) first order gradient methods. However, in practice, such speedups were not necessary for any of the real-world datasets used in this paper. For the largest dataset in our experiments, with $n=45,222$ and $p=14$, fitting an F-GLM took 7 seconds to compute $\mathbf{D}$ and an additional 9 seconds until convergence of Newton-Raphson iterations.

\begin{figure*}[p]
    \centering
    \subfloat[Adult--Binary--Gender(2)]{\includegraphics[width=0.32\linewidth,keepaspectratio]{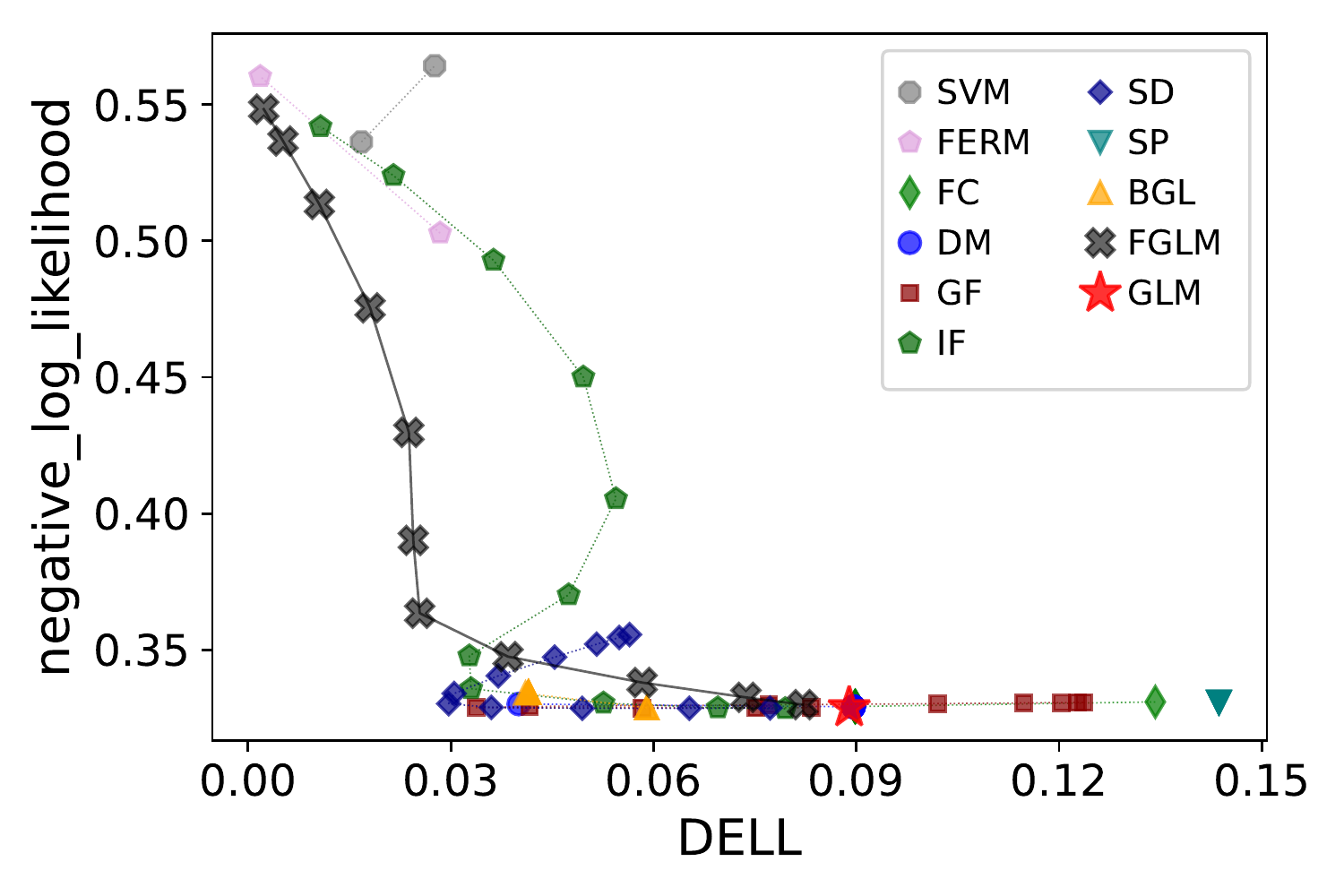}}
    \subfloat[Arrhythmia--Binary--Gender(2)]{\includegraphics[width=0.32\linewidth,keepaspectratio]{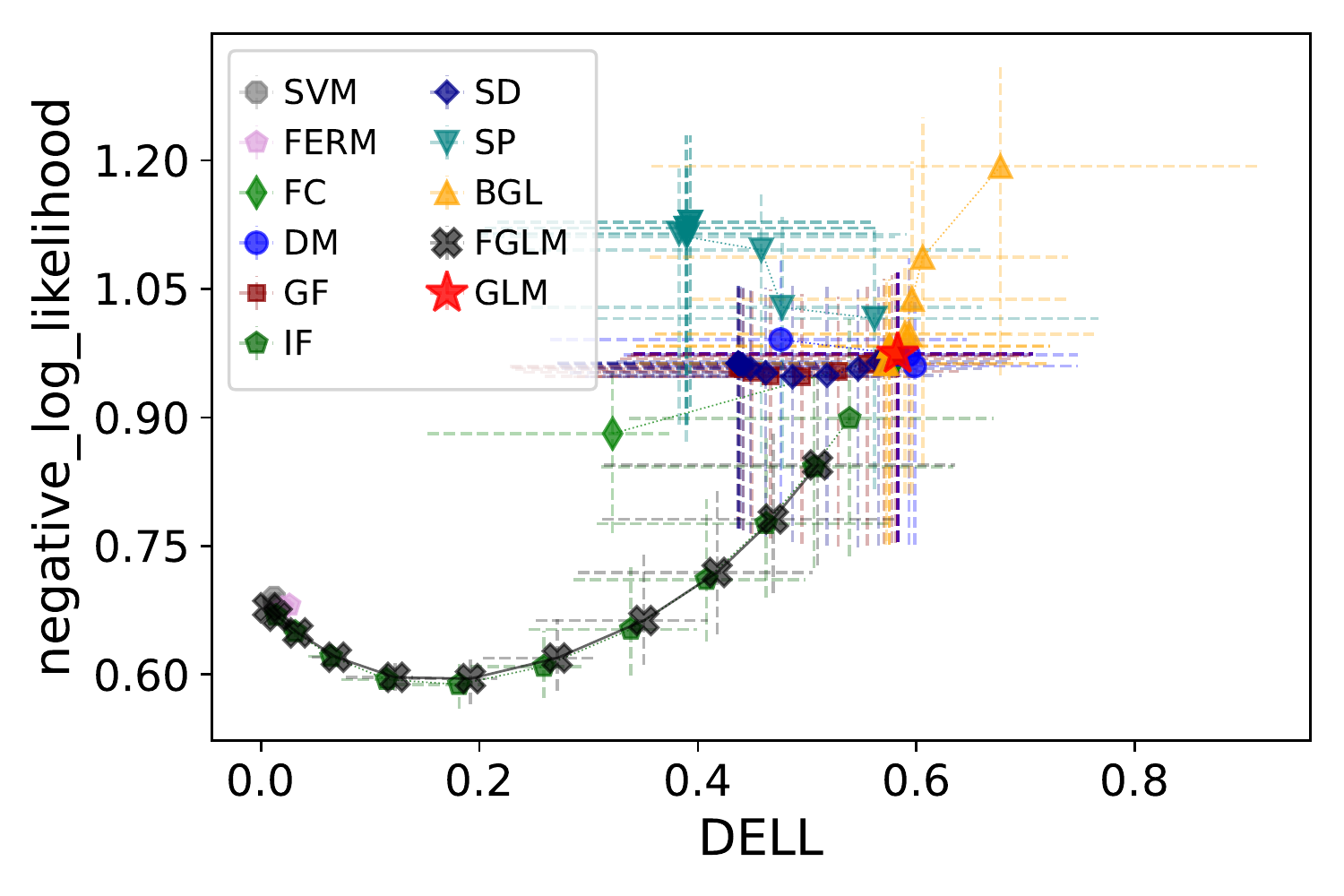}}
    \subfloat[COMPAS--Binary--Race(4)]{\includegraphics[width=0.32\linewidth,keepaspectratio]{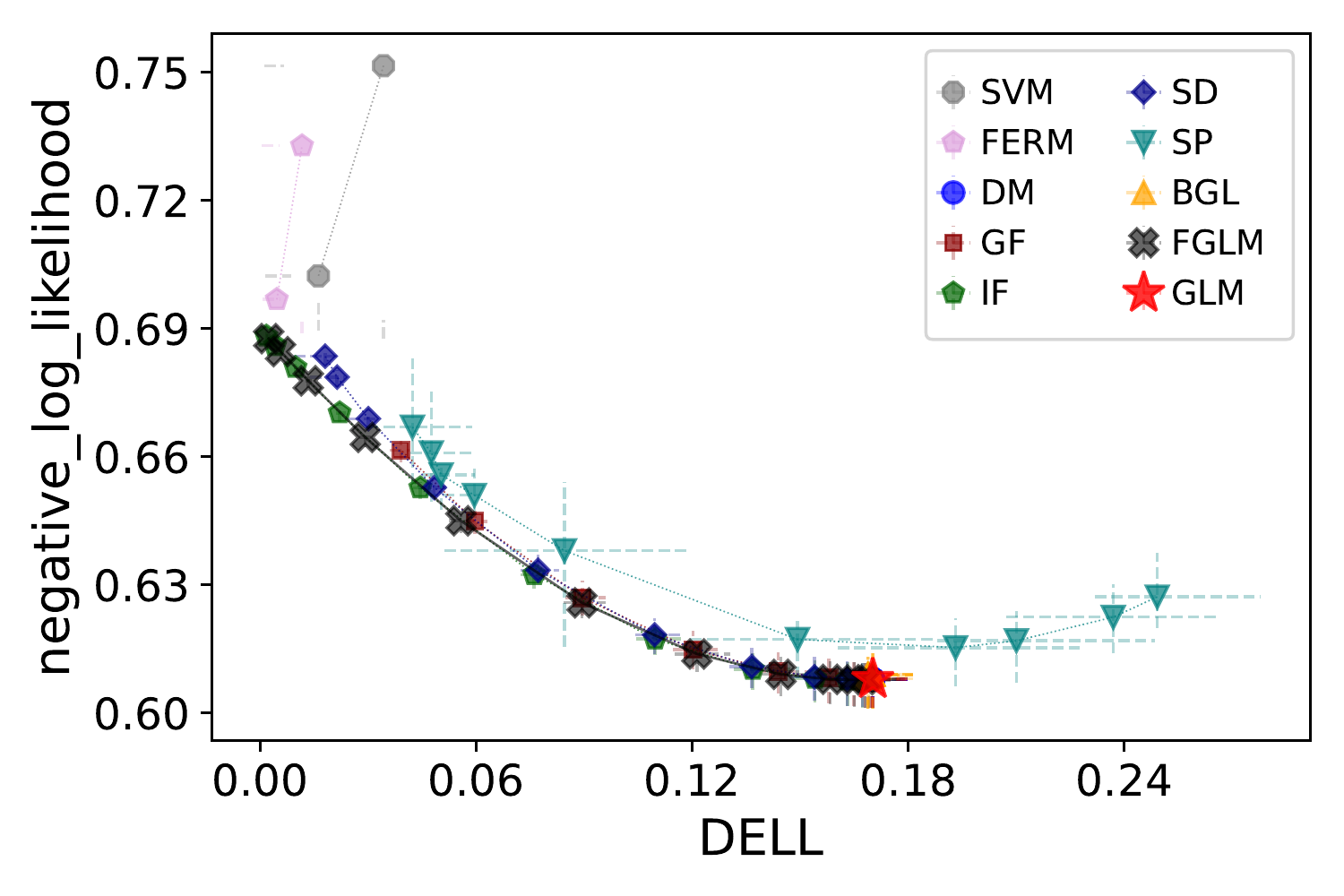}}
    \\
    \subfloat[Drug--Binary--Race(2)]{\includegraphics[width=0.32\linewidth,keepaspectratio]{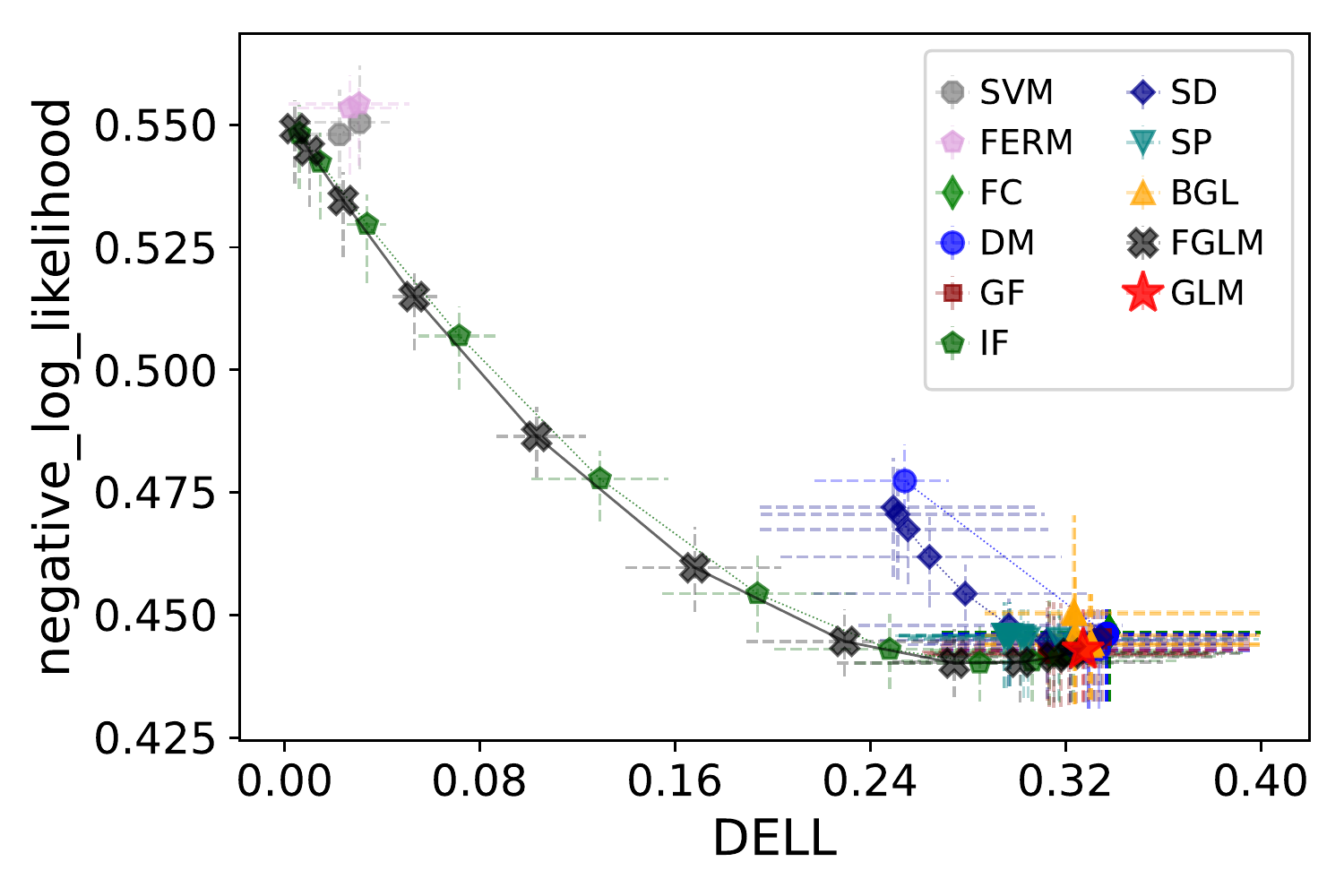}}
    \subfloat[German--Binary--Race(2)]{\includegraphics[width=0.32\linewidth,keepaspectratio]{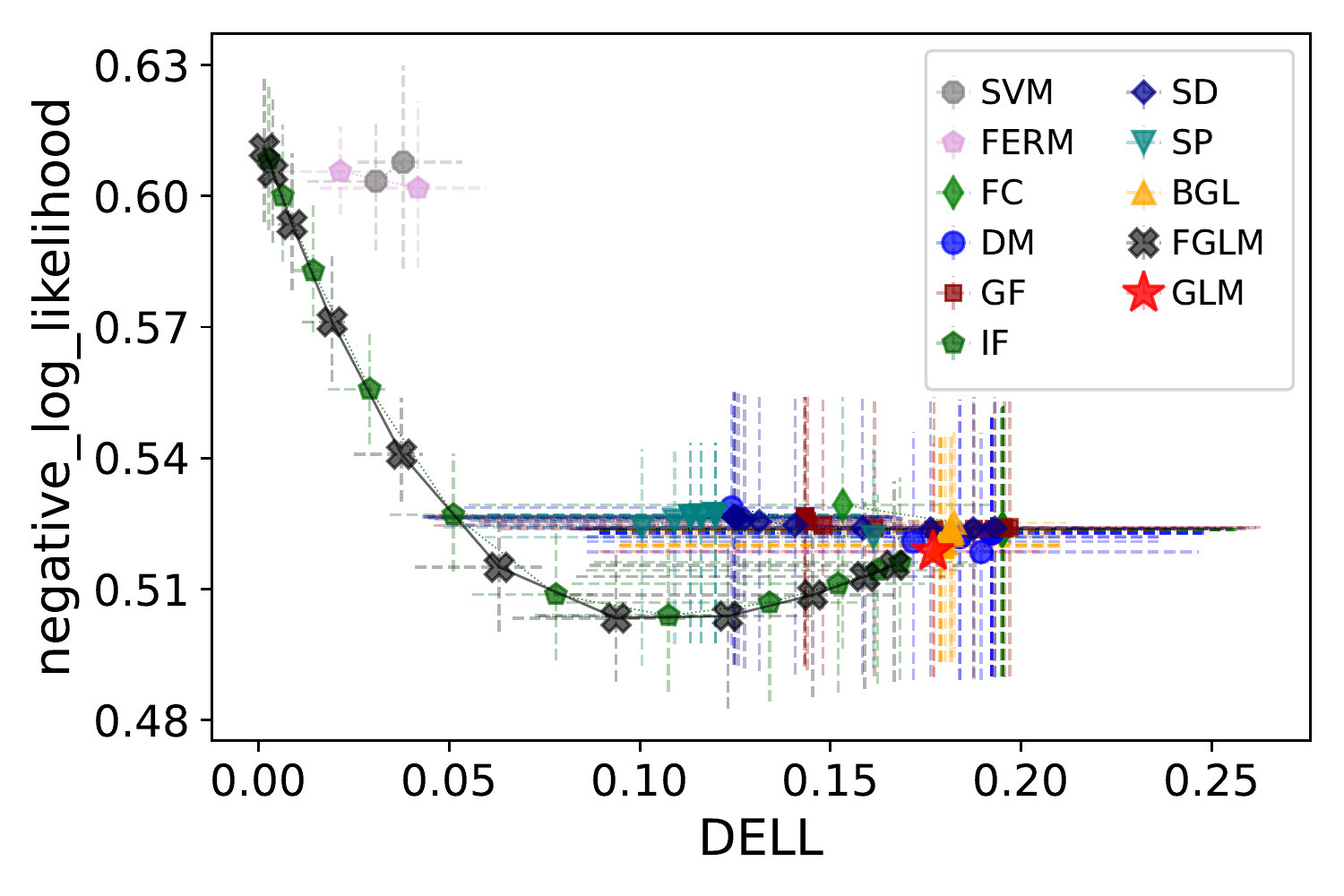}}
    \subfloat[Crime-Cont-Race(3)]{\includegraphics[width=0.32\linewidth,keepaspectratio]{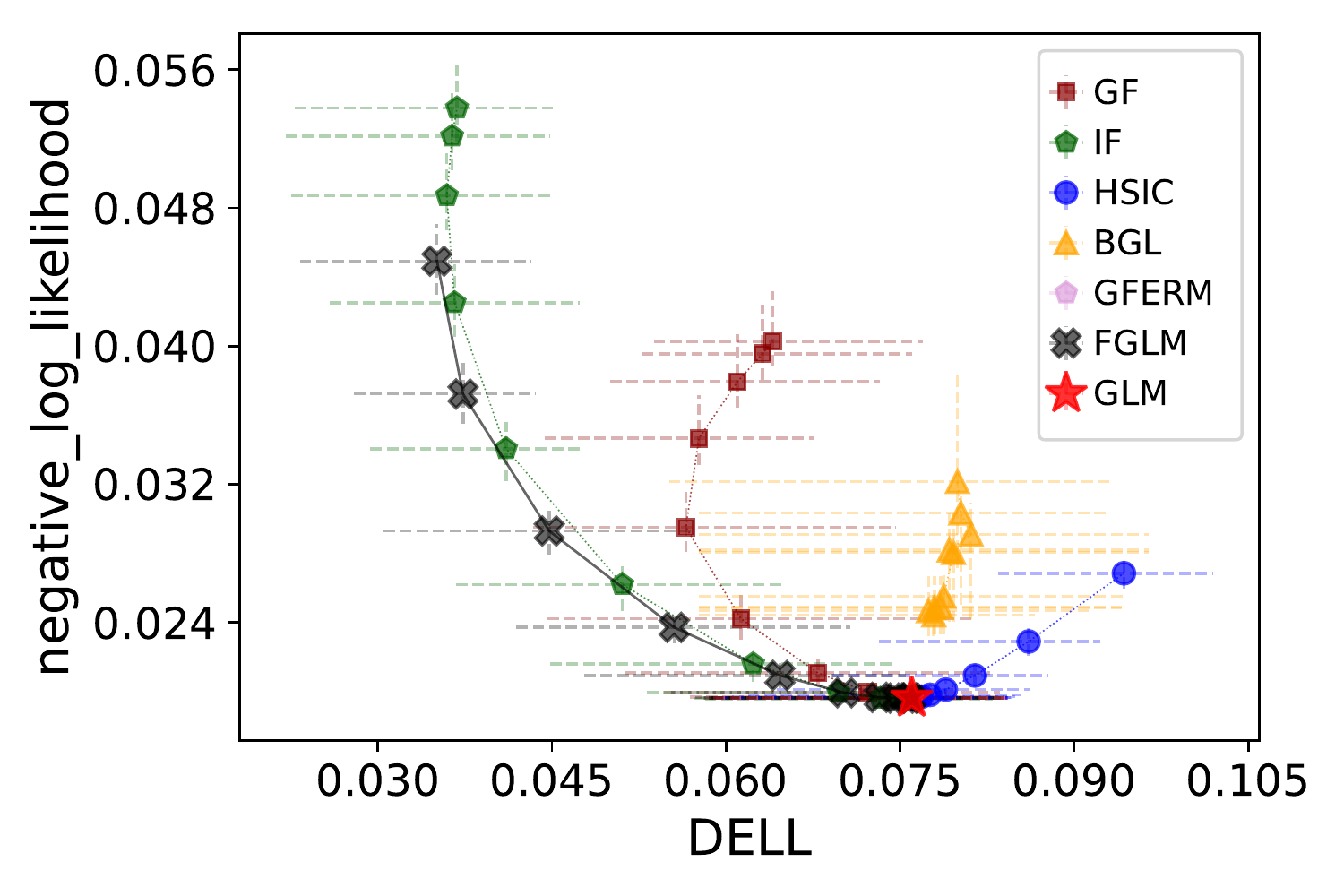}}
    \\
    \subfloat[LSAC--Cont--Race(5)]{\includegraphics[width=0.32\linewidth,keepaspectratio]{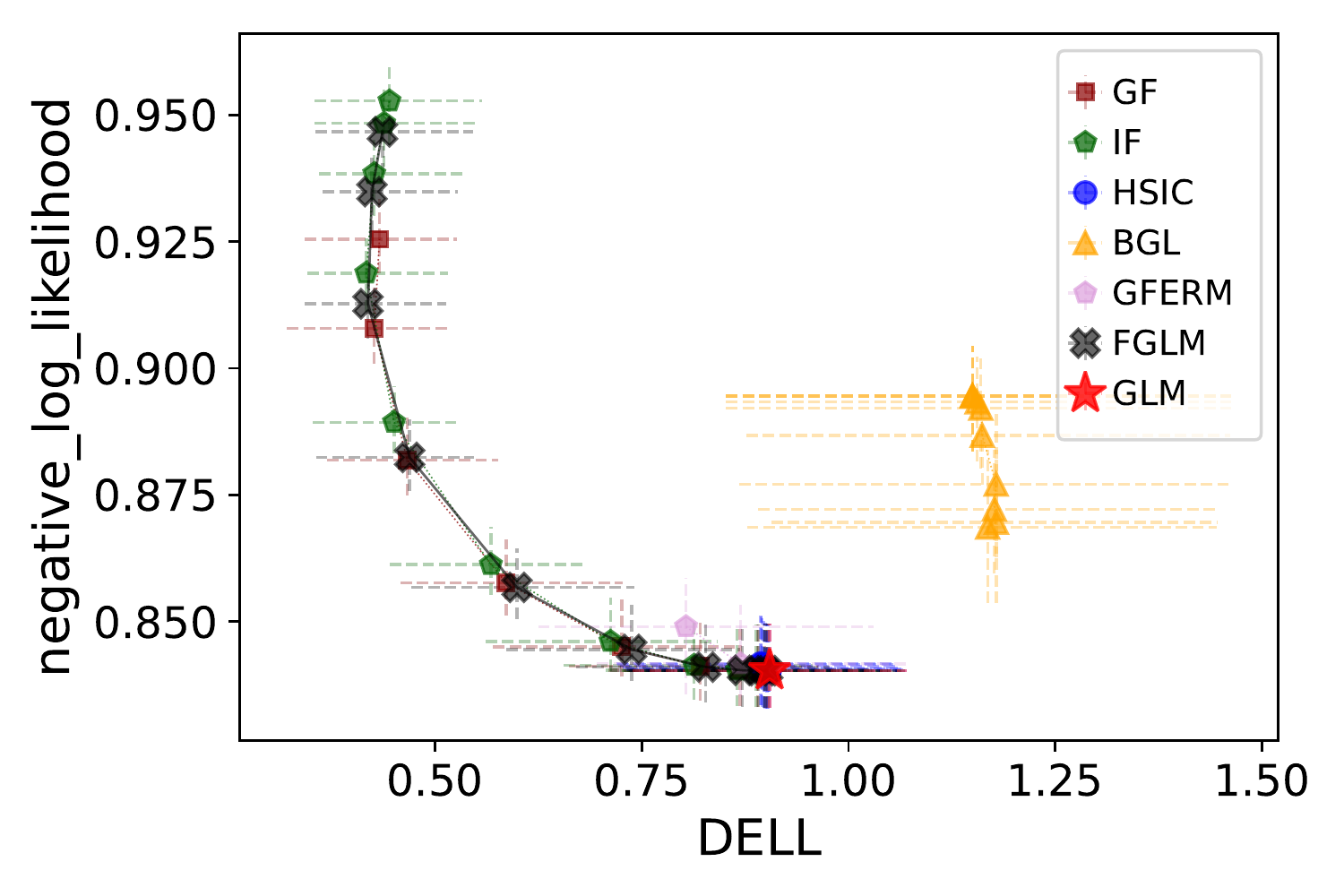}}
    \subfloat[Parkinsons--Cont--Gender(2)]{\includegraphics[width=0.32\linewidth,keepaspectratio]{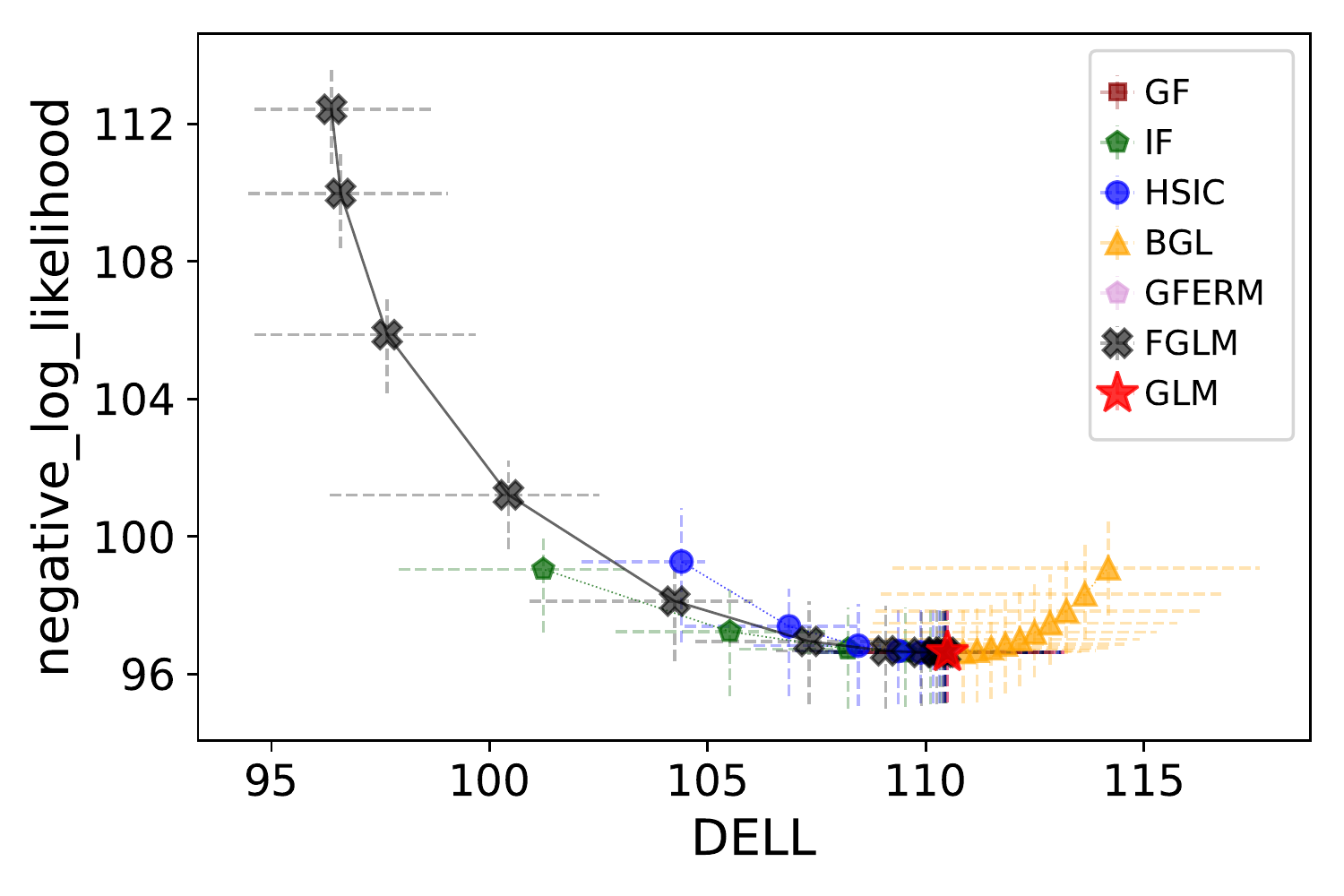}}
    \subfloat[Student--Cont--Gender(2)]{\includegraphics[width=0.32\linewidth,keepaspectratio]{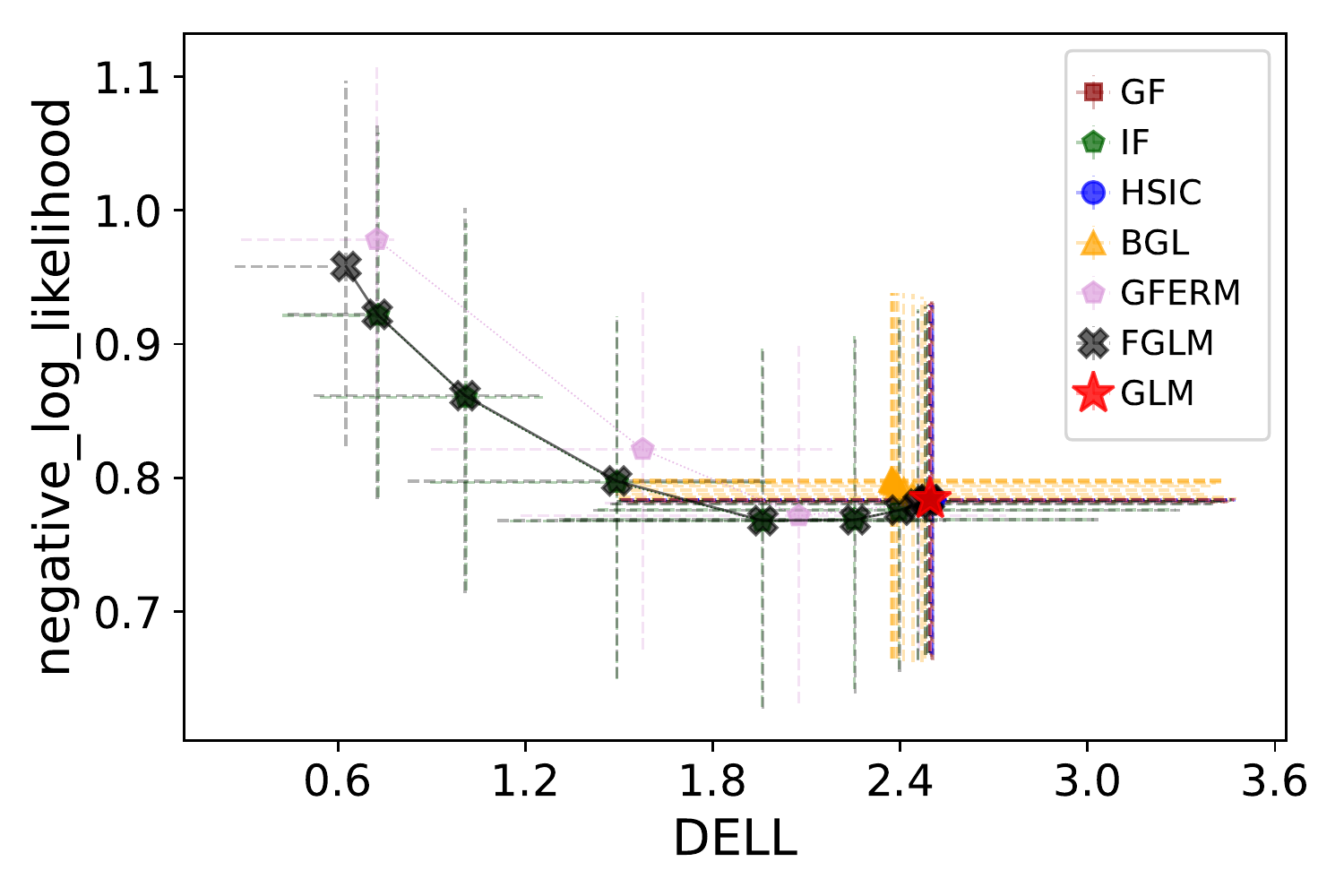}}
     \\
     \subfloat[Drug--Multi--Race(2)]{\includegraphics[width=0.32\linewidth,keepaspectratio]{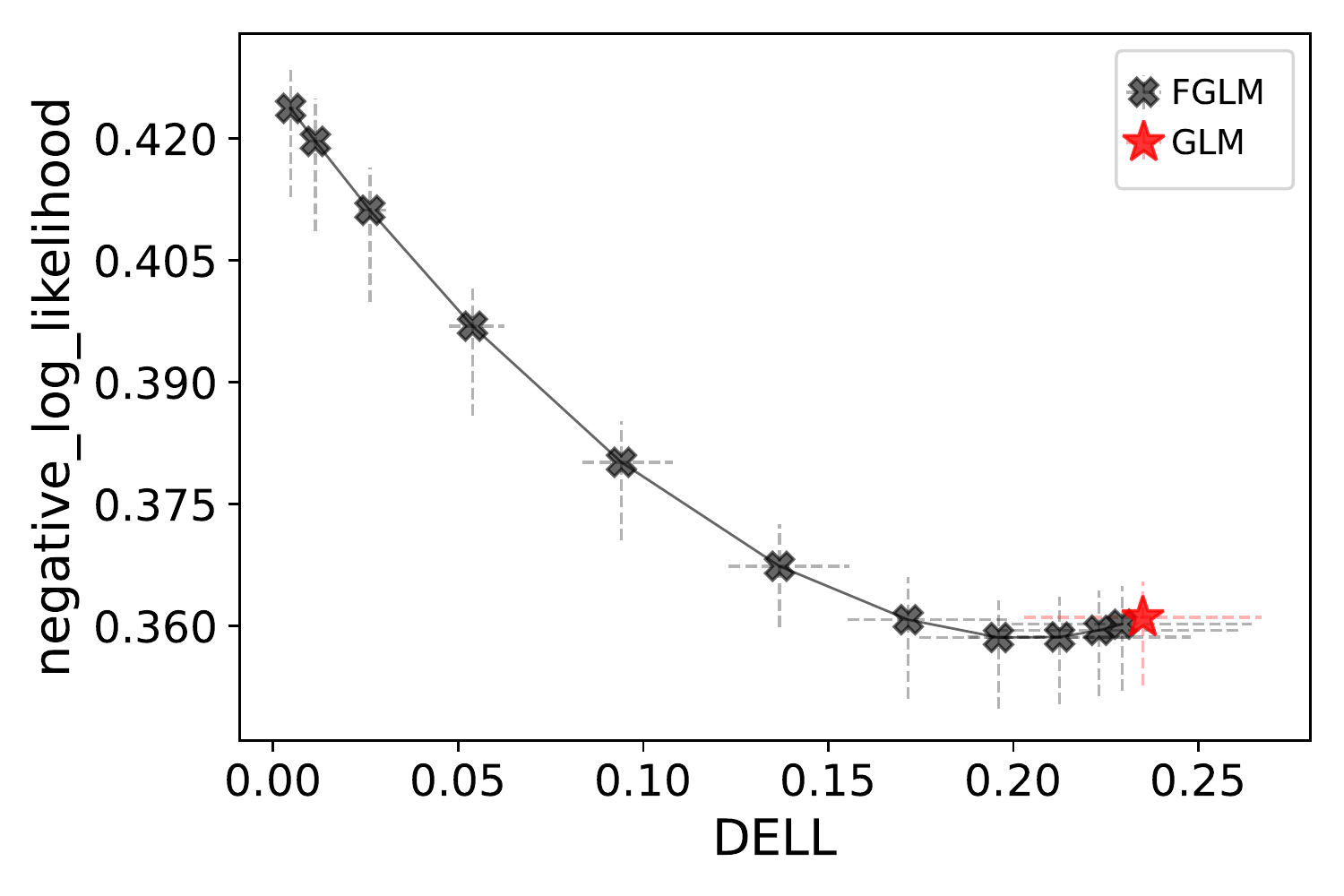}}
     \subfloat[Obesity--Multi--Gender(2)]{\includegraphics[width=0.32\linewidth,keepaspectratio]{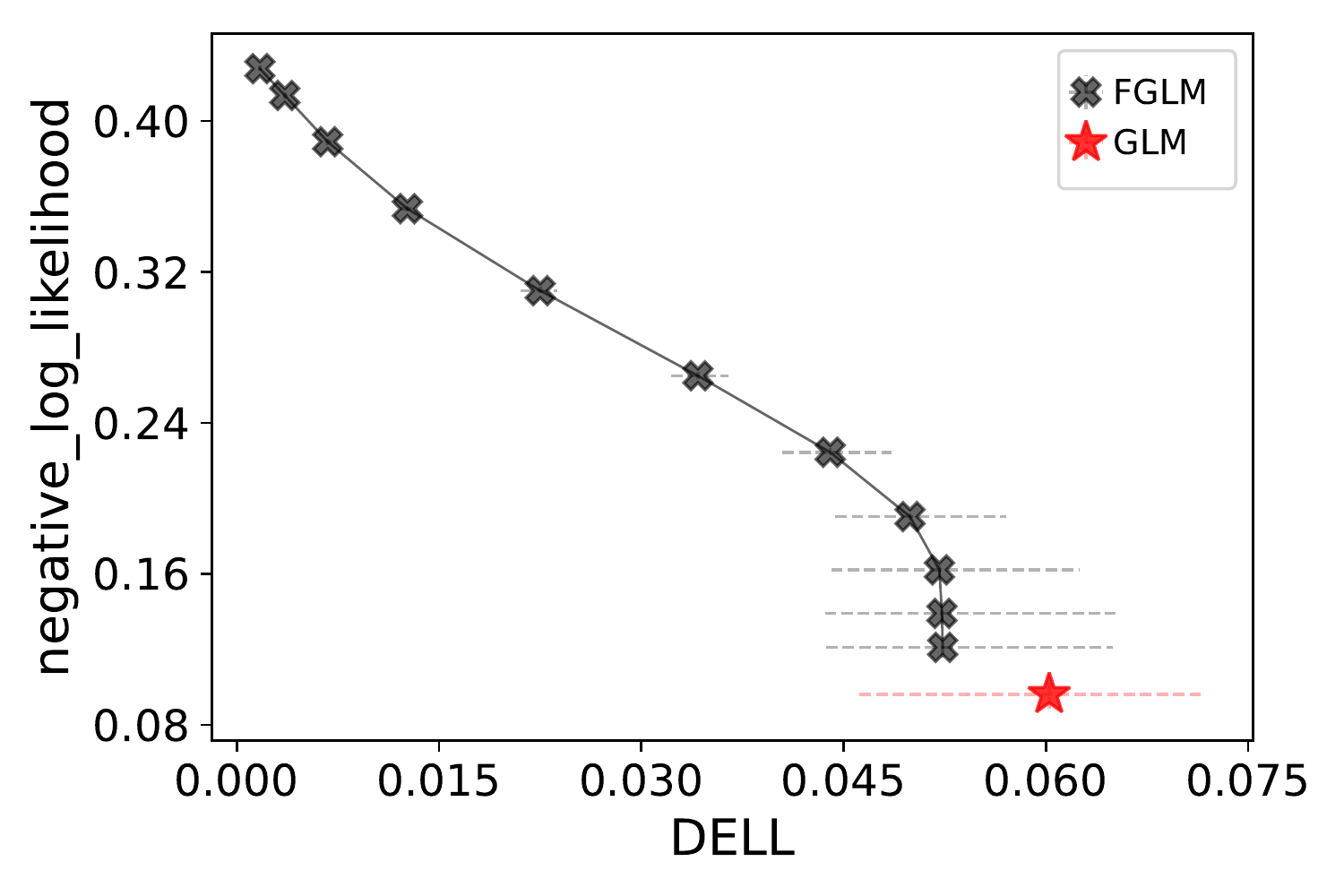}}
     \subfloat[HRS--Count--Race(4)]{\includegraphics[width=0.32\linewidth,keepaspectratio]{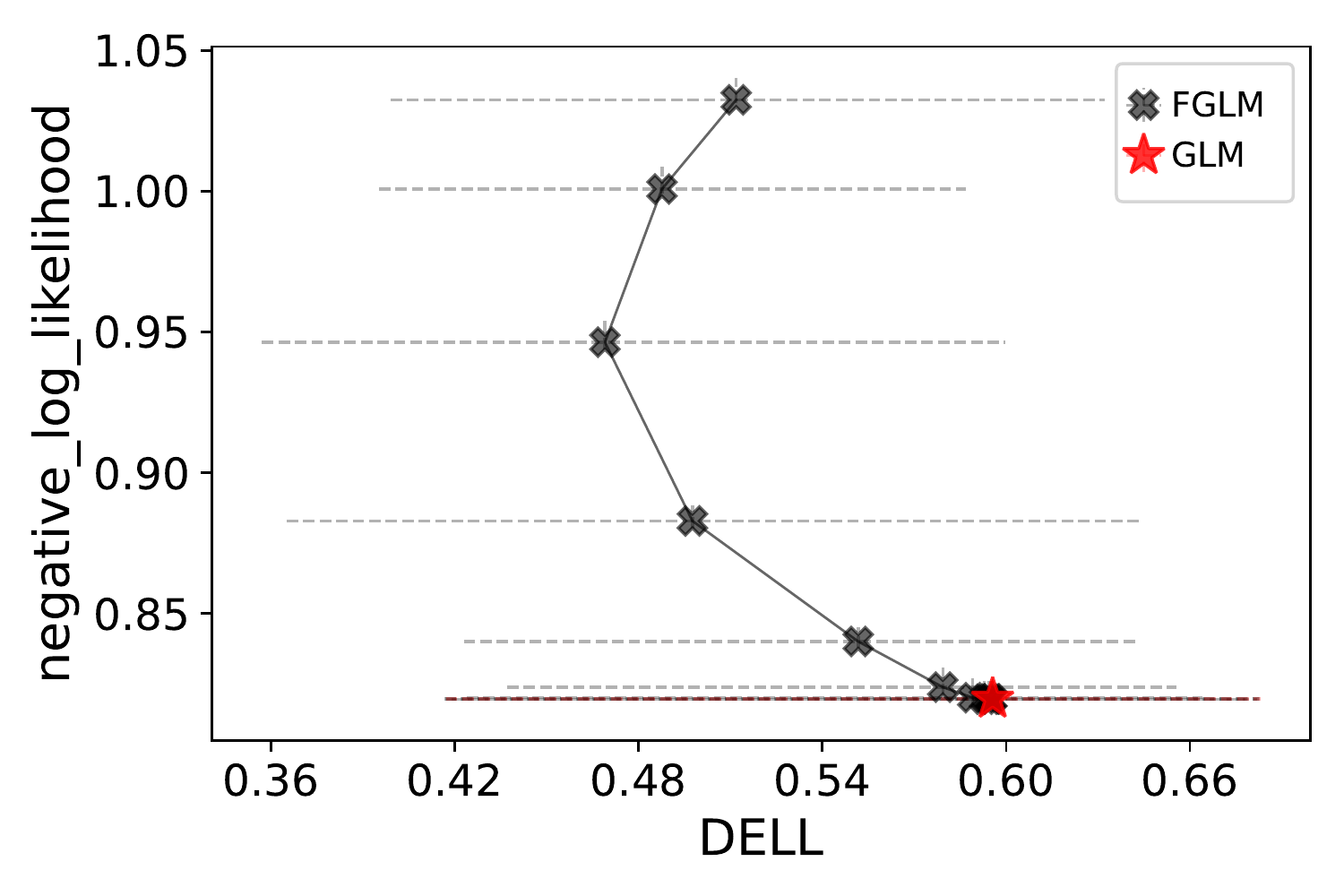}}
    \caption{Experimental results \emph{for negative log-likelihoods} and $\mathcal{D}_{\text{ELL}}$ for 11 real world datasets, with binary (a-e) and continuous outcomes (f-i). Each subtitle is in the form of {Dataset}--{Outcome Type}--{Sensitive Attribute($K$)}. For both binary and continuous outcomes we use a Generalized Linear Model (GLM, red star \inlinegraphics{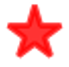}), Fair Generalized Linear Model (F-GLM, black X \inlinegraphics{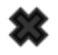}), Individual Fairness penalty (IF, green pentagon \inlinegraphics{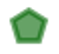}), Group Fairness penalty (GF, dark red square \inlinegraphics{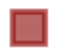}), and Bounded Group Loss (BGL, orange triangle \inlinegraphics{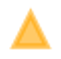}). Methods for binary outcomes also include the Support Vector Machine (SVM, grey hexagon \inlinegraphics{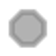}), Fair Constraints (FC, green diamond \inlinegraphics{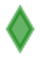}), Disparate Mistreatment (DM, blue circle \inlinegraphics{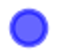}), Squared Difference Penalizer (SD, dark blue diamond \inlinegraphics{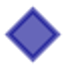}), Fair Empirical Risk Minimization (FERM, plum pentagon \inlinegraphics{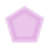}), Statistical Parity (SP, teal triangle \inlinegraphics{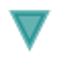}). Methods for continuous outcomes include the HSIC penalty (HSIC, blue circle \inlinegraphics{markers/DM-HSIC.png}), General Fair Empirical Risk Minimization (GFERM, plum pentagon \inlinegraphics{markers/GFERM.png}). See Table 1 for additional information for each method. Each dot represents the mean performance across test sets for a specific hyperparameter value $\lambda$ and the vertical and horizontal dotted lines reflect variation across test sets (IQRs) for each of performance and disparity.}
    \label{fig:nll-dell}
\end{figure*}

\section{Experiments and Results}\label{sec:experiments}

We performed  experiments for a comprehensive list of benchmark datasets to evaluate the proposed F-GLM, comparing it with the naive GLM and with multiple \emph{in-process} linear model-based fairness-aware methods. 

\subsection{Datasets and Fairness-Aware Methods}
We consider four different tasks/outcome types: binary classification (5 datasets), multiclass classification (2 datasets), continuous outcomes (4 datasets), and count outcomes  (1 dataset).  General characteristics of the datasets  are summarized in Table \ref{table:datasets}. For the binary classification and regression tasks, we evaluated the naive GLM, the proposed F-GLM, and the methods listed in Table 1. For count and multinomial outcome prediction tasks, we evaluated the naive GLM and the proposed F-GLM. 
Canonical link functions are used for the GLM and F-GLM: the identity function for normal continuous outcomes, logit function for binary outcomes, log function for count outcomes, and logit functions for multinomial outcomes.

\subsection{Evaluation Metrics}
Each method was evaluated in two aspects: (i) overall prediction performance measured by \emph{log-likelihood} and (ii) \emph{disparity of the log-likelihoods} between groups. Both metrics were computed using the test instances. Note that the disparity of the log-likelihoods is an empirical estimate of $\mathcal{D}_{\text{ELL}}$. Since there is not a clear consensus on the best choice for a prediction disparity measure in the fairness literature, we also   investigated   model performances for other disparity measures (including $\mathcal{D}_{\text{EO}})$ and included them in Appendix \ref{appendix:additional-results}. 

\subsection{Experimental Methods}
We randomly divided each dataset into  training (70\%) and testing (30\%) sets, except for the Adult dataset which has predefined train/test splits. Each model was trained on the training set by varying its fairness-related hyperparameter (if it exists) over a suitable range. Varying the hyperparameters in this manner produces trajectories of model performance that illustrate the trade-offs in prediction-disparity for each method. Evaluation for a range of such operating points is commonly done in the fairness literature rather than focusing on selecting a single hyperparameter value for each method. The performance and disparity measures were then estimated on the test dataset. For each value of the hyperparameter, the performance and disparity measures were estimated by averaging over 20 replicates of random splits of the training and testing sets (except for the Adult dataset).  

\subsection{Results}
The results for binary and continuous outcomes are displayed in Figure \ref{fig:nll-dell}, where the $x$ axis represents disparity in the group log-likelihoods and the $y$ axis is the overall log-likelihood, as measured on test sets. For most datasets, the naive GLM is the most unfair solution (largest value of the $x$ axis), which is expected due to the absence of any fairness constraint. Most of the fairness-aware methods we evaluated show wide-ranging trajectories of trade-offs between overall prediction performance and disparity. The trajectories of the competitive methods seem to be broadly consistent with earlier empirical results. The proposed F-GLM is generally one of the best performers relative to competitors for most datasets, in that it can decrease disparity substantially while maintaining overall predictive accuracy.  

The results for multiclass and count outcomes are also displayed in Figure \ref{fig:nll-dell}. For the Drug and Obesity datasets, the negative log-likelihoods of the underrepresented sensitive attribute group were 19\% and 43\% worse than those of the majority group. For the HRS dataset, the MSEs of the non-Hispanic black and Hispanic groups are 1.75 and 1.61 times of that of the non-Hispanic white subjects, respectively, highlighting the need for fairness-aware prediction algorithms in practice. Overall the proposed F-GLM results in trajectories that flexibly trade-off overall prediction accuracy with disparity. 

Results of the disparity and overall prediction performance using other metrics, including log-likelihoods and AUROCs, are presented in the supplementary materials, and showed mostly similar patterns as Figure \ref{fig:nll-dell}.

\section{Conclusions}\label{sec:conclusion}
We presented a fair generalized linear model (F-GLM), incorporating a convex penalty term based solely on the linear components of the GLM, in order to learn fair predictions. We provided statistical justification that the F-GLM achieves fairness both for the expected outcomes and log-likelihoods between groups. Thus, the framework is appealing both theoretically and computationally. Experimental results on benchmark datasets suggest that the F-GLM framework can improve prediction parity while maintaining overall accuracy for binary classification and regression, as well as for less-studied outcomes such as count and multiclass. 

One limitation of the F-GLM is its linearity; even though it provides good interpretability, its predictive power can be sub-optimal if the relationship between predictor and response variables is non-linear and complex. We conjecture that the F-GLM approach proposed here could be extended beyond linear models to provide a useful fairness framework for nonlinear machine learning models such as kernel machines or neural networks, by equalizing the linear components in the final decision layer.

An interesting extension of work would be to explore the connection between our fairness penalty and variance regularization, particularly when used for robust learning of predictive models or domain generalization.

\section*{Acknowledgements}
We thank the ICML reviewers for their suggestions on improving the original version of this paper. This work was supported by in part by the National Institutes of Health under awards NIH R01-LM013344, R01AG054467, R01AG065330 and R01-AG065330-02S1, by the National Science Foundation under award IIS-1900644, by the HPI Research Center in Machine Learning and Data Science at UC Irvine (PP), and by a Qualcomm Faculty award (PS). The Health and Retirement Study public use dataset is produced and distributed by the University of Michigan, Ann Arbor, MI. with funding from the National Institute on Aging (grant number NIA U01AG009740).

\newpage
\bibliography{references.bib}
\bibliographystyle{icml2022}

\newpage
\appendix
\onecolumn

\section{Proofs}\label{appendix:proof}

\subsection{Lemma \ref{lem:key-lemma}}
Since $h$ is differentiable, by the mean value theorem, we have $h(\theta^{k}) - h(\theta^{l}) = h'(\theta^{m})(\theta^{k} - \theta{^l})$, where $\theta^{m} = \alpha \theta^{k} + (1 - \alpha)\theta^{l}$, $\alpha \in [0,1]$. Thus, by applying Cauchy-Schwarz inequality, we obtain
\begin{align}\nonumber
    \left(\mathbb{E}\left[h(\theta^{k}) - h(\theta^{l})\right]\right)^{2} & = \left(\mathbb{E}\left[h'(\theta^{m}) (\theta^{k} - \theta^{l}) \right]\right)^{2}
    \\ \nonumber & \leq
    \mathbb{E}\left[h'(\theta^{m})^{2}\right]\mathbb{E}\left[(\theta^{k} - \theta^{l})^{2}\right].
\end{align}
Note that it is not necessary to assume $\theta^{k}$ and $\theta^{l}$ are independent. 

\subsection{Proposition \ref{prop:outcome}}
The inverse of the canonical link functions $\mu$ for GLMs are monotone and differentiable \citepsupp{supp_dobson2018introduction}. Let $\theta^{k} = \mathbf{X}^{ky}\boldsymbol{\beta}$, $\theta^{l} = \mathbf{X}^{ly}\boldsymbol{\beta}$, and $h = \mu$. Note that it is not necessary to assume $\mathbf{X}^{ky}$ and $\mathbf{X}^{ly}$ are independent. Applying Lemma 1 yields
\begin{align}\nonumber
    \left(\mathbb{E}\left[\mu(\mathbf{X}^{ky}\boldsymbol{\beta}) - \mu(\mathbf{X}^{ly}\boldsymbol{\beta})\right]\right)^{2} & = \left(\mathbb{E}\left[\mu'(\mathbf{X}^{m}\boldsymbol{\beta}) (\mathbf{X}^{ky}\boldsymbol{\beta} - \mathbf{X}^{ly}\boldsymbol{\beta}) \right]\right)^{2}
    \\ \nonumber & \leq
    \mathbb{E}\left[\mu'(\mathbf{X}^{m}\boldsymbol{\beta})^{2}\right]\mathbb{E}\left[(\mathbf{X}^{ky}\boldsymbol{\beta} - \mathbf{X}^{ly}\boldsymbol{\beta})^{2}\right],
\end{align}    
where $\mathbf{X}^{m} = \alpha \mathbf{X}^{ky} + (1-\alpha)\mathbf{X}^{ly}$, for some $\alpha \in [0,1]$. 

\subsection{Proposition \ref{prop:likelihood}}
Let 
\begin{equation}\nonumber
    h(\theta) = \frac{y\theta - b(\theta)}{a(\phi)} + c(y,\phi)  = \frac{y\mathbf{X}\boldsymbol{\beta} - b(\mathbf{X}\boldsymbol{\beta})}{a(\phi)} + c(y,\phi) = \ell(\boldsymbol{\beta};\mathbf{X},y) ,
\end{equation}
here, $y$ is a fixed value and $\theta = \mathbf{X}\boldsymbol{\beta}$. The log-likelihood is a concave and differentiable function of $\theta$. Thus, with $\theta^{k} = \mathbf{X}^{ky}\boldsymbol{\beta}$ and $\theta^{l} = \mathbf{X}^{ly}\boldsymbol{\beta}$, we can apply Lemma 1. That is,

\begin{align}\nonumber
    \left(\mathbb{E}\left[\ell(\boldsymbol{\beta};\mathbf{X}^{ky},y) - \ell(\boldsymbol{\beta};\mathbf{X}^{ly},y)\right]\right)^{2} & = \left(\mathbb{E}\left[\ell'(\boldsymbol{\beta};\mathbf{X}^{m},y) (\mathbf{X}^{ky}\boldsymbol{\beta} - \mathbf{X}^{ly}\boldsymbol{\beta}) \right]\right)^{2}
    \\ \nonumber & \leq
    \mathbb{E}\left[\ell'(\boldsymbol{\beta};\mathbf{X}^{m},y)^{2}\right]\mathbb{E}\left[(\mathbf{X}^{ky}\boldsymbol{\beta} - \mathbf{X}^{ly}\boldsymbol{\beta})^{2}\right],
\end{align}    
where $\mathbf{X}^{m} = \alpha \mathbf{X}^{ky} + (1-\alpha)\mathbf{X}^{ly}$, for some $\alpha \in [0,1]$. 

\subsection{Theorem \ref{thm:outcome}}\label{proof:thm-outcome}

If $\mu'$ is bounded (e.g. Bernoulli or Multinomial, see Table \ref{table:glm-links}), we can easily find $C_{\mu} = \sup (\mu')^{2}$; however, for some link functions (or outcomes) $\mu'$ is not bounded and in that case we cannot find $C_{\mu}$ in the same way. Instead, for that case, we use the fact that  $\mu'$ is \emph{monotonically increasing} and nonnegative to complete the proof. We first introduce a lemma.
\begin{lemma}\label{lem:b1}
    Let $\widehat{\boldsymbol{\beta}}_{\text{\normalfont FGLM}}$ be the solution for \eqref{eqn:expected-objective} given $\lambda \geq 0$. Then, 
    \begin{equation}
        \|\widehat{\boldsymbol{\beta}}_{\text{\normalfont FGLM}}\|_{2}^{2} \leq \frac{\delta_{\max}(\Delta)}{ \delta_{\min}(\Delta)} \|\widehat{\boldsymbol{\beta}}_{\text{\normalfont GLM}}\|_{2}^{2},
    \end{equation}
\end{lemma}
where $\delta_{\max}(\Delta)$ and $\delta_{\min}(\Delta)$ are the largest and the smallest nonnegative eigenvalues of $\Delta$. We further note that, by the Perron-Frobenius Theorem, $\delta_{\max}$ and $\delta_{\min}$ are the largest and the smallest row sums of $\Delta$.

\begin{proof}
    The lemma follows from the following chain of inequalities:
    \begin{align}
        \delta_{\min}(\Delta) \|\widehat{\boldsymbol{\beta}}_{\text{\normalfont FGLM}}\|_{2}^{2} & \leq \widehat{\boldsymbol{\beta}}_{\text{\normalfont FGLM}}^{T} \Delta \widehat{\boldsymbol{\beta}}_{\text{\normalfont FGLM}} \leq \widehat{\boldsymbol{\beta}}_{\text{\normalfont GLM}}^{T} \Delta \widehat{\boldsymbol{\beta}}_{\text{\normalfont GLM}} \leq \delta_{\max}(\Delta)\|\widehat{\boldsymbol{\beta}}_{\text{\normalfont GLM}}\|_{2}^{2}.
    \end{align}
    The left-most inequality follows from the eigenvalue decomposition of $\Delta$, that is,
    \begin{equation}\nonumber
        \widehat{\boldsymbol{\beta}}_{\text{\normalfont FGLM}}^{T} \Delta \widehat{\boldsymbol{\beta}}_{\text{\normalfont FGLM}} = \widehat{\boldsymbol{\beta}}_{\text{\normalfont FGLM}}^{T} Q^{-1} \Lambda Q \widehat{\boldsymbol{\beta}}_{\text{\normalfont FGLM}} \geq \delta_{\min}(\Delta) \widehat{\boldsymbol{\beta}}_{\text{\normalfont FGLM}}^{T} Q^{-1} Q \widehat{\boldsymbol{\beta}}_{\text{\normalfont FGLM}} = \delta_{\min}(\Delta) \|\widehat{\boldsymbol{\beta}}_{\text{\normalfont FGLM}}\|_{2}^{2},
    \end{equation}
    where $\Lambda$ is a diagonal matrix whose entries are the eigenvalues of $\Delta$. Since $\Delta$ is positive semi-definite, the entries of $\Lambda$ are all non-negative. Likewise, we have
    \begin{equation}\nonumber
        \widehat{\boldsymbol{\beta}}_{\text{\normalfont GLM}}^{T} \Delta \widehat{\boldsymbol{\beta}}_{\text{\normalfont GLM}} = \widehat{\boldsymbol{\beta}}_{\text{\normalfont GLM}}^{T} Q^{-1} \Lambda Q \widehat{\boldsymbol{\beta}}_{\text{\normalfont GLM}} \leq \delta_{\max}(\Delta) \widehat{\boldsymbol{\beta}}_{\text{\normalfont GLM}}^{T} Q^{-1} Q \widehat{\boldsymbol{\beta}}_{\text{\normalfont GLM}} = \delta_{\max}(\Delta) \|\widehat{\boldsymbol{\beta}}_{\text{\normalfont GLM}}\|_{2}^{2},
    \end{equation}
    which yields the right-most inequality. The inequality in the middle trivially holds because $\widehat{\boldsymbol{\beta}}_{\text{\normalfont FGLM}}$ and $\widehat{\boldsymbol{\beta}}_{\text{\normalfont GLM}}$ are optimal solutions for the F-GLM and the naive GLM problems.
\end{proof}

\begin{proof}[Proof of Theorem \ref{thm:outcome}]
    Here we only consider monotonically increasing $\mu'$ because otherwise (Bernoulli, multinomial, and normal) $\mu'$ is bounded and thus we can easily find the quantity that bounds $\mathbb{E}[\mu'(\mathbf{X}^{m}\widehat{\boldsymbol{\beta}}_{\text{\normalfont FGLM}})^{2}]$. We have a chain of inequalities that follows from Lemma \ref{lem:b1} as well as the eigenvalue decomposition: $$ \widehat{\boldsymbol{\beta}}_{\text{\normalfont FGLM}}(\mathbf{X}^{mT}\mathbf{X}^{m})\widehat{\boldsymbol{\beta}}_{\text{\normalfont FGLM}} \leq \delta_{\max}(\mathbf{X}^{mT}\mathbf{X}^{m}) \|\widehat{\boldsymbol{\beta}}_{\text{\normalfont FGLM}}\|_{2}^{2} \leq \delta_{\max}(\mathbf{X}^{mT}\mathbf{X}^{m}) (\delta_{\max}(\Delta)/\delta_{\min}(\Delta)) \|\widehat{\boldsymbol{\beta}}_{\text{\normalfont GLM}}\|_{2}^{2}.$$
    Therefore, $$|\mathbf{X}^{m}\widehat{\boldsymbol{\beta}}_{\text{\normalfont FGLM}}|= ({\widehat{\boldsymbol{\beta}}_{\text{\normalfont FGLM}}^{T}(\mathbf{X}^{mT}\mathbf{X}^{m})\widehat{\boldsymbol{\beta}}_{\text{\normalfont FGLM}}})^{1/2} \leq (\delta_{\max}(\mathbf{X}^{mT}\mathbf{X}^{m}) (\delta_{\max}(\Delta)/\delta_{\min}(\Delta)))^{1/2}\|\widehat{\boldsymbol{\beta}}_{\text{\normalfont GLM}}\|_{2},$$ which yields, $$\mathbb{E}[\mu'(\mathbf{X}^{m}\widehat{\boldsymbol{\beta}}_{\normalfont \text{FGLM}})^{2}] \leq \mathbb{E}[\mu'(|\mathbf{X}^{m}\widehat{\boldsymbol{\beta}}_{\normalfont \text{FGLM}}|)^{2}] \leq \mathbb{E}[\mu'((\delta_{\max}(\mathbf{X}^{mT}\mathbf{X}^{m}) (\delta_{\max}(\Delta)/\delta_{\min}(\Delta)))^{1/2}\|\widehat{\boldsymbol{\beta}}_{\text{\normalfont GLM}}\|_{2})^{2}].$$ This term is independent of $\lambda$ but depend on the predictors and responses.
\end{proof}

\subsection{Theorem \ref{thm:likelihood}}
\begin{proof}[Proof of Theorem \ref{thm:likelihood}]
    In the case of Bernoulli and multinomial, we can take $\sup_{x}(y-\mu(x))^{2}$ where $0 \leq \mu(x) \leq 1$; otherwise even if $\mu$ is unbounded, it is still monotonically increasing. Thus,
\begin{align}
    \mathbb{E}[\ell'(\widehat{\boldsymbol{\beta}}_{\text{\normalfont FGLM}};\mathbf{X}^{m},y)^{2}] & = a(\phi)^{-2}\mathbb{E}[y^{2} - 2y\mu(\mathbf{X}^{m}\widehat{\boldsymbol{\beta}}_{\text{\normalfont FGLM}}) + \mu(\mathbf{X}^{m}\hat{\boldsymbol{\beta}}_{\text{\normalfont FGLM}})^{2}] \\
    & \leq a(\phi)^{-2}(y^{2} + 2|y| \mathbb{E}[|\mu(\mathbf{X}^{m}\widehat{\boldsymbol{\beta}}_{\text{\normalfont FGLM}})|] + \mathbb{E}[\mu(\mathbf{X}^{m}\widehat{\boldsymbol{\beta}}_{\text{\normalfont FGLM}})^{2}]) \\ & \leq
    a(\phi)^{-2}(y^{2} + 2|y|\mathbb{E}[|\mu(|\mathbf{X}^{m}\hat{\boldsymbol{\beta}}_{\text{\normalfont FGLM}}|)|] + \mathbb{E}[\mu(|\mathbf{X}^{m}\widehat{\boldsymbol{\beta}}_{\text{\normalfont FGLM}}|)^{2}]) \\ & \leq a(\phi)^{-2}(y^{2} + 2|y|\mathbb{E}[|\mu((\delta_{\max}(\mathbf{X}^{mT}\mathbf{X}^{m}) (\delta_{\max}(\Delta)/\delta_{\min}(\Delta)))^{1/2}\|\widehat{\boldsymbol{\beta}}_{\text{\normalfont GLM}}\|_{2})|] \\ & \quad + \mathbb{E}[\mu(|(\delta_{\max}(\mathbf{X}^{mT}\mathbf{X}^{m}) (\delta_{\max}(\Delta)/\delta_{\min}(\Delta)))^{1/2}\|\widehat{\boldsymbol{\beta}}_{\text{\normalfont GLM}}\|_{2}|)^{2}]) \\ & = a(\phi)^{-2}\mathbb{E}[(|y| + \mu(|(\delta_{\max}(\mathbf{X}^{mT}\mathbf{X}^{m}) (\delta_{\max}(\Delta)/\delta_{\min}(\Delta)))^{1/2}\|\widehat{\boldsymbol{\beta}}_{\text{\normalfont GLM}}\|_{2}|)^{2}]
\end{align} 
\end{proof}

Moreover, we have an inequality $\delta_{\max}(\mathbf{X}^{m{\scriptscriptstyle T}}\mathbf{X}^{m}) \leq \delta_{\max}(\mathbf{X}^{ky{\scriptscriptstyle T}}\mathbf{X}^{ky}) + \delta_{\max}(\mathbf{X}^{ly{\scriptscriptstyle T}}\mathbf{X}^{ly}) + \delta_{\max}(\mathbf{X}^{ky{\scriptscriptstyle T}}\mathbf{X}^{ly} + \mathbf{X}^{ly{\scriptscriptstyle T}}\mathbf{X}^{ky})$ that allows us to remove $\mathbf{X}^{m}$ and $\alpha$ which are unknown.

\subsection{Lemma \ref{lem:consistency}}
For any $k, l \in \mathcal{A}$ and $y \in \mathcal{Y}$, we have a chain of inequalities
\begin{equation}
    \mathbf{D}^{kly} = \frac{1}{n^{kly}} \sum_{(i,j) \in \mathcal{S}^{kly}} (\mathbf{x}_{i} - \mathbf{x}_{j})^{T}(\mathbf{x}_{i} - \mathbf{x}_{j}),
\end{equation}
where $\mathcal{S}^{kly} = \{(i,j):y_{i} = y_{j} = y, A_{i} = k, A_{j} = l\}$, which is a set of samples drawn from the joint distribution of $(\mathbf{X}^{ky}, \mathbf{X}^{ly})$. Thus, as $n^{k}, n^{l} \to \infty$, $\mathbf{D}^{kly} \to \mathbb{E}[(\mathbf{X}^{ky} - \mathbf{X}^{ly})^{2}] = \mathbb{V}[\mathbf{X}^{ky} - \mathbf{X}^{ly}] + \mathbb{E}[\mathbf{X}^{ky} - \mathbf{X}^{ly}]^{2}$. Therefore, 
\begin{equation}
    \mathbf{D} = \frac{2\lambda}{|\mathcal{Y}|K(K-1)}\sum_{k,l\in \mathcal{A}}\sum_{y \in \mathcal{Y}} \mathbf{D}^{kly} \to \frac{2\lambda}{|\mathcal{Y}|K(K-1)} \sum_{k,l\in \mathcal{A}}\sum_{y \in \mathcal{Y}}\left(\mathbb{V}[\mathbf{X}^{ky} - \mathbf{X}^{ly}] + \mathbb{E}[\mathbf{X}^{ky} - \mathbf{X}^{ly}]^{2} \right),
\end{equation}
as $\min_{k} n^{k} \to \infty$. 

\subsection{Theorem \ref{thm:consistency}}
For the proof, we assume the two regularity conditions given in \citetsupp{supp_zou2006adaptive}:
\begin{enumerate}
    \item The Fisher information matrix $\mathcal{I}(\boldsymbol{\beta}^{*}) = \mathbf{\Sigma}^{-1} =  \mathbb{E}[b''(\mathbf{X}\boldsymbol{\beta}^{*})\mathbf{X}^{T}\mathbf{X}]$ is finite and positive definite.
    \item There is a sufficiently large enough open set $\mathcal{U}$ that contains the true $\boldsymbol{\beta}$ such that $\forall \boldsymbol{\gamma} \in \mathcal{U}$,
    $$ |b'''(\mathbf{X}\boldsymbol{\gamma)}| \leq M(\mathbf{X}) < \infty $$ and $$\mathbb{E}[M(\mathbf{X})|\mathbf{X}_{j}\mathbf{X}_{k}\mathbf{X}_{l}|] < \infty$$ for all $1 \leq j, k, l \leq p$.
\end{enumerate}
Note that these regularity conditions are considered to be \emph{mild} \citepsupp{supp_zou2006adaptive}.

Now define
\begin{equation}\nonumber
    V_{n}(\mathbf{u}) = F\left( \boldsymbol{\beta} + \frac{\mathbf{u}}{\sqrt{n}} \right) - F(\boldsymbol{\beta}),
\end{equation}
where $F$ is the objective function for the F-GLM. Then $V_{n}(\mathbf{u})$ is minimized at $\mathbf{u}=\sqrt{n}\left(\widehat{\boldsymbol{\beta}}_{\text{FGLM}} - \boldsymbol{\beta}^{*}\right)$. Using the Taylor series expansion, we can rewrite $V_{n}(\mathbf{u})$ as
\begin{align}\nonumber
    V_{n}(\mathbf{u}) = & -\sum_{i=1}^{n}(y_{i} - b'(\mathbf{x}_{i}\boldsymbol{\beta}^{*}))\frac{\mathbf{x}_{i}\mathbf{u}}{\sqrt{n}} + \sum_{i=1}^{n}\frac{1}{2}b''(\mathbf{x}_{i}\boldsymbol{\beta}^{*})\mathbf{u}^{T}\frac{\mathbf{x}_{i}^{T}\mathbf{x}_{i}}{n}
    \mathbf{u} + n^{-\frac{3}{2}} \sum_{i=1}^{n}\frac{1}{6}b'''(\mathbf{x}_{i}\tilde{\boldsymbol{\beta}}^{*})(\mathbf{x}_{i}\mathbf{u})^{3} \\ \nonumber &+  \lambda_{n}\left[\left(\boldsymbol{\beta}^{*} + \frac{\mathbf{u}}{\sqrt{n}}\right)^{T}\mathbf{D}\left(\boldsymbol{\beta}^{*} + \frac{\mathbf{u}}{\sqrt{n}}\right) - \boldsymbol{\beta}^{*T}\mathbf{D}\boldsymbol{\beta}^{*}\right],
\end{align}
where $\tilde{\boldsymbol{\beta}}^{*}$ is between $\boldsymbol{\beta}^{*}$ and $\boldsymbol{\beta}^{*}+\frac{\mathbf{u}}{\sqrt{n}}$.
Given the regularity conditions the first three terms converges to
\begin{equation}\nonumber
 \mathbf{u}^{T}\mathbf{W} + \frac{1}{2}\mathbf{u}^{T}\mathbf{\Sigma}\mathbf{u}
\end{equation}
in distribution, where $\mathbf{W} \sim \mathcal{N}(\mathbf{0},\mathbf{\Sigma})$. On the other hand, for the last term, we have
\begin{equation}\nonumber
    \lambda_{n}\left[\frac{2\mathbf{u}^{T}\mathbf{D}\boldsymbol{\beta}}{\sqrt{n}} + \frac{\mathbf{u}^{T}\mathbf{D}\mathbf{u}}{n}\right] \overset{p}{\to} 2\lambda_{0}\mathbf{u}^{T}\Delta\boldsymbol{\beta}^{*} + \mathbf{0},
\end{equation}\nonumber
provided $\lambda_{n}/\sqrt{n} \to \lambda_{0} \geq 0$ and $\mathbf{D} \to \Delta$ as $\min_{k} n^{k} \to \infty$. Thus, we have 
\begin{equation}\nonumber
    V_{n}(\mathbf{u}) \overset{d}{\to} V(\mathbf{u}) = \mathbf{u}^{T}\mathbf{W} + \frac{1}{2}\mathbf{u}^{T}\mathbf{\Sigma}\mathbf{u} + 2\lambda_{0}\mathbf{u}^{T}\Delta\boldsymbol{\beta}^{*}.
\end{equation}
Therefore,
\begin{equation}\nonumber
    \sqrt{n}\left(\widehat{\boldsymbol{\beta}}_{\text{FGLM}} - \boldsymbol{\beta}^{*}\right) = \underset{\mathbf{u}}{\text{\normalfont argmin}} ~ \mathcal{V}_{n}(\mathbf{u}) \overset{d}{\to} \underset{\mathbf{u}}{\text{\normalfont argmin}} ~ \mathcal{V}(\mathbf{u}).
\end{equation}\nonumber
Note that $\mathcal{V}(\mathbf{u})$ is minimized at $\mathbf{u} = -\boldsymbol{\Sigma}^{-1}(2\lambda_{0}\Delta\boldsymbol{\beta}^{*} + \mathbf{W})$.

\newpage

\section{Discretization}\label{appendix:discretization}

\begin{figure}[!b]
    \centering
    \subfloat[Crime--Continuous--Race(3)]{\includegraphics[width=0.32\linewidth,keepaspectratio]{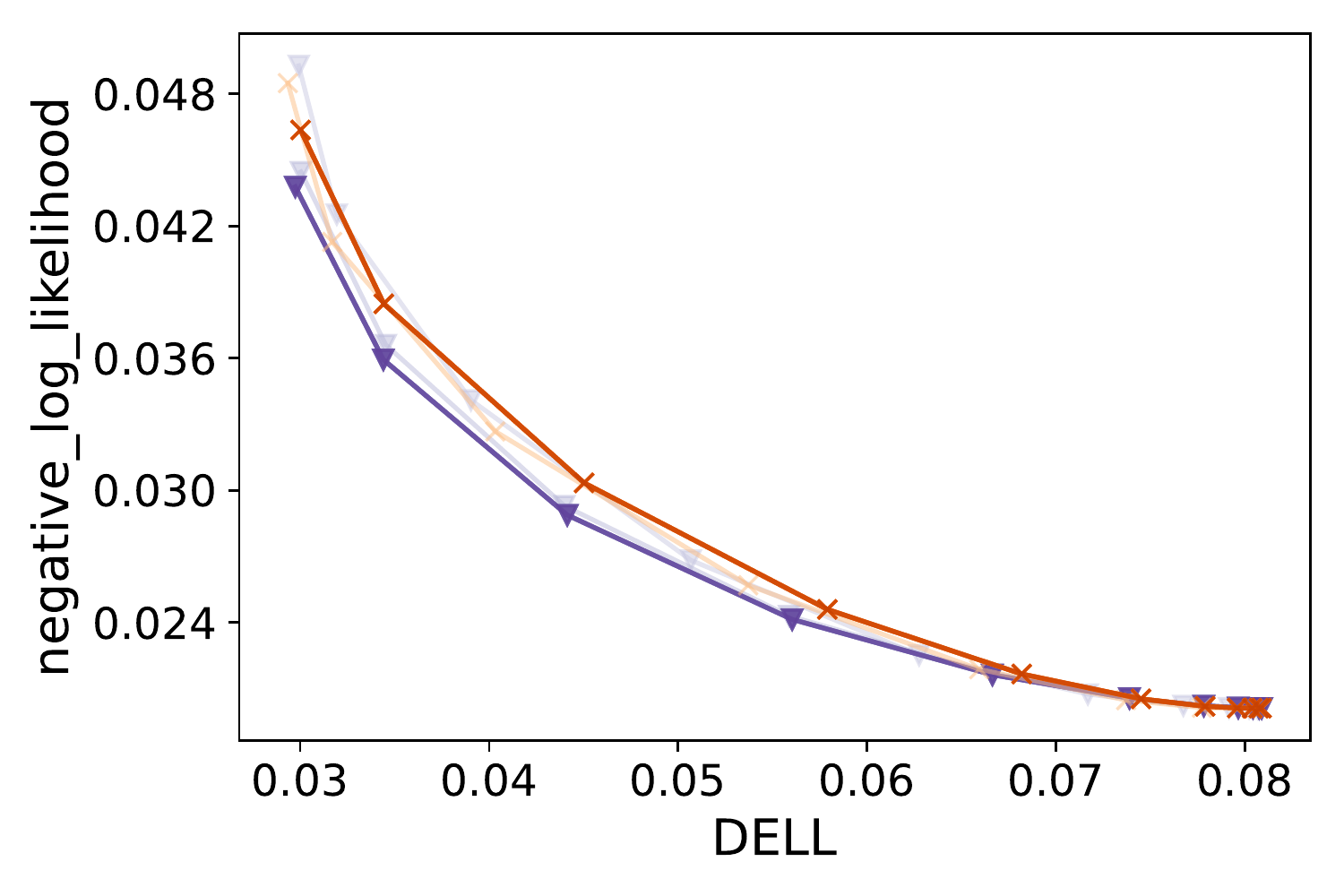}}
    \subfloat[Parkinsons--Continuous--Gender(2)]{\includegraphics[width=0.32\linewidth,keepaspectratio]{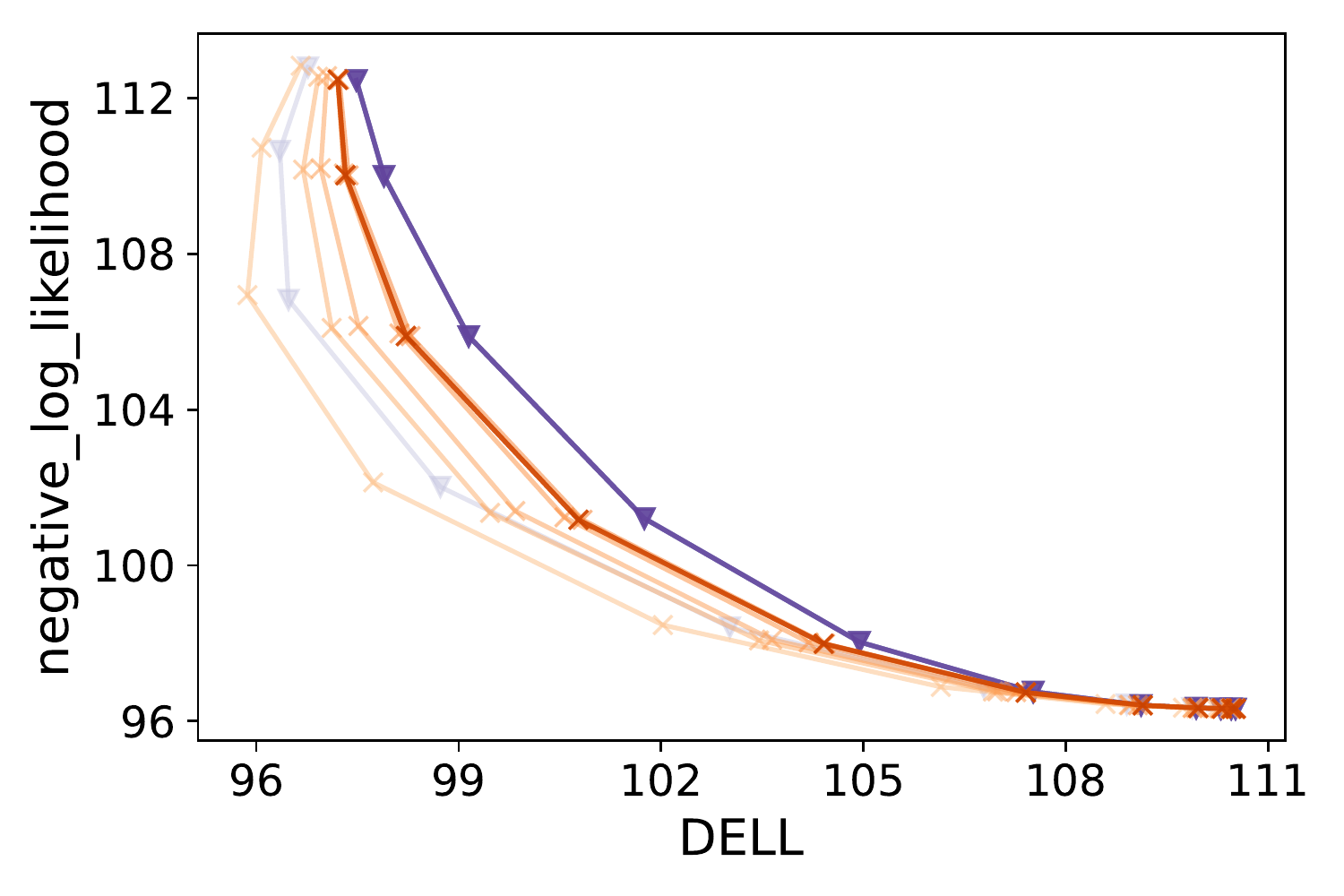}}
    \subfloat[Student--Continuous--Gender(2)]{\includegraphics[width=0.32\linewidth,keepaspectratio]{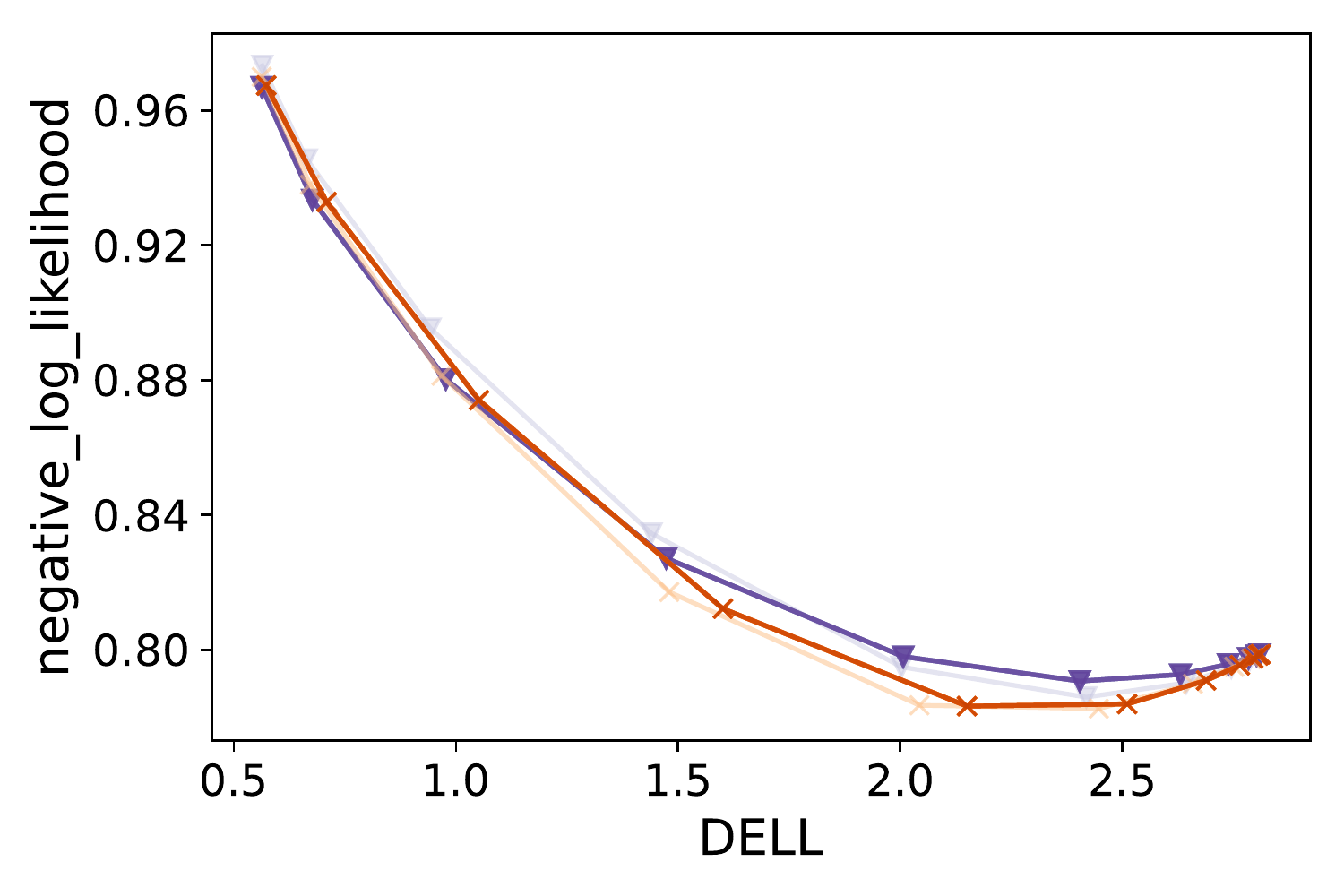}}
    \caption{Experimental results for equal counts ({\color{orange}orange X markers}) and equal lengths ({\color{Purple}purple triangle markers}) with various numbers of segments (ranging from 1 to 3 -- 25 depending on the dataset); a darker color means a greater number of segments. Our observation from the three datasets is that there is no straightforward relationship between the number of segments and better trade-off trajectories. However, the overall shape of the tradeoff trajectories between performance and disparity remains similar.}
    \label{fig:discretization}
\end{figure}

\subsection{Continuous Outcomes}
For continuous outcomes, i.e., the regression task, we investigated two different discretization strategies, which we refer to as \emph{equal counts} and \emph{equal lengths}. For the equal counts strategy, we construct segments $[\delta_{j},\delta_{j+1})$ which each include the same amount of samples, regardless of their group memberships, while the length of the segments are allowed to vary. In contrast, the equal lengths strategy makes each segment be the same length while the number of samples inside each segment can differ. Both strategies do not guarantee that each segment includes at least one sample from all groups which causes the penalty term to be undefined for some segments. To avoid this, we vary the number of segments starting from a large number and check if all the segments include at least one sample from all the groups. If not, we continually decrease the number of segments until we get a set of segments with each including at least one sample from all the groups. Algorithm 1 describes this discretization procedure for continuous outcomes. We  intuitively expect that a larger number of segments will provide better approximations. Thus, for the experiments, we set the max number of segments to 100. We found the equal counts based discretization results  (on average across 20 different splits of the training data) in 2.75, 25, and 8 segments for crime, parkinsons, and student datasets, respectively, while the equal lengths results  on average in 4.2, 7, and 5 segments, respectively.

We performed additional experiments to check if the F-GLM with continuous outcomes is sensitive to the number of segments. The results are summarized in Figure \ref{fig:discretization}. We  see that the choice of segments can change the performance-disparity trade-off trajectories; however, the overall patterns do not change much. 

We note that discretization of $y$ was not a primary focus of our study: further investigation is likely to be worthwhile, both from a theoretical perspective as well as investigating other algorithmic discretization strategies.

\subsection{Count Outcomes}
For count outcomes, i.e., the Poisson regression task, we find the smallest and the largest integers $L$ and $U$ satisfying $\{(\mathbf{x}_{i},y_{i},A_{i}):y_{i} = y, A_{i}=k)\} \neq \varnothing$ for all $k \in \mathcal{A}$ and $L < y < U$. Then, we set $y_{i} = \min\{y_{i}, L\}$ and $y_{i} = \max\{y_{i}, U\}$ for all $i$.

\section{Computational Complexity}\label{appendix:complexity}

\subsection{Preparing D}
Since $\mathbf{x}_{i}, \mathbf{x}_{j} \in \mathbb{R}^{1 \times p}$, the complexity of computing $(\mathbf{x}_{i} - \mathbf{x}_{j})^{T}(\mathbf{x}_{i} - \mathbf{x}_{j})$ is $\mathcal{O}(p^{2})$. Thus, the complexity to compute $\mathbf{D}^{kly}$ is $\mathcal{O}(n^{kly}p^{2})$. Moreover, we have $n^{kly} \leq n(n-1)/2$. Thus, the complexity of preparing $\mathbf{D}$ is $\mathcal{O}(n^{2}p^{2}K^{2}|\mathcal{Y}|)$.

\subsection{Newton-Raphson Iteration}
Computing the gradient consists of two matrix multiplication operations with complexities $\mathcal{O}(np)$ and $\mathcal{O}(p^{2})$. Moreover, the complexity of computing the Hessian is $\mathcal{O}(np^{2})$ provided $\mathbf{W}$ is a diagonal matrix. Also, inverting the Hessian is $\mathcal{O}(p^{3})$ since the Hessian is a dense $p \times p$ matrix. Therefore, the per-iteration complexity of the Newton-Raphson algorithm for the F-GLM is $\mathcal{O}(np^{2} + p^{3})$. \vspace{-1em}

\section{Datasets and Preprocessing Details}\label{sec:dataset}

\subsection{Adult Dataset}

The Adult dataset contains the \emph{income} records for 45,222 individuals from the 1994 census database, where the outcome \emph{income} is dichotomized into a binary variable (\emph{below \$50K in income versus above \$50K in income}). Each record is composed of the outcome and 14 predictors, among which 8 are categorical and 6 are continuous. 

The variable \emph{gender} was considered a sensitive attribute in this dataset (Male: 67\%, Female = 33\%). Other covariates of interest included age, professional occupation, education level, marital and relationship status, capital gains and losses, ethnicity, and country of origin.

Note that due to missing data the total of 48,842 instances was decreased to 45,222 after filtering out incomplete records.

\subsection{Arrhythmia Dataset}

The Arrhythmia dataset contains the \emph{presence of arrhythmia status} for 418 individuals, which we treat as a binary outcome (\emph{presence vs absent}). Each record is composed of the outcome and 80 predictors of interest.

The variable \emph{sex} was considered a sensitive attribute in this dataset (Male: 53\%, Female = 47\%). Other attributes included variables such as age, height, weight, QRS duration among others.

Note that the original sample size of 452 was reduced down to 418 records due to the removal to all samples with missing data and 2 individuals with nonsensical height values.

\subsection{COMPAS Dataset}

The COMPAS dataset contains records for 6,172 criminal defendants across the United States of America, where we use whether each defendant became a recidivist within 2 years of the first offense as outcome (\emph{was a recidivist within 2 years vs was not}). Each defendant has 10 predictor variables, including \emph{gender} and \emph{race/ethnicity}. The latter variables were considered as sensitive features and have imbalanced distributions across defendants. (Sex: Female: 19\%, Male: 81\%, we set Female as the baseline category. Race: Caucasian: 34\%, African-American: 52\%, Hispanic: 8\%, Other: 6\% which contain Asian and Native-American ethnicities, we set Caucasian as the baseline).

The dataset further contains 1 categorical variable, degree of the charge  (F: 0.36, M: 0.64, we set F as the baseline), 5 continuous variables (age in years with mean = 34.5, number of priors counts with mean = 3.2, juvenile felony counts with mean = 0.06, juvenile misconduct counts with mean = 0.1, juvenile other category counts with mean = 0.11) and 2 time variables (time in jail (days) with mean = 15, time in custody (days) with mean = 35).

Note that the original dataset contained 7,214 records. However to ensure data quality we removed records that had a charge date of a defendant's COMPAS score crime that was not within 30 days from when the person was arrested, under the assumption that this is not the correct offense for this record. We further removed records that had missing fields for recidivism or the degree of the charge of interest, resulting in a total sample size of 6,172 with complete observation data.

\subsection{Drug Consumption Dataset}
The drug consumption dataset contains records for 1,885 respondents and each respondent has 12 predictor variables, including gender and race/ethnicity. Participants were questioned concerning their use of 18 legal and illegal drugs and answered with one of the following seven categories: \emph{never used, used over a decade ago, used in the last decade, used in last year, used in last month, used in last week, and used in last day}. 

Thus, we can define three different tasks based on the participant's response: binary classification of classifying \emph{never used} versus \emph{the others (ever used)} and both ordered and unordered multiclass classification classifying the original seven categories of the outcomes. For the binary classification we used \emph{methadone}.
Further, we choose to use \emph{methamphetamine} as a response variable,  one of the most addictive and widely used drugs. 

The dataset consists of 4 categorical and 8 continuous predictors. All the continuous predictors were standardized to have zero mean and unit variance a priori, by the data provider, so we did not apply any transformations. We applied one-hot encoding to all the categorical predictors.

There are two sensitive attributes: \emph{gender} (Female: 50\%, Male: 50\%) and \emph{race/ethnicity} (Asian: 1.38\%, Black: 1.75\%, Mixed-Black/Asian 0.16\%, Mixed-White/Asian: 1.06\%, Mixed-White/Black: 1.06\%, Other 3.34\%, and White: 91.25\%). Because of the severe imbalance in race/ethnicity, we merged all the non-White race/ethnicity into a single non-White category. 

\subsection{German Credit Dataset}

The German Credit dataset contains records for 1000 individuals, describing the level of risk for their \emph{credit}, which we treated as binary (\emph{good vs bad}). Each record is composed of the outcome and 21 predictor variables among which 14 were categorical and 7 were continuous. 

The variable \emph{sex} was considered a sensitive attribute in this dataset (Male:  69\%,  Female = 31\%). Other covariates of interest included the status of checking account, credit history, saving in accounts and bonds, employment status, property ownership, disposable income and age among others.

There were no reduction  from the original sample size due to missingness, since all records had complete information available.

\subsection{Communities and Crime Dataset}

The Communities and Crime dataset contains the criminal records for 1,994 communities in the United States of America from socio-economic data from 1990 US Census and law enforcement data from the 1990 US LEMAS survey. The outcome of interest is the \emph{violent crimes per population} (continuous). Each record is composed of the outcome and 31 predictors, among which is the sensitive attribute \emph{race} (stratified between Asian: 4\%, Black: 11\%, Hispanic: 6\%, White: 79\%). Other predictors included age, income, urbanism, police budget among others. 

Note that the original sample size was significantly reduced due to high levels of data missingness across predictors. Moreover, the variables, state, county, community, community name and the fold for cross-validation were removed as they serve no purpose for prediction.

\subsection{Law School Admission Council Dataset}

The Law School Admission Council (LSAC) dataset contains records for 22,407 Law School students gathered by a National Longitudinal Study primarily undertaken in response to  {reports suggesting bar passage rates were lower among examinees of color}. The outcome of interest is the \emph{Grade Point Average} of students during Law School which is a continuous variable. We consider both the \emph{race/ethnicity} and the \emph{gender} of students to be sensitive factors. (Gender: Female = 44\%, Male = 56\%, we set Female as the baseline category. Race: White/Caucasian: 88.2\%, African-American: 6\%, Asian: 4\%, we group all other ethnicities under Other: 1.8\%. We set White/Caucasian as the baseline category). We further included two continuous variables as predictors, namely the LSAT score of each student (median 37, IQR 33-41) and the university GPA of students (median 3.2, IQR 3-3.5) and a single categorical variable specifying if students are participating in the academic program at full or part time (part time: 7.7\%, with full time being the baseline category).

After removing observations containing missing data, the final dataset contained records for 22,368 students with complete information.

\subsection{Parkinson's Telemonitoring Dataset}

The Parkinson's Telemonitoring dataset contains the record of 42 patients with early-stage Parkinson's disease recruited through a six-month trial of telemonitoring for remote symptom progression monitoring. The dataset includes 5,875 instances of data observations across all patients with outcome \emph{Unified Parkinson's Disease Rating Scale (UPDR) score} which is a \emph{continuous} value evaluating various aspects of Parkinson's disease. 

The sensitive attribute of the dataset was set to be the \emph{sex} of patients (Female: 33\%, Male: 67\%). We further included 16 predictors. All records had complete information and thus there was no sample size reduction due to missingness.

\subsection{Student Performance Dataset}

The Student dataset contains records for 382 students for 2 separate classes (mathematics and Portuguese) over 3 trimesters, we use the grades students received in the third trimester  (numeric score between 0 and 20). We separate this dataset into two separate datasets, with mathematics scores and Portuguese scores respectively. The sensitive variable in this dataset is the \emph{sex} of students. (Sex: Female: 48\%, Male: 52\%, we set Female as the baseline). The dataset share the same set of 25 further predictors, among which 3 are continuous and 21 are categorical. Note that some categorical variables were further collapsed to eliminate smaller categories (such as  the education level of the mother and father, the travel and study times, family relationship statuses and free time levels of students).

\subsection{Health \& Retirement Survey Dataset}

The University of Michigan Health and Retirement Study (HRS) longitudinal dataset, recording  survey responses on health and aging. The dataset contains 12,744 instances. The \emph{number of dependencies in daily activities} was set as the target outcome as a \emph{count variable}. This is encoded as the \emph{score} in the dataset, ranging from 0 to 10. 

The \emph{ethnicity} of patients was set as the sensitive attribute of the dataset (Afro-American: 15\%, Hispanic: 10\%, Other: 2\% and White: 73\%). The 22 predictor variables included gender, marital status, age, education and net worth among others. Note that large portions of the data entries were missing and we considered only complete cases.

\subsection{Obesity Dataset}

The Obesity dataset contains the health records of 2,111 individuals with their assessed level of \emph{obesity}. \emph{Obesity} is treated a multilevel outcome with levels: \emph{Insufficient weight} (13\%), \emph{normal weight} (13\%), \emph{overweight level 1} (14\%), \emph{overweight level 2} (14\%), \emph{obesity type I} (17\%), and \emph{obesity type II/III} (29\%). 

The \emph{gender} of patients was set as the sensitive attribute of the dataset (Female: 49\%, Male: 51\%), and we further included 14 predictor variables such as age, family history and smoking status. Note that there were no missing data in this dataset, all individuals had complete information.

\section{Experimental Setting Details}
\subsection{Implementation of Competitive Methods}
For the fair constraints \citepsupp{supp_zafar2017fc} and disparate mistreatment \citepsupp{supp_zafar2017dm} methods, we used the Python code provided by the authors\footnote{\url{https://github.com/mbilalzafar/fair-classification}}. For the squared difference penalty \citepsupp{supp_bechavod2017penalizing}, and the group and individual fairness convex penalty \citepsupp{supp_berk2017convex}, we adapted our Newton-Raphson method (all three methods can be expressed in the same form as that of F-GLM using a different $\mathbf{D}$.) Note that both papers suggested using CVXPY \citepsupp{supp_diamond2016cvxpy,supp_agrawal2018rewriting}, which is an off-the-shelf optimization solver, for solving their problems. We also implemented the HSIC penalty \citepsupp{supp_perez2017fair} which can easily be solved because the problem has a closed form solution for the linear case. For the linear FERM method, we used its Python implementation provided by the authors\footnote{\url{https://github.com/jmikko/fair\_ERM}}. We used the \texttt{fairlearn}\footnote{\url{https://fairlearn.org/}} Python package for the reductions approach with statistical parity or bounded group loss \citepsupp{supp_pmlr-v80-agarwal18a,supp_pmlr-v97-agarwal19d}. Specifically, we used grid search (\texttt{GridSearch} function) instead of the exponentiated gradient method to get a single model. For the general FERM \citepsupp{supp_oneto2020general}, since we did not find any available code online, we implemented the method with CVXPY. Note that the authors suggested using CPLEX\footnote{\url{https://www.ibm.com/analytics/cplex-optimizer}}, which is an off-the-shelf optimization solver.

\subsection{Hyperparameters}
We presented the range of the hyperparameters used for our experiments in Table \ref{table:hyperparameters}. For some datasets, we used slightly different range of hyperparameters for some methods. Details can be found in our code. 

\begin{table}[ht]
\setlength{\tabcolsep}{8pt}
\renewcommand{\arraystretch}{1.08}
\caption{The range of the hyperparameters that control accuracy-fairness trade-off used for the experiments}
\label{table:hyperparameters}
\centering
    \begin{tabular}{lccc}
    \hline\hline
    Methods                  & hyperparameter & min value & max value \\
    \hline
    Generalized Linear Model &  -             & -         &   -        \\
    Linear SVM$^{*}$         &  -             & -         &   -        \\
    Fair Constraints         & $c$            & $10^{-3}$         &  20        \\
    Disparate Mistreatment   & $c$            & $10^{-3}$         &  20        \\
    Squared Difference       & $\lambda$      & $10^{-3}$ &  10        \\
    Group Fairness           & $\lambda$      & $10^{-3}$ &  10        \\
    Individual Fairness      & $\lambda$      & $10^{-3}$ &  10        \\
    HSIC Penalty             & $\mu$          & $10^{-3}$ &   5        \\
    FERM$^{*}$               & $\epsilon^{**}$& 0         &   0        \\
    GFERM                    & $\epsilon$     & 0         & 100        \\
    Statistical Parity       & $w^{***}$      & 0         &   1        \\
    Bounded Group Loss       & $w^{***}$      & 0         &   1        \\
    Fair GLM                 & $\lambda$      & $10^{-3}$ &  10        \\
    \hline\hline
    \end{tabular}
    \\
    {\raggedright \small
    \hspace{8em}
    *We used $\nu \in \{0.05, 0.01\}$ for SVM and FERM. \\
    \hspace{8em}
    **The code provided by the authors allows only $\epsilon=0$.  \\
    \hspace{8em}
    ***We varied \texttt{constraint\_weight} parameter of \texttt{GridSearch} function. \hfill}
\end{table}

\section{Additional Experimental Results}\label{appendix:additional-results}
We provide additional plots that summarize the experimental results here. Figure \ref{fig:results3} shows the overall performance and disparity in \emph{mean squared error} (for binary and multiclass classification)--also referred to as \emph{Brier score}). The overall patterns are quite similar to those in Figure 1 in the main paper, supporting the conclusion that the F-GLM can produce favorable performance-disparity trajectories. For regression tasks, the negative log-likelihoods and mean squared errors are equivalent. 

Figure \ref{fig:results4} shows the performance and disparity measured for miscellaneous task-specific metrics; AUROC for binary classification, mean absolute error (MAE) for regression, misclassification rate for multiclass classification, and MAE for Poisson regression. We note that AUROC is the only higher-the-better metric so we plotted 1-AUROC instead to be consistent with other metrics. Also, AUROC cannot be calculated for each class because it is a concordance score, so it was not separately calculated for each class. Here the results are more mixed than in Figure \ref{fig:nll-dell}, in Figure \ref{fig:nll-deo} or in Figure \ref{fig:results3} (in terms of comparing F-GLM with other methods), but this is to be expected since performance criteria such as AUROC are not necessarily highly correlated with the other performance metrics.

\begin{figure*}[p]
    \centering
    \subfloat[Adult--Binary--Gender(2)]{\includegraphics[width=0.32\linewidth,keepaspectratio]{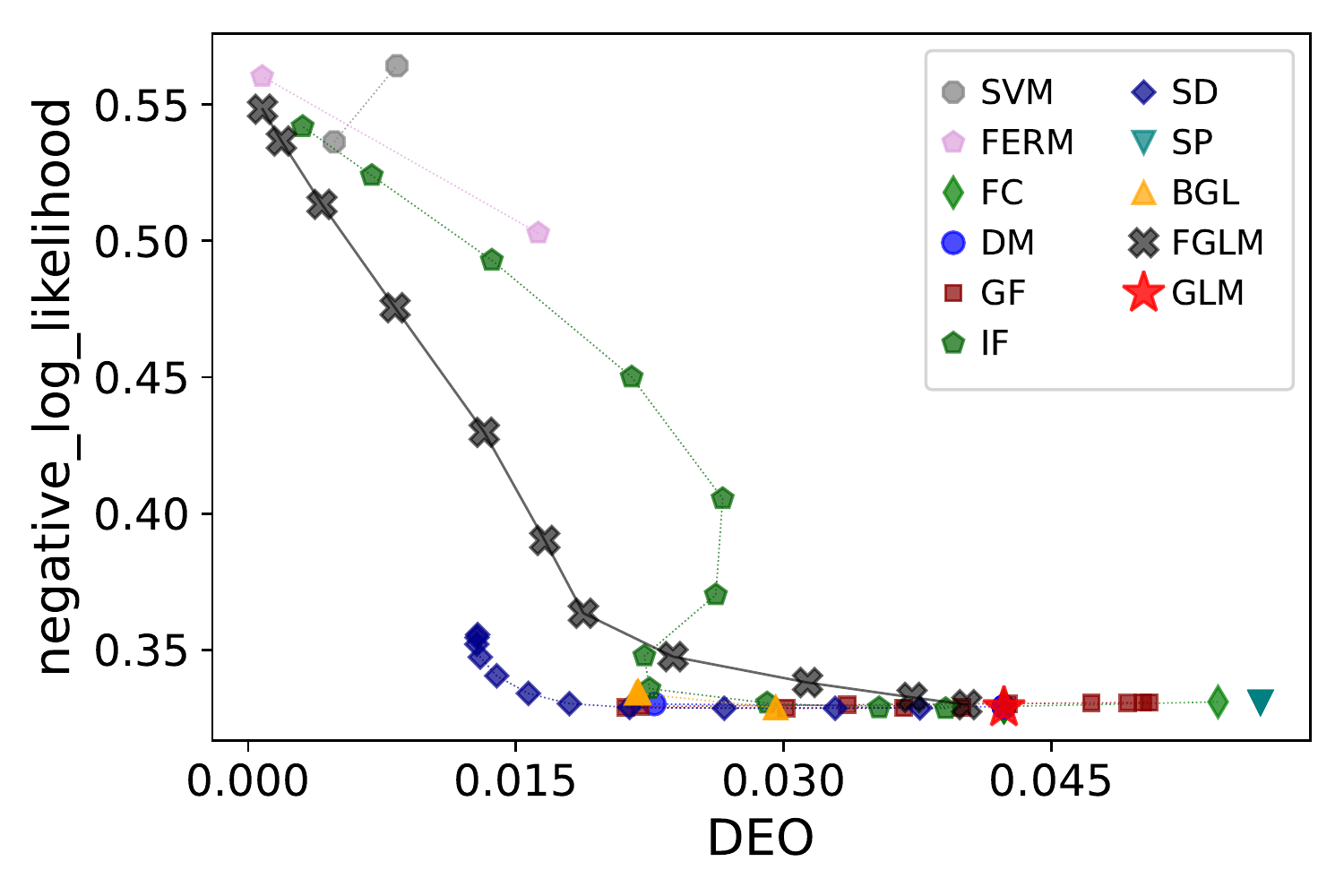}}
    \subfloat[Arrhythmia--Binary--Gender(2)]{\includegraphics[width=0.32\linewidth,keepaspectratio]{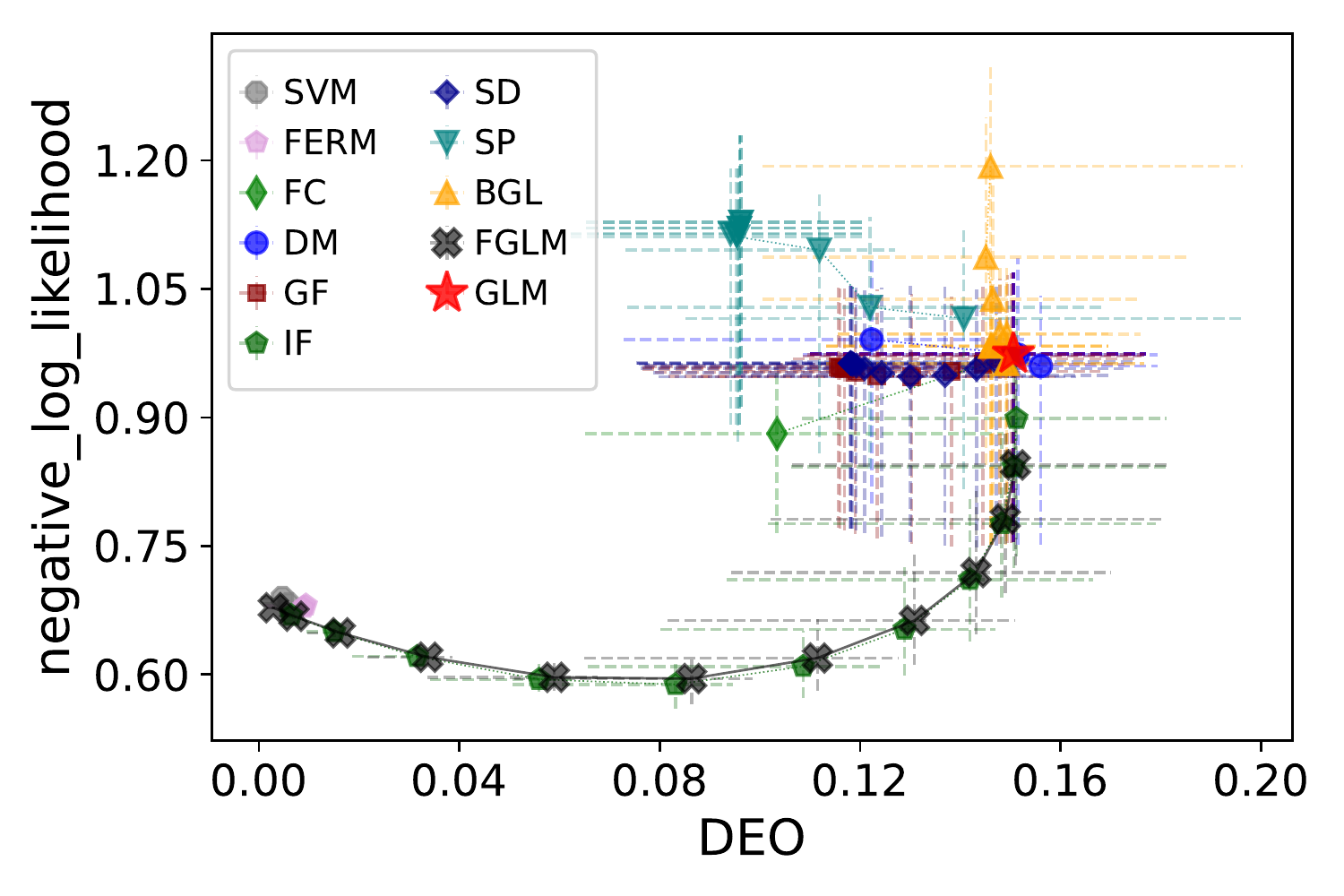}}
    \subfloat[COMPAS--Binary--Race(4)]{\includegraphics[width=0.32\linewidth,keepaspectratio]{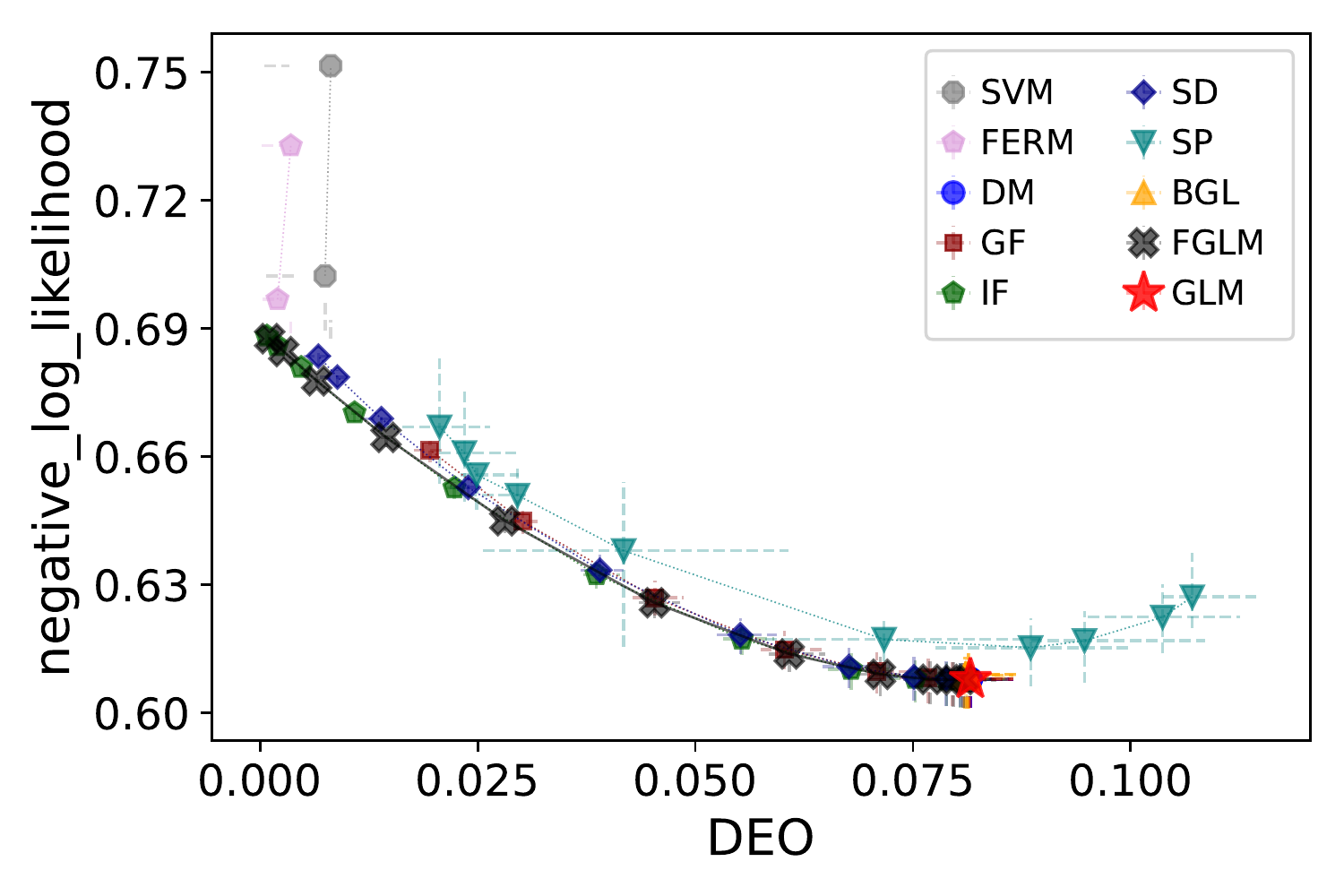}}
    \\
    \subfloat[Drug--Binary--Race(2)]{\includegraphics[width=0.32\linewidth,keepaspectratio]{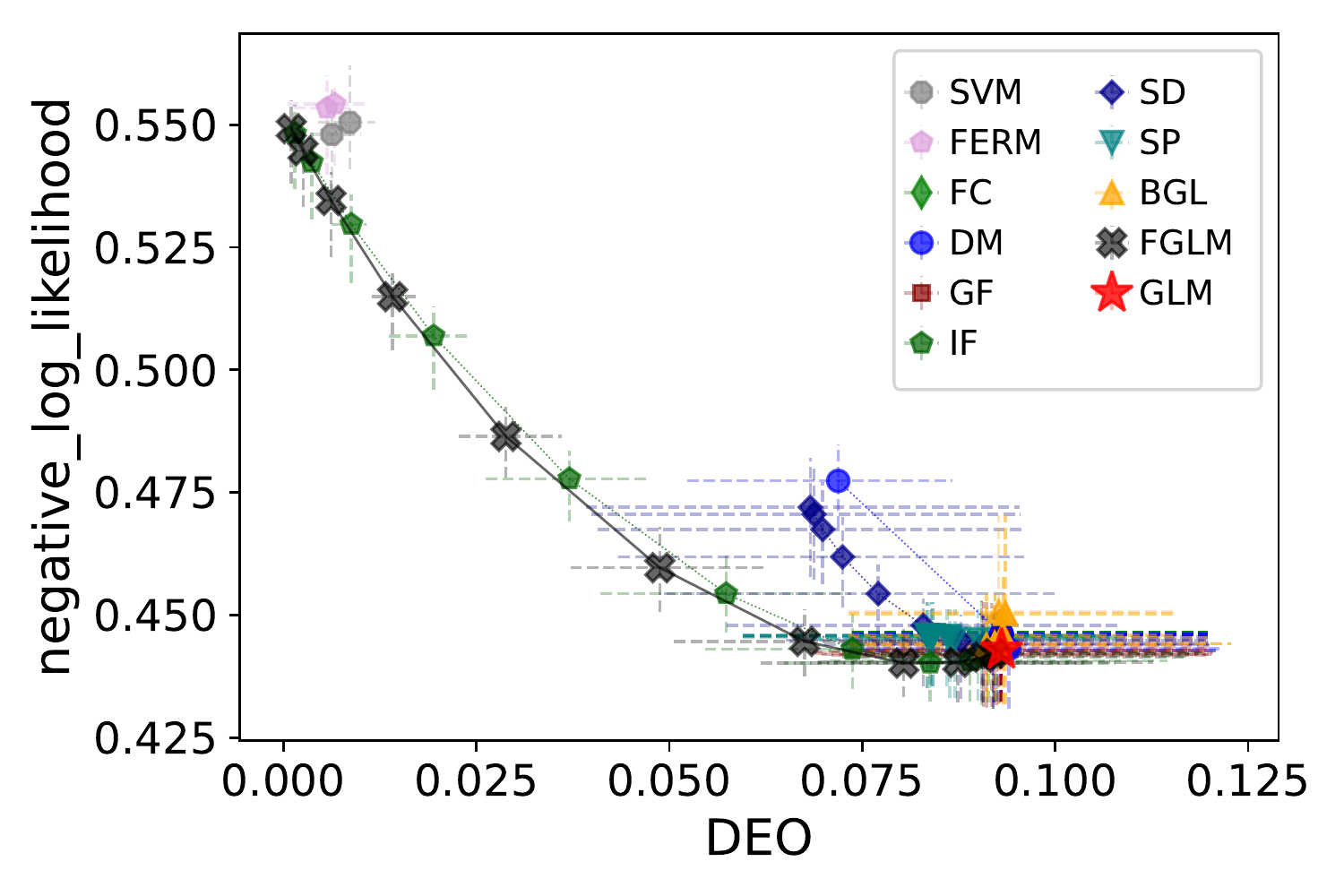}}
    \subfloat[German--Binary--Race(2)]{\includegraphics[width=0.32\linewidth,keepaspectratio]{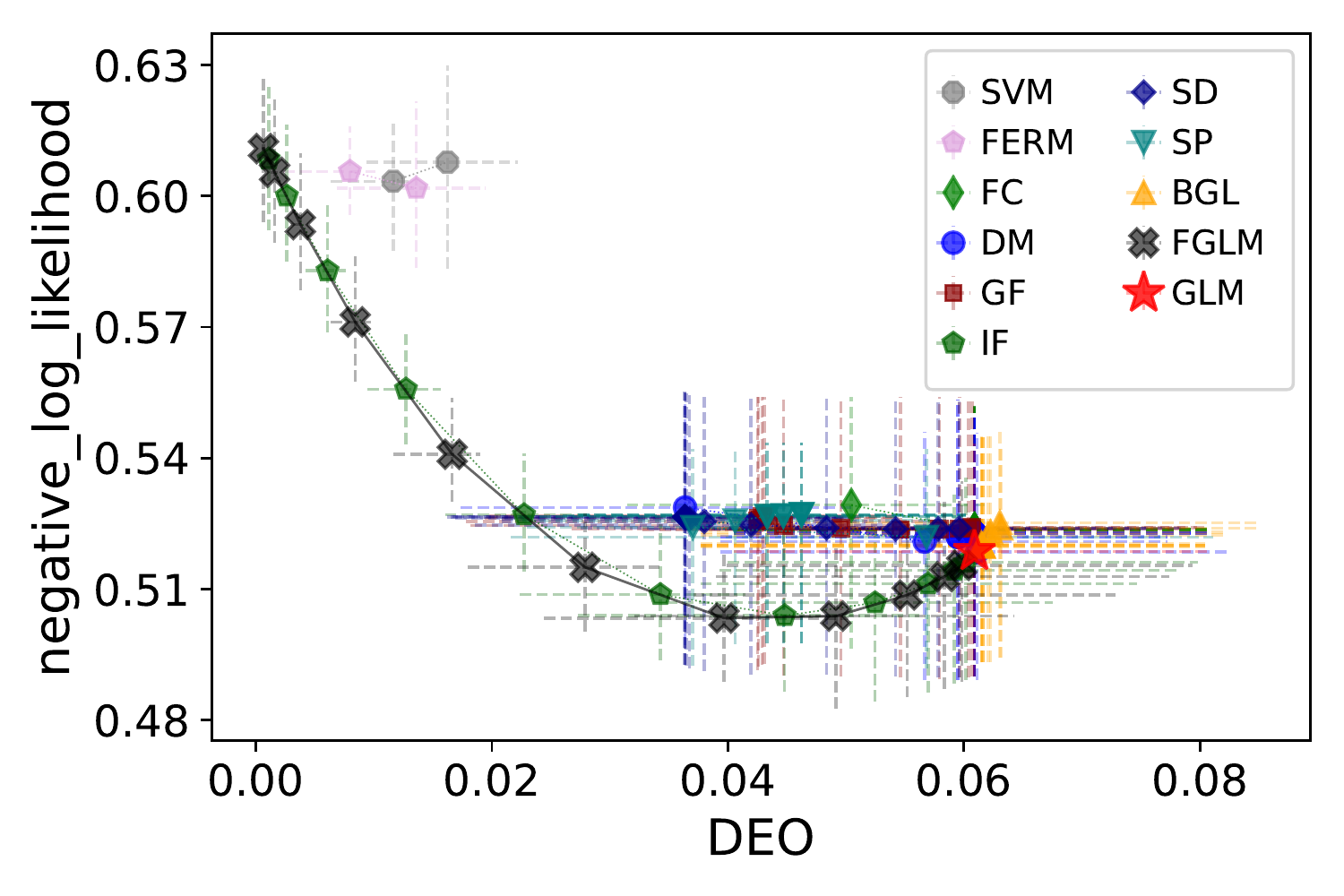}}
    \subfloat[Crime-Cont-Race(3)]{\includegraphics[width=0.32\linewidth,keepaspectratio]{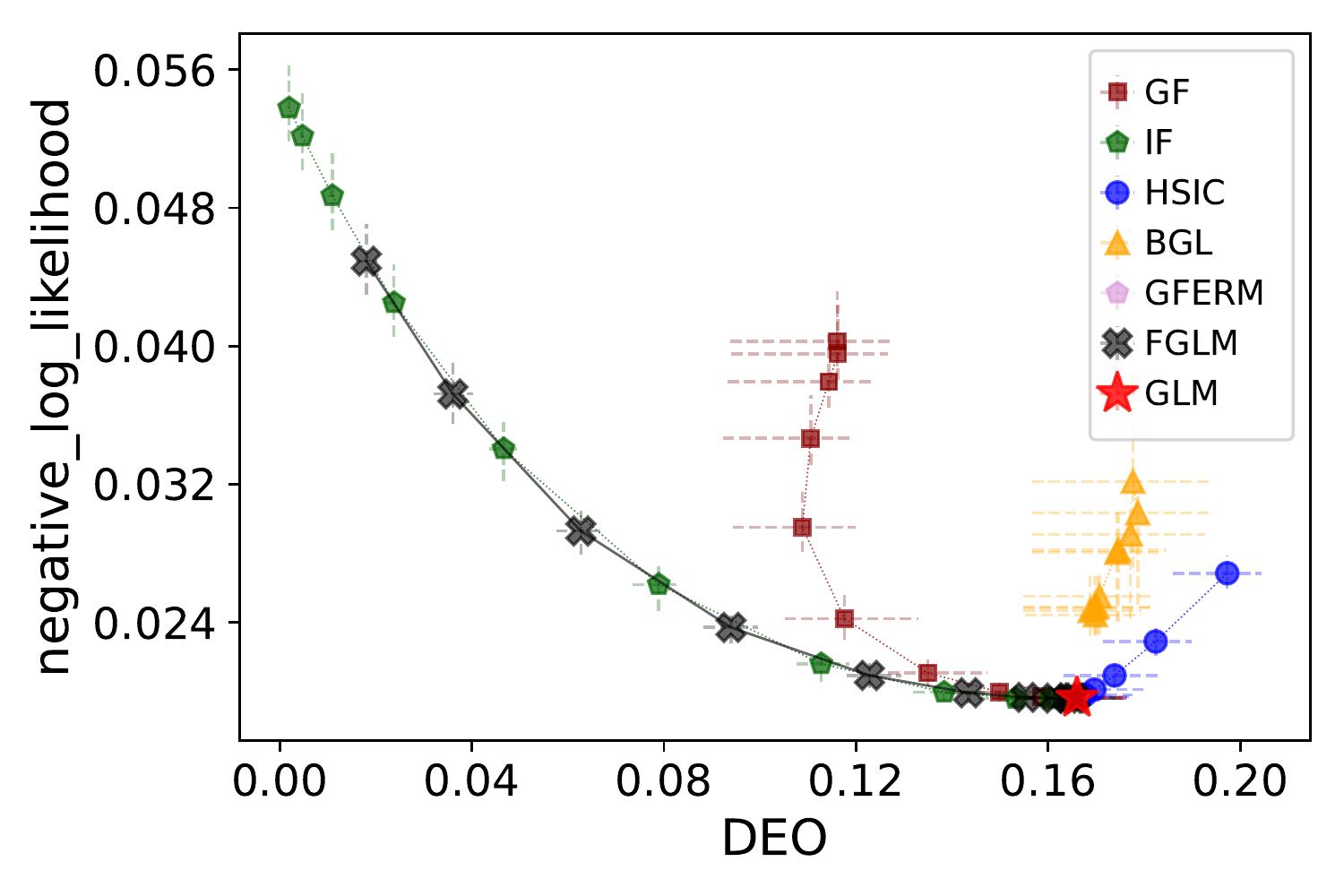}}
    \\
    \subfloat[LSAC--Cont--Race(5)]{\includegraphics[width=0.32\linewidth,keepaspectratio]{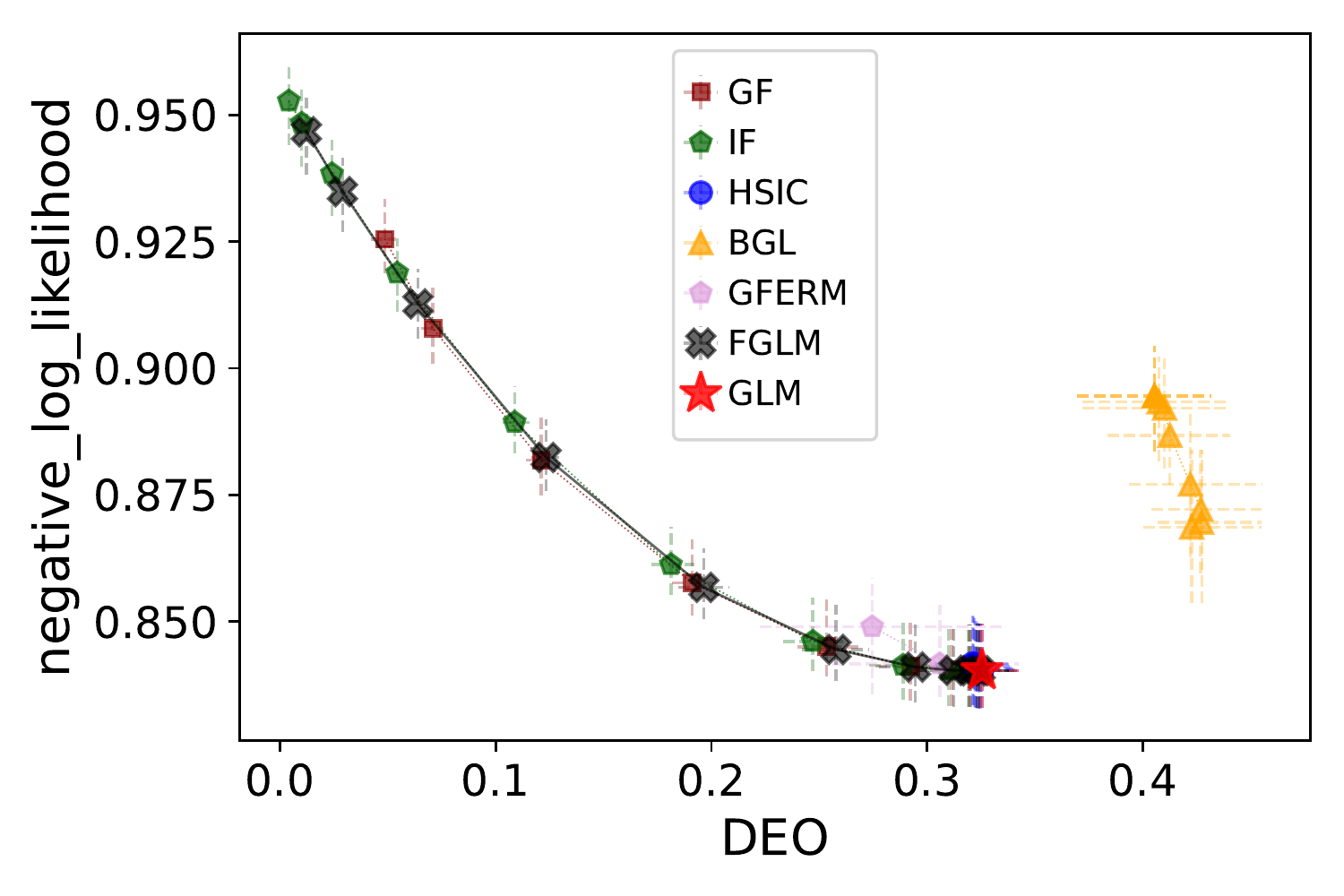}}
    \subfloat[Parkinsons--Cont--Gender(2)]{\includegraphics[width=0.32\linewidth,keepaspectratio]{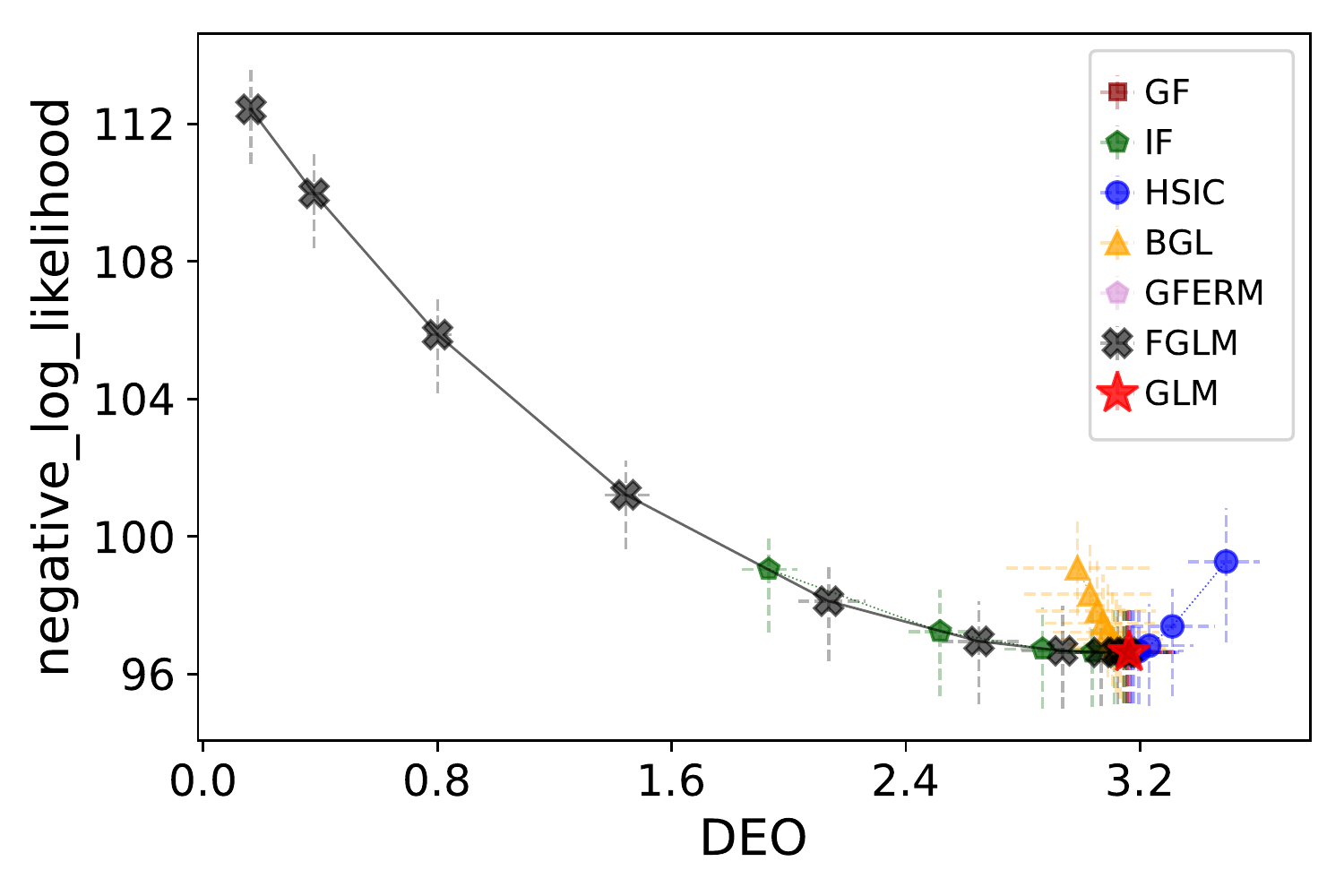}}
    \subfloat[Student--Cont--Gender(2)]{\includegraphics[width=0.32\linewidth,keepaspectratio]{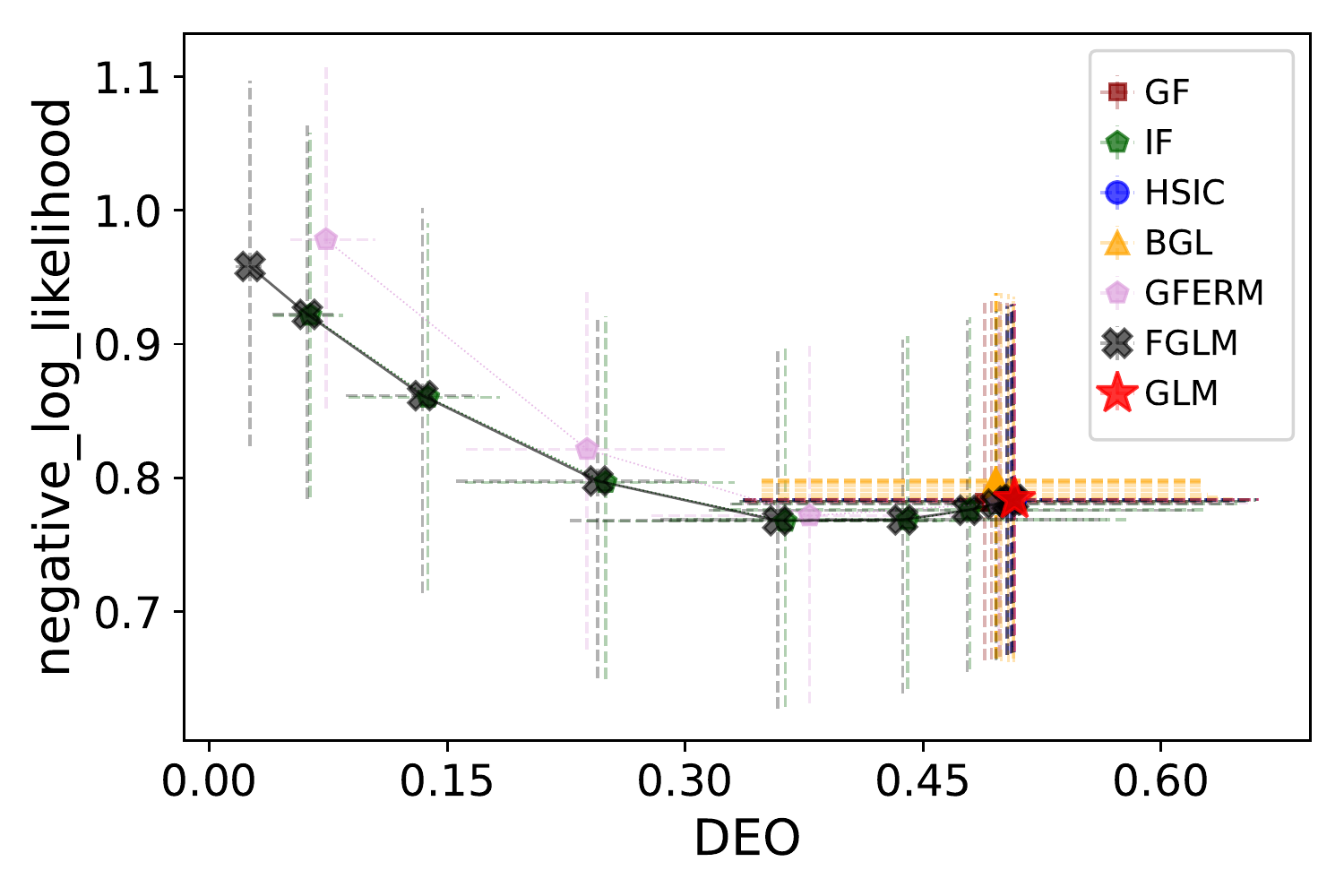}}
     \\
     \subfloat[Drug--Multi--Race(2)]{\includegraphics[width=0.32\linewidth,keepaspectratio]{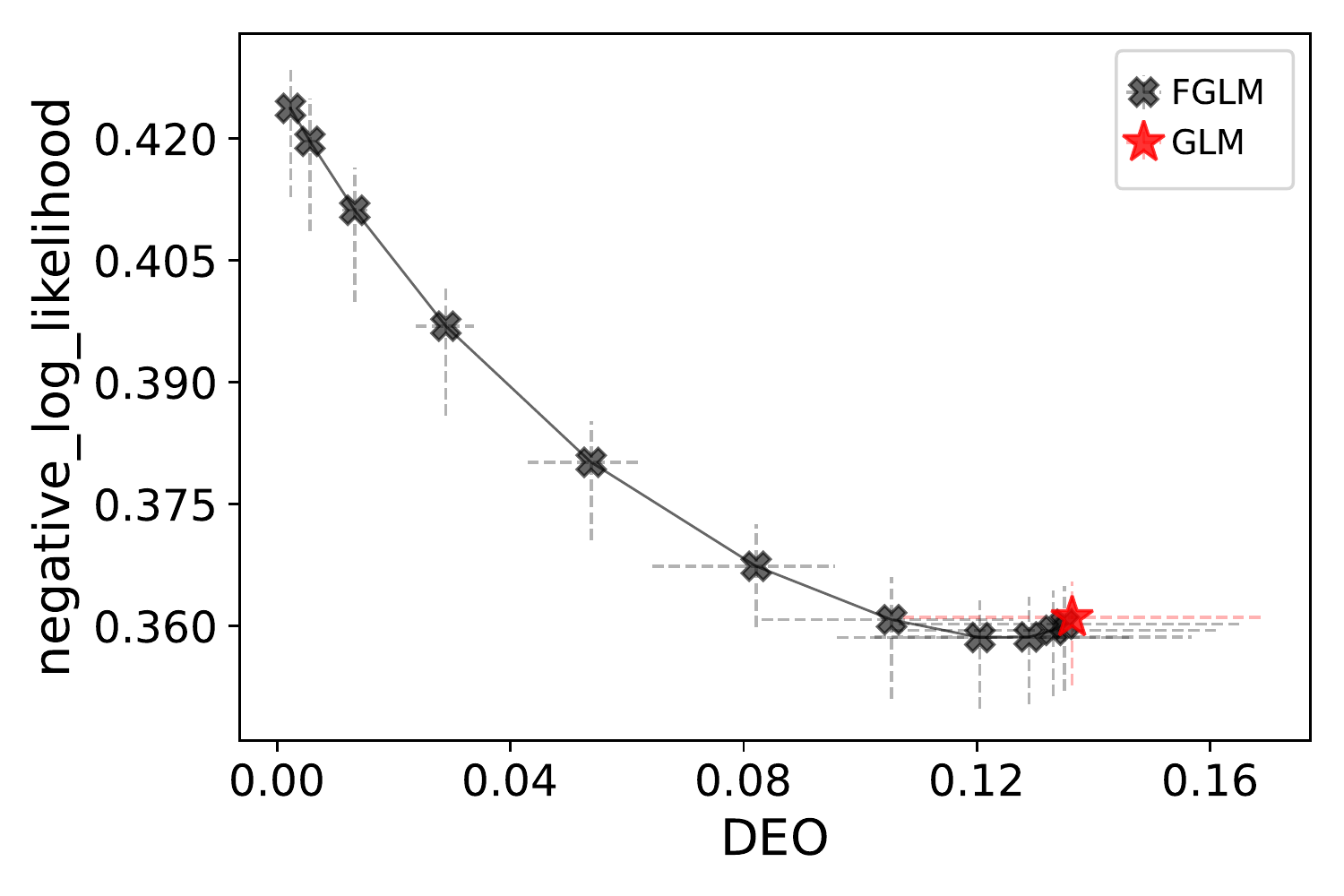}}
     \subfloat[Obesity--Multi--Gender(2)]{\includegraphics[width=0.32\linewidth,keepaspectratio]{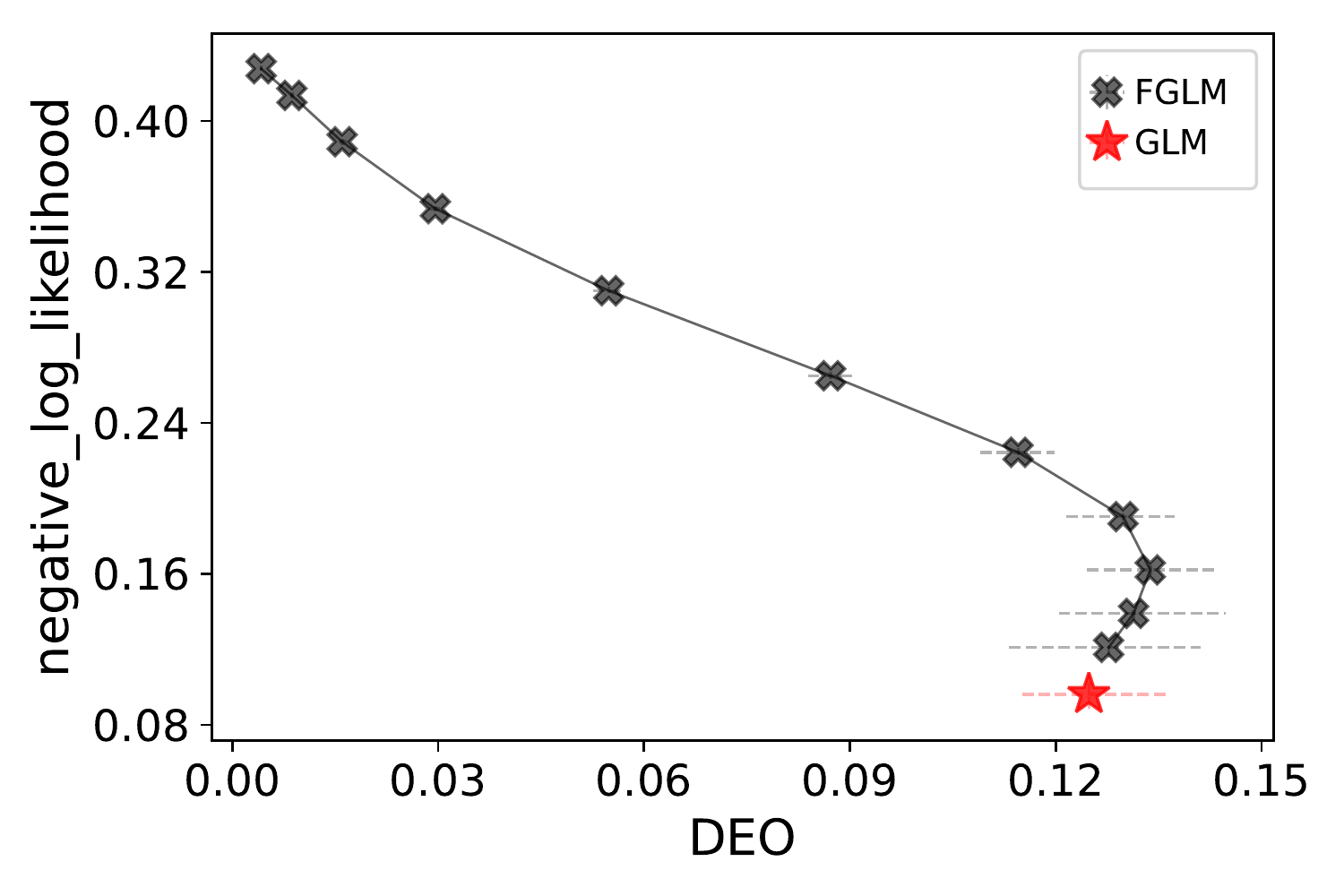}}
     \subfloat[HRS--Count--Race(4)]{\includegraphics[width=0.32\linewidth,keepaspectratio]{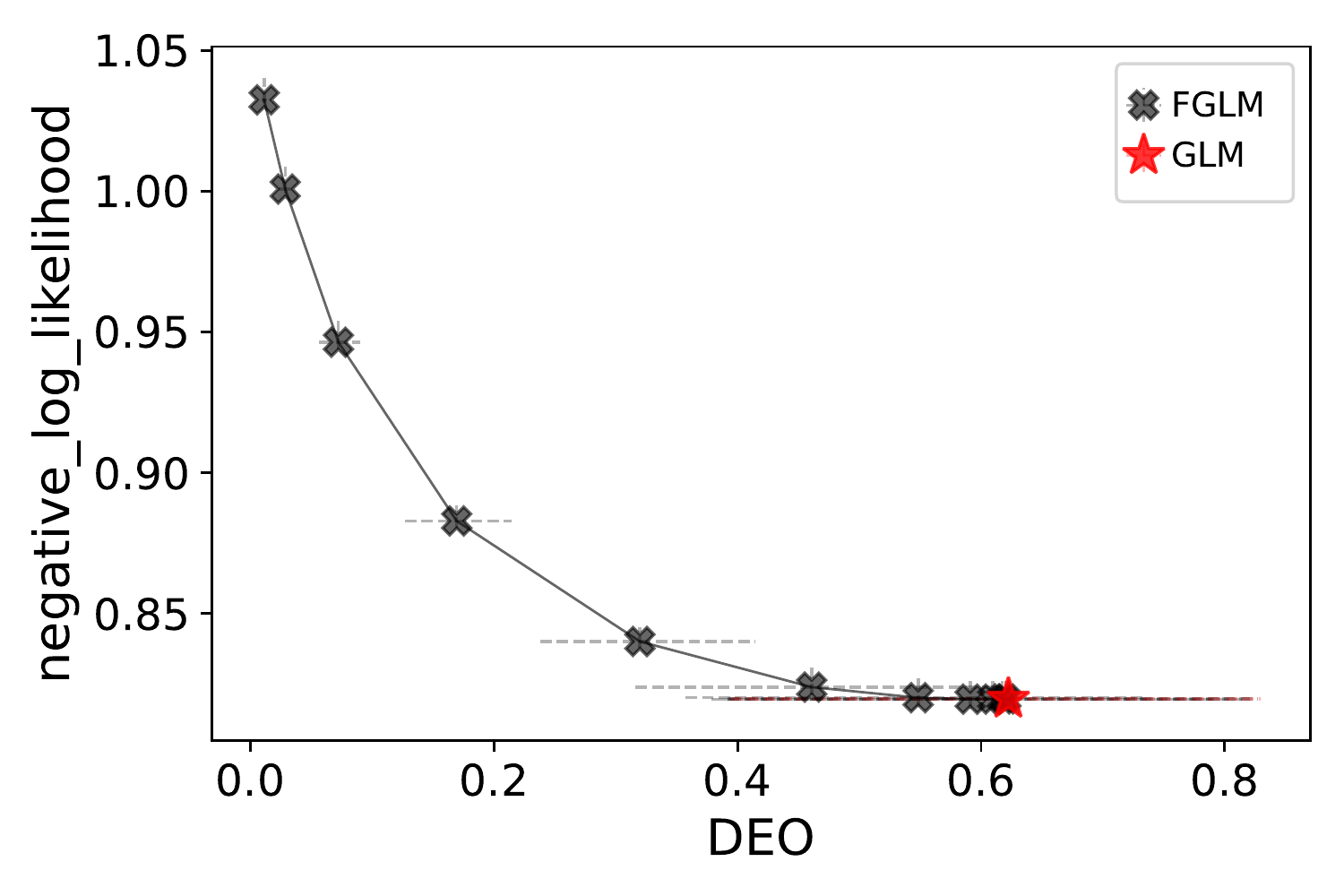}}
    \caption{Experimental results \emph{in negative log-likelihoods} and $\mathcal{D}_{\text{EO}}$ from 11 real world datasets with binary (a-e) and continuous outcomes (f-i). Each subtitle is in the form of {Dataset}--{Outcome Type}--{Sensitive Attribute($K$)}. For both binary and continuous outcomes we use a Generalized Linear Model (GLM, red star \inlinegraphics{markers/GLM.png}), Fair Generalized Linear Model (F-GLM, black X \inlinegraphics{markers/FGLM.png}), Individual Fairness penalty (IF, green pentagon \inlinegraphics{markers/IF.png}), Group Fairness penalty (GF, dark red square \inlinegraphics{markers/GF.png}), and Bounded Group Loss (BGL, orange triangle \inlinegraphics{markers/BGL.png}). Methods for binary outcomes also include the Support Vector Machine (SVM, grey hexagon \inlinegraphics{markers/SVM.png}), Fair Constraints (FC, green diamond \inlinegraphics{markers/FC.png}), Disparate Mistreatment (DM, blue circle \inlinegraphics{markers/DM-HSIC.png}), Squared Difference penalizer (SD, dark blue diamond \inlinegraphics{markers/SD.png}), Fair Empirical Risk Minimization (FERM, plum pentagon \inlinegraphics{markers/GFERM.png}), Statistical Parity (SP, teal triangle \inlinegraphics{markers/SP.png}). Methods for continuous outcomes include the HSIC penalty (HSIC, blue circle \inlinegraphics{markers/DM-HSIC.png}), General Fair Empirical Risk Minimization (GFERM, plum pentagon \inlinegraphics{markers/GFERM.png}). See Table 1 for additional information for each method. Each dot represents mean performance across test sets for a specific hyperparameter value $\lambda$.}
    \label{fig:nll-deo}
\end{figure*}

\begin{figure}[p]
    \centering
    \subfloat[Adult--Binary--Gender(2)]{\includegraphics[width=0.33\linewidth,keepaspectratio]{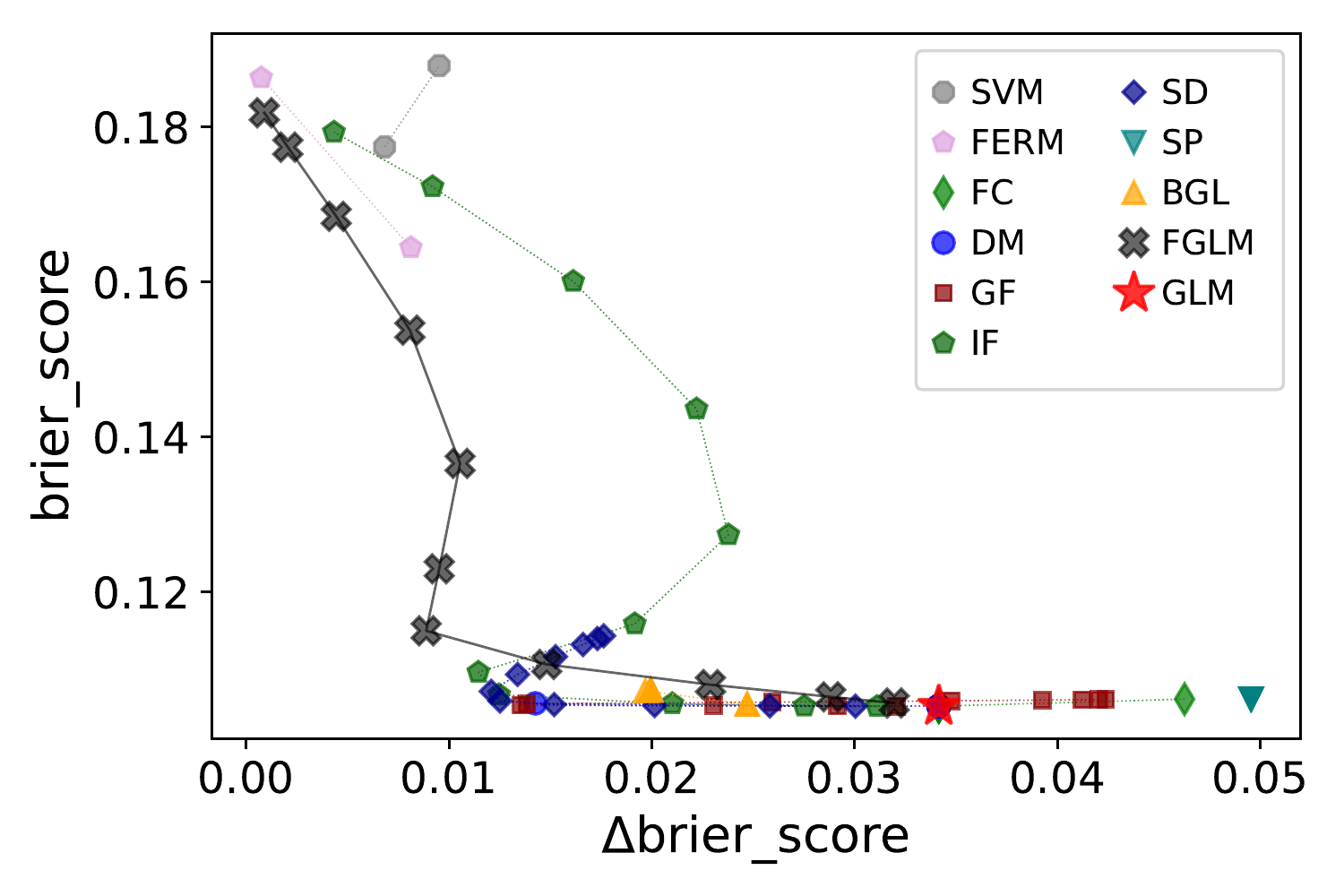}}
    \subfloat[Arrhythmia--Binary--Gender(2)]{\includegraphics[width=0.33\linewidth,keepaspectratio]{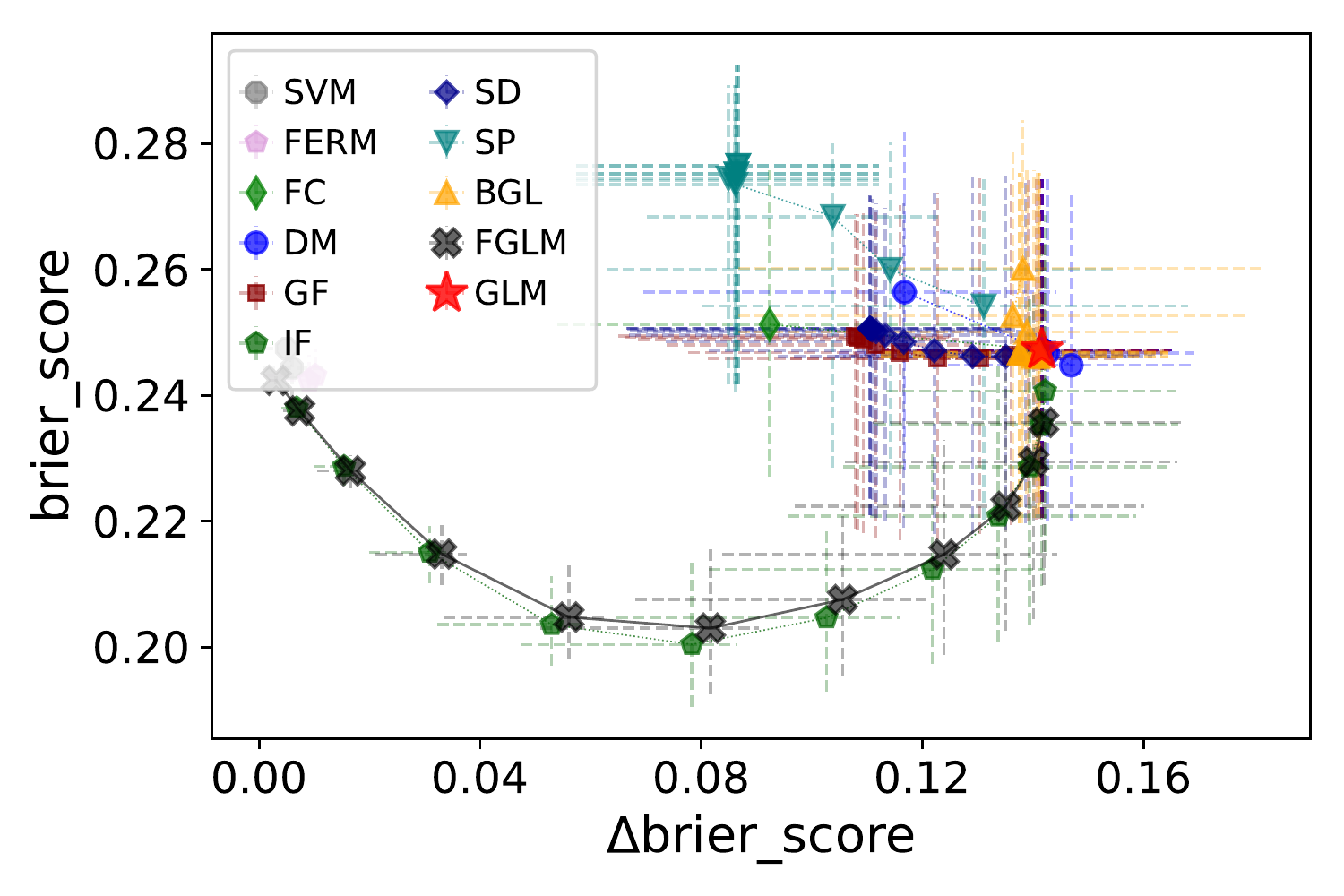}}
    \subfloat[COMPAS--Binary--Race(4)]{\includegraphics[width=0.33\linewidth,keepaspectratio]{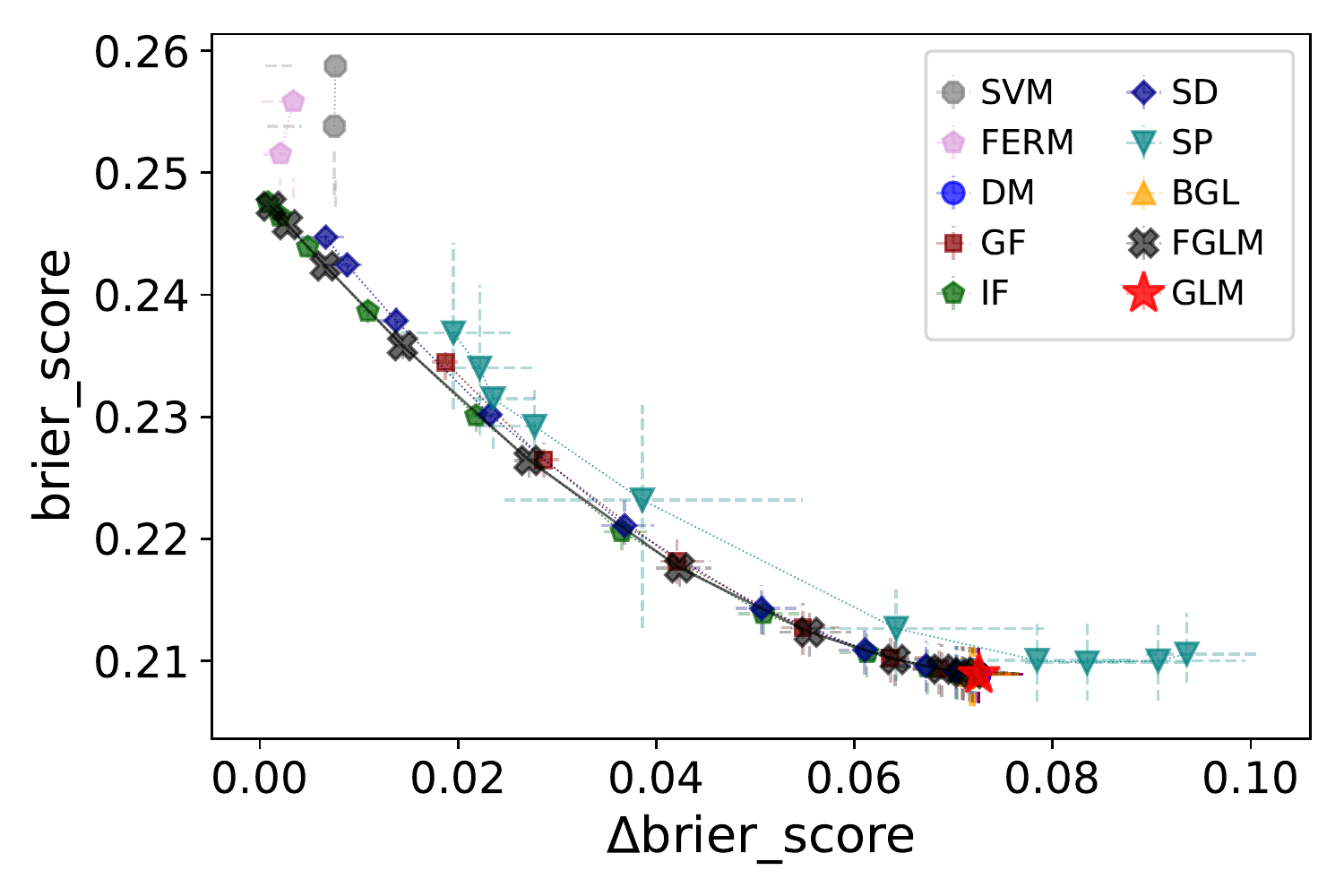}}
    \\
    \subfloat[Drug--Binary--Race(2)]{\includegraphics[width=0.33\linewidth,keepaspectratio]{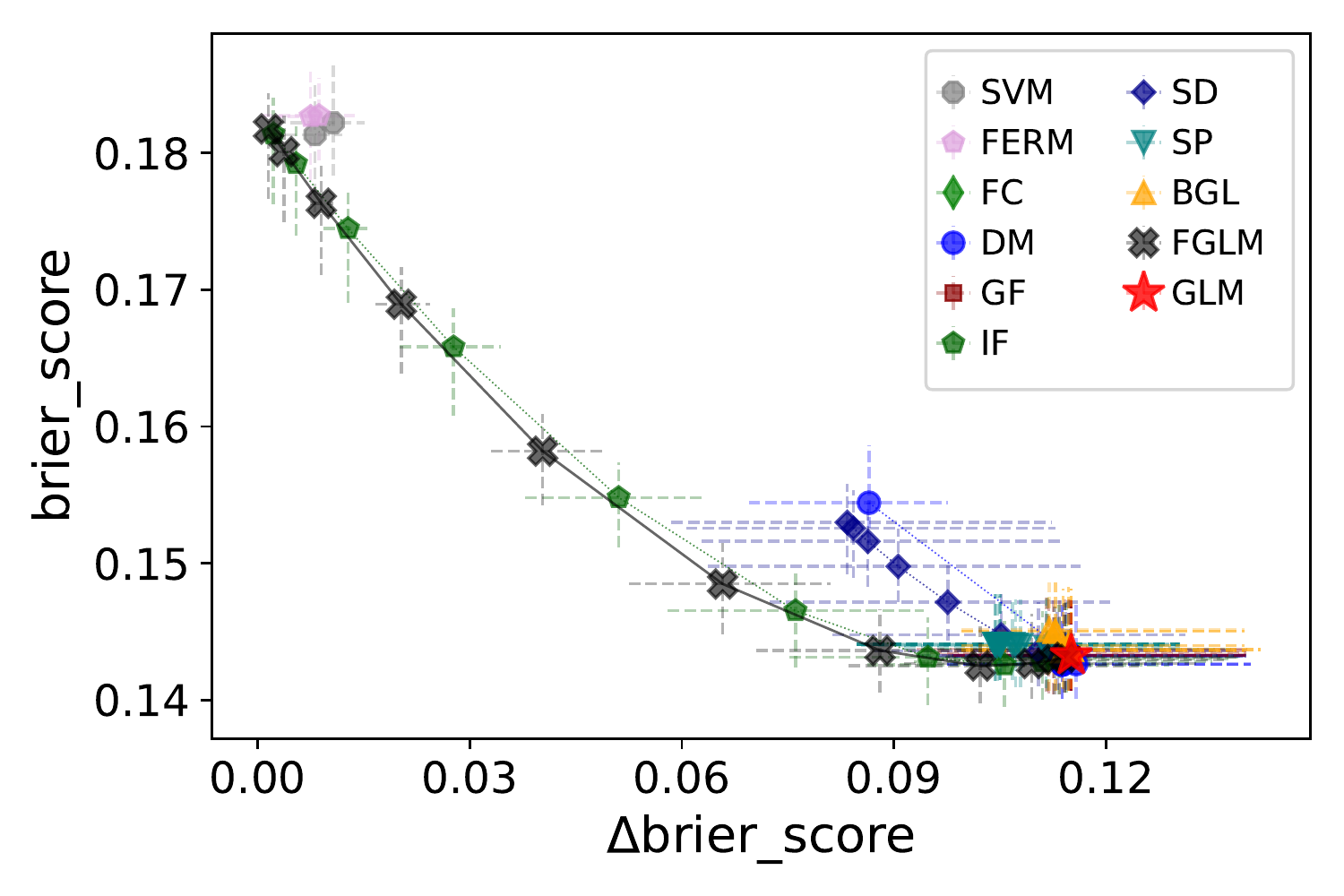}}
    \subfloat[German--Binary--Race(2)]{\includegraphics[width=0.33\linewidth,keepaspectratio]{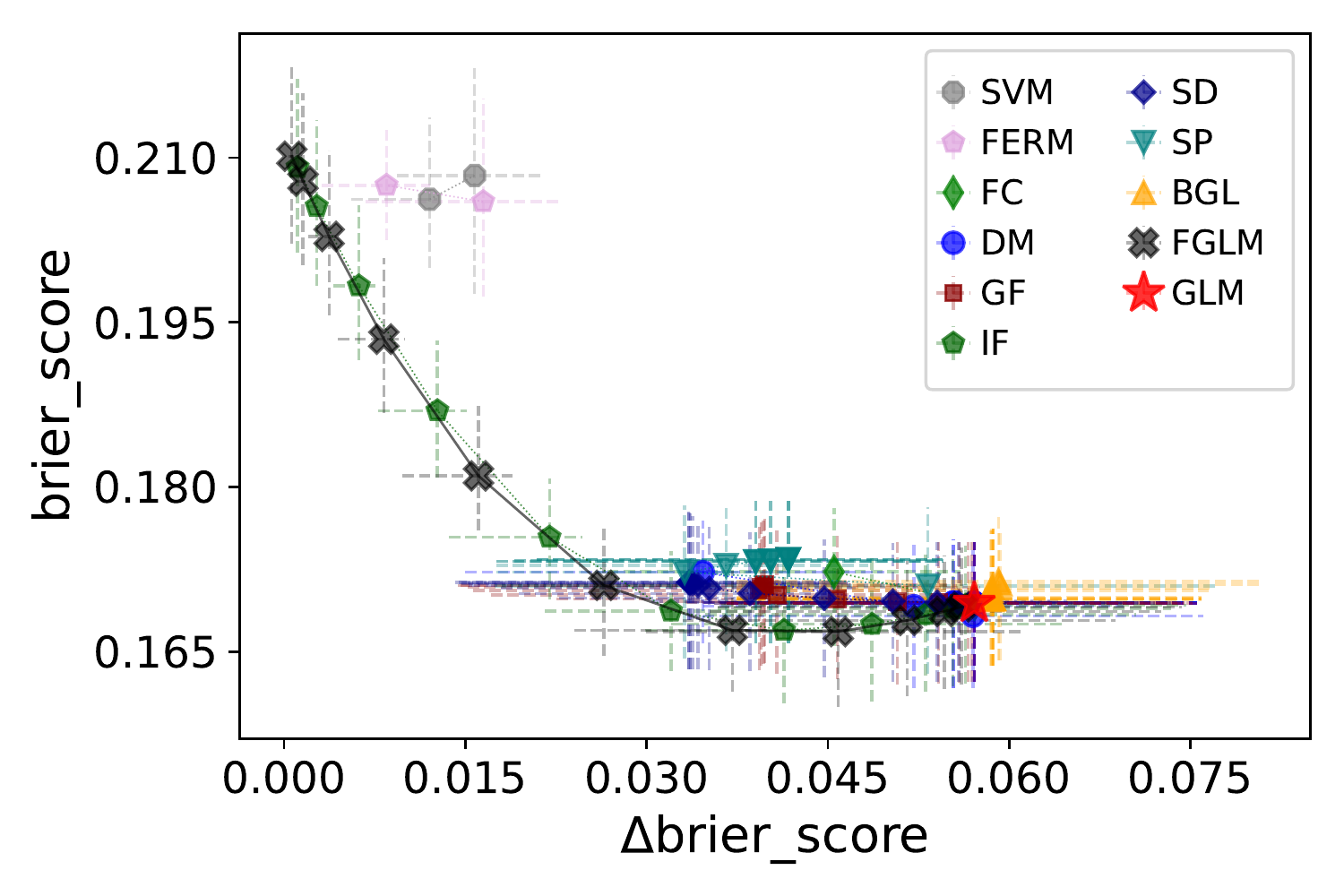}}
    \subfloat[Crime-Cont-Race(3)]{\includegraphics[width=0.33\linewidth,keepaspectratio]{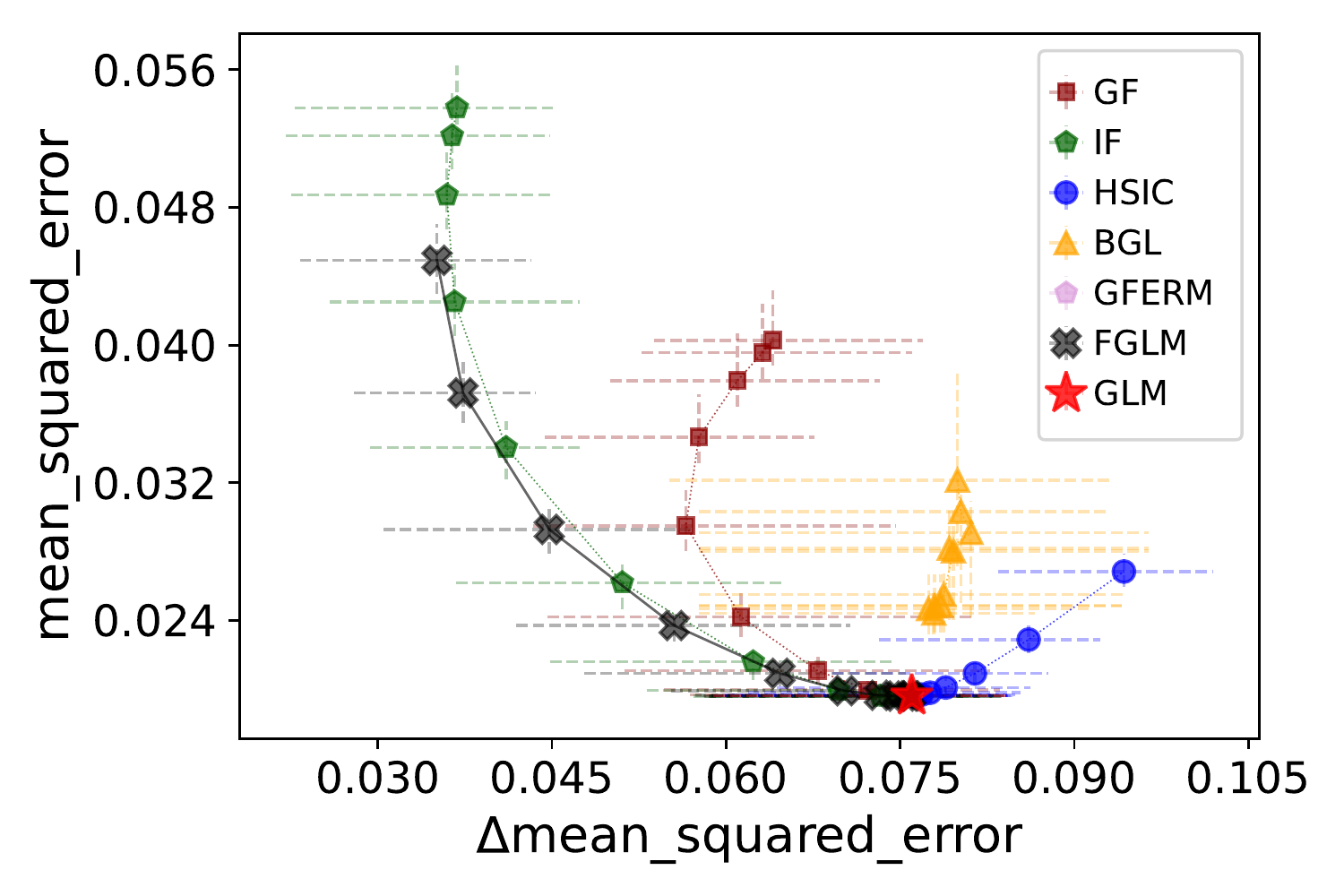}}
    \\
    \subfloat[LSAC--Cont--Race(5)]{\includegraphics[width=0.33\linewidth,keepaspectratio]{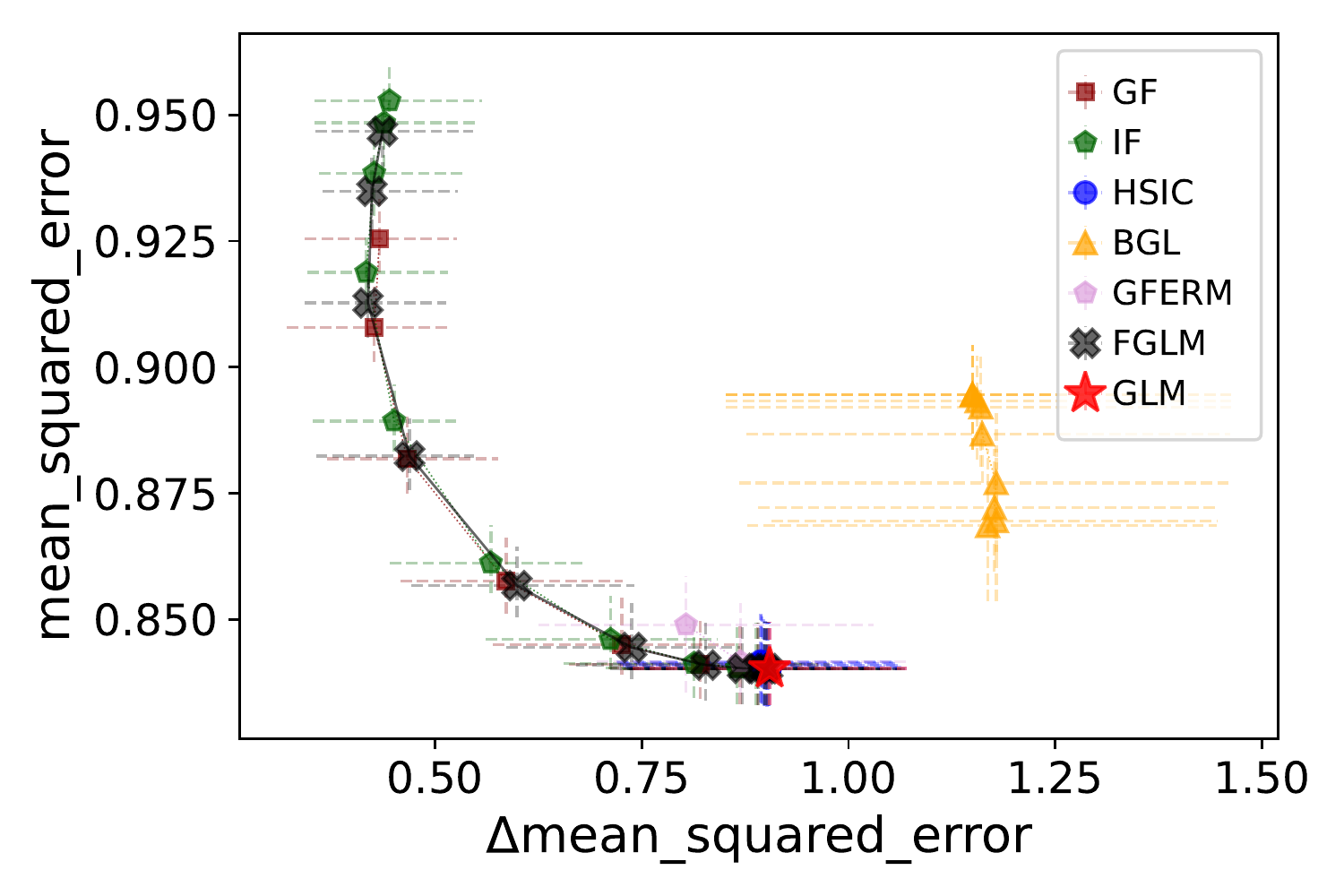}}
    \subfloat[Parkinsons--Cont--Gender(2)]{\includegraphics[width=0.33\linewidth,keepaspectratio]{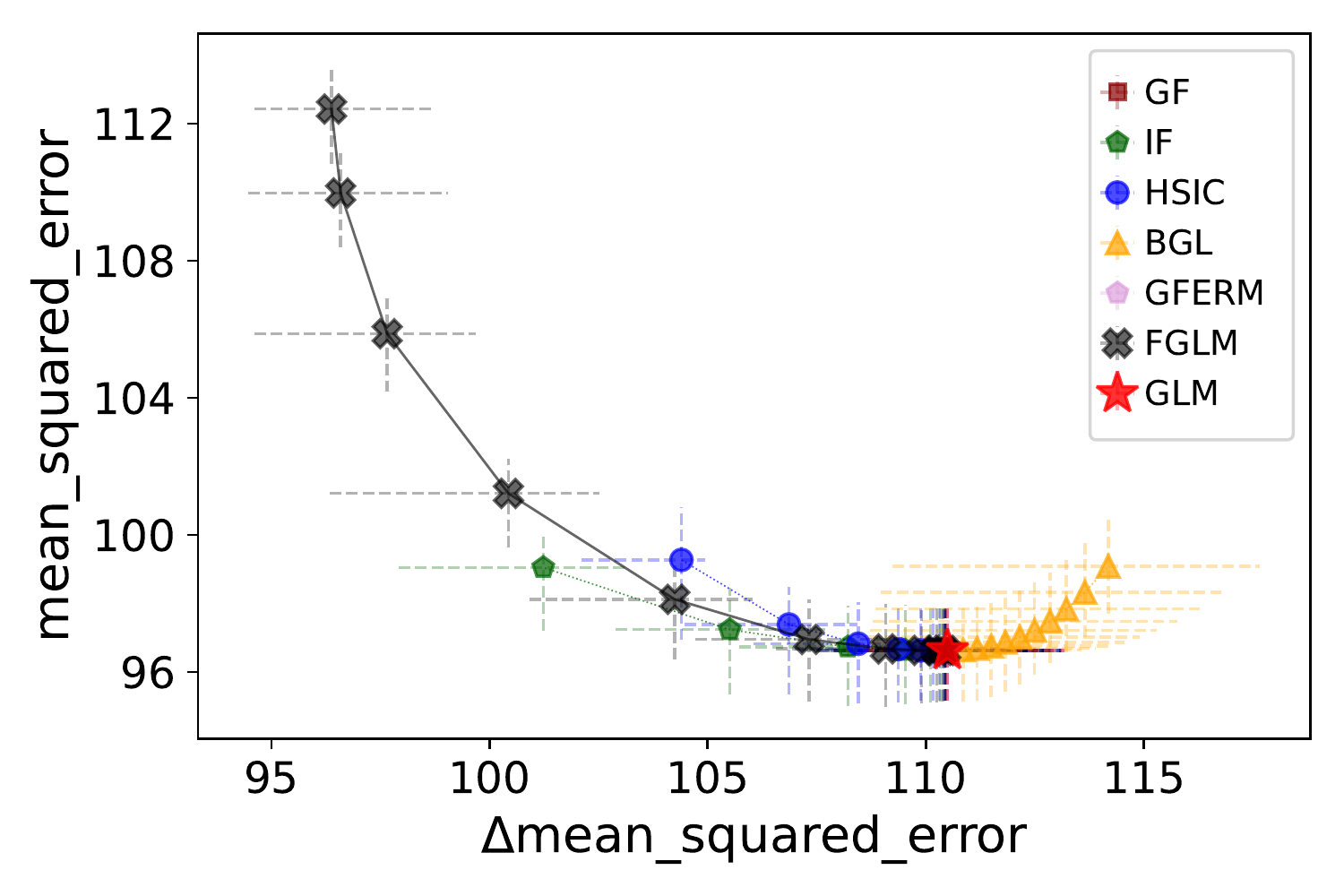}}
    \subfloat[Student--Cont--Gender(2)]{\includegraphics[width=0.33\linewidth,keepaspectratio]{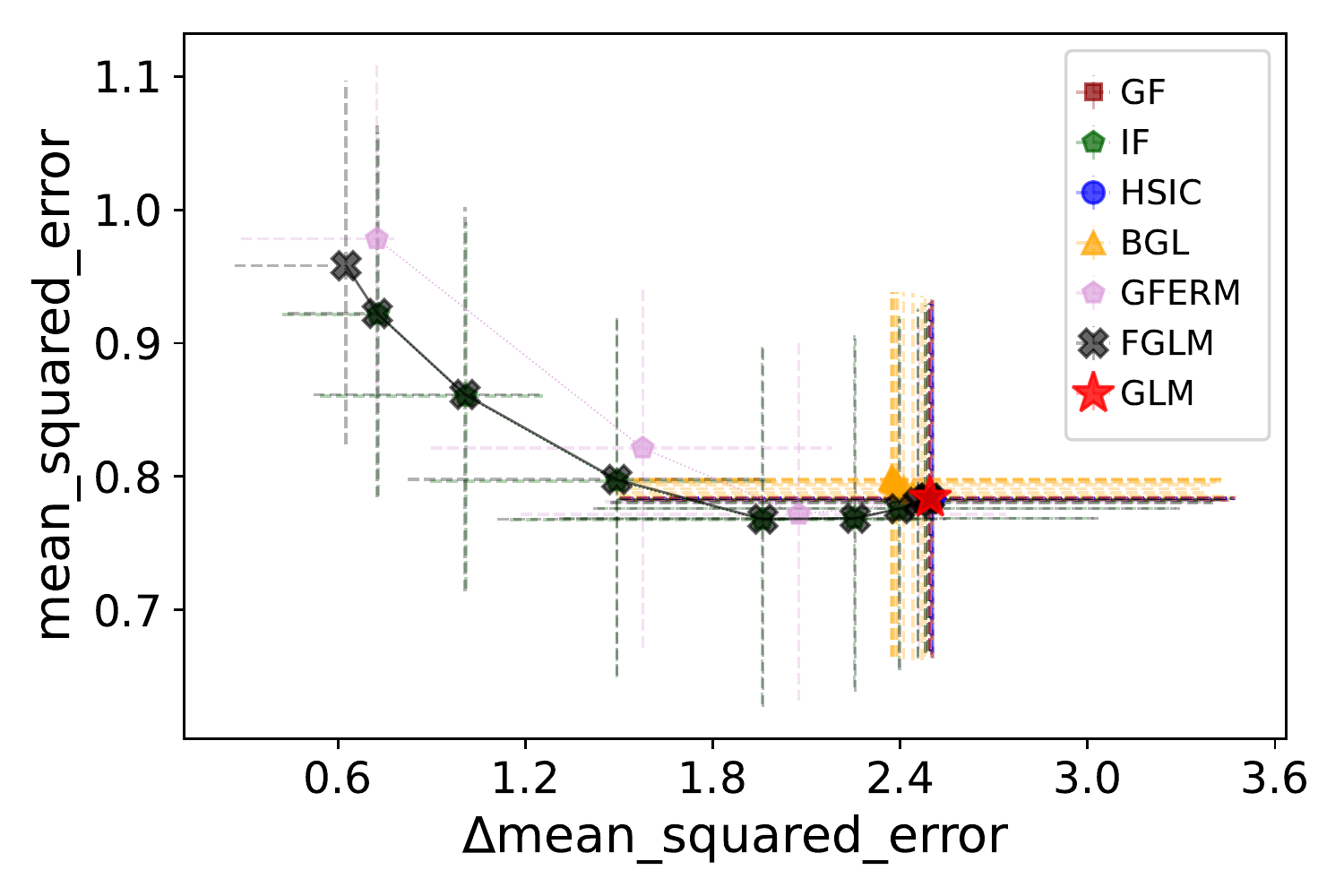}}
     \\
     \subfloat[Drug--Multi--Race(2)]{\includegraphics[width=0.33\linewidth,keepaspectratio]{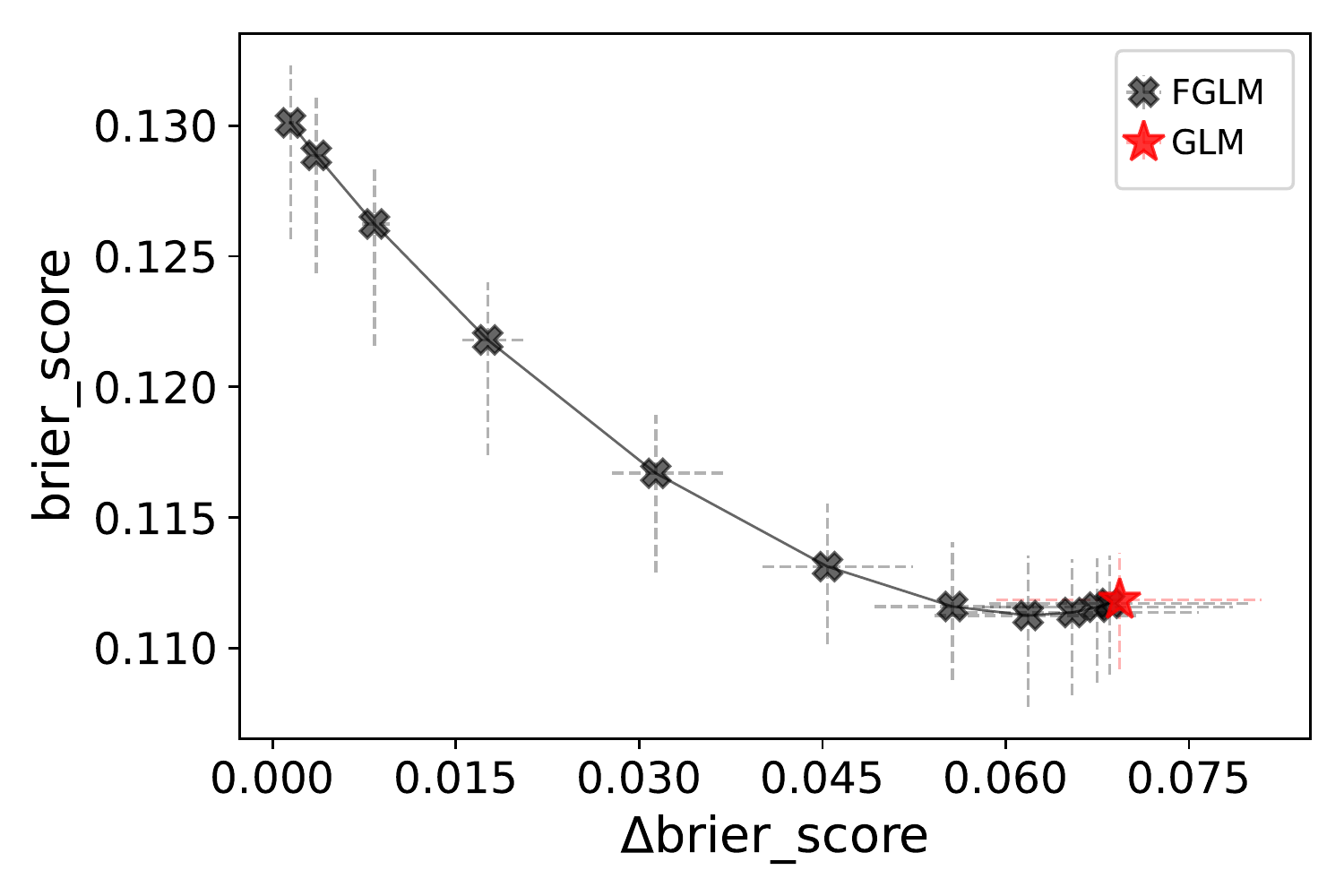}}
     \subfloat[Obesity--Multi--Gender(2)]{\includegraphics[width=0.33\linewidth,keepaspectratio]{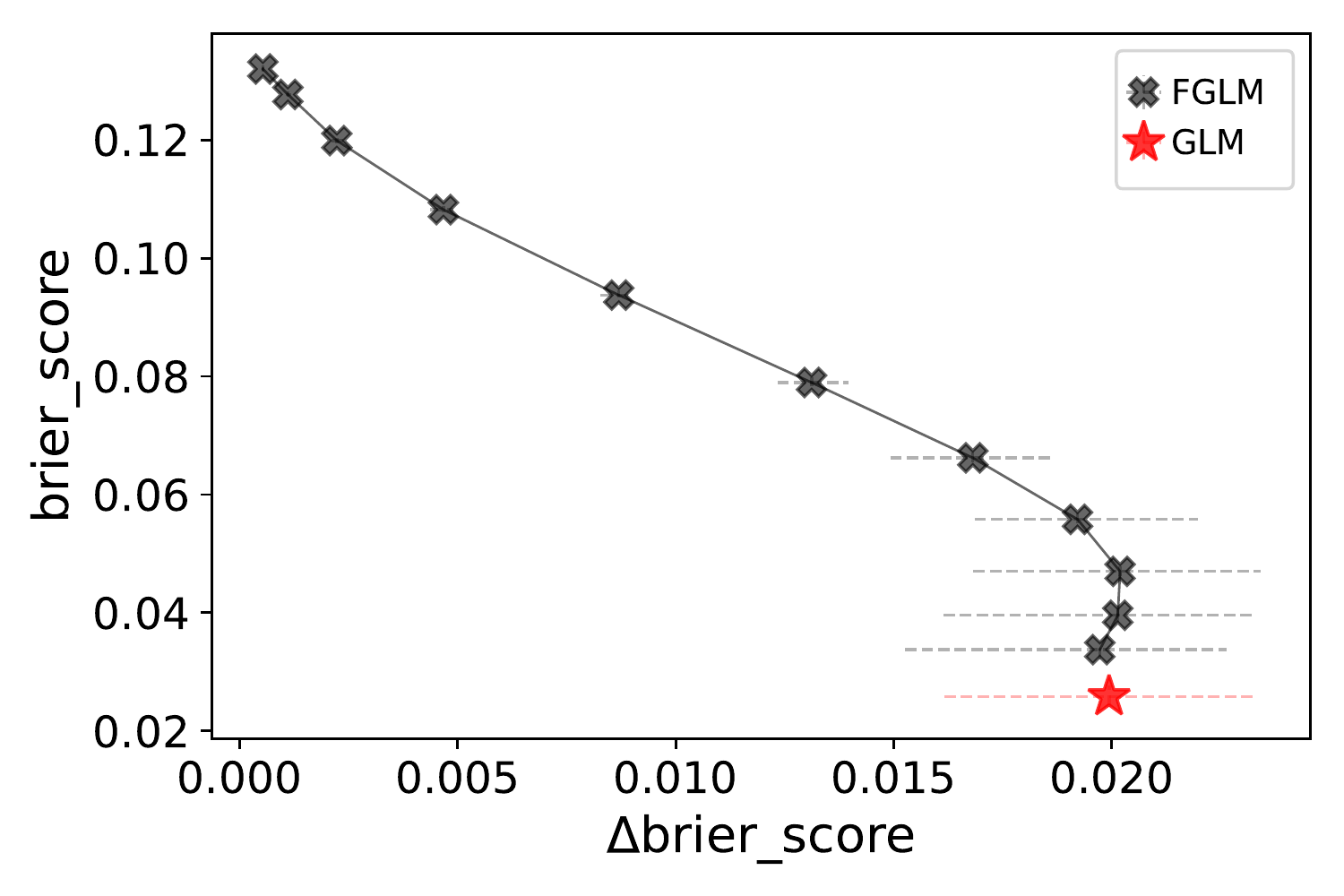}}
     \subfloat[HRS--Count--Race(4)]{\includegraphics[width=0.33\linewidth,keepaspectratio]{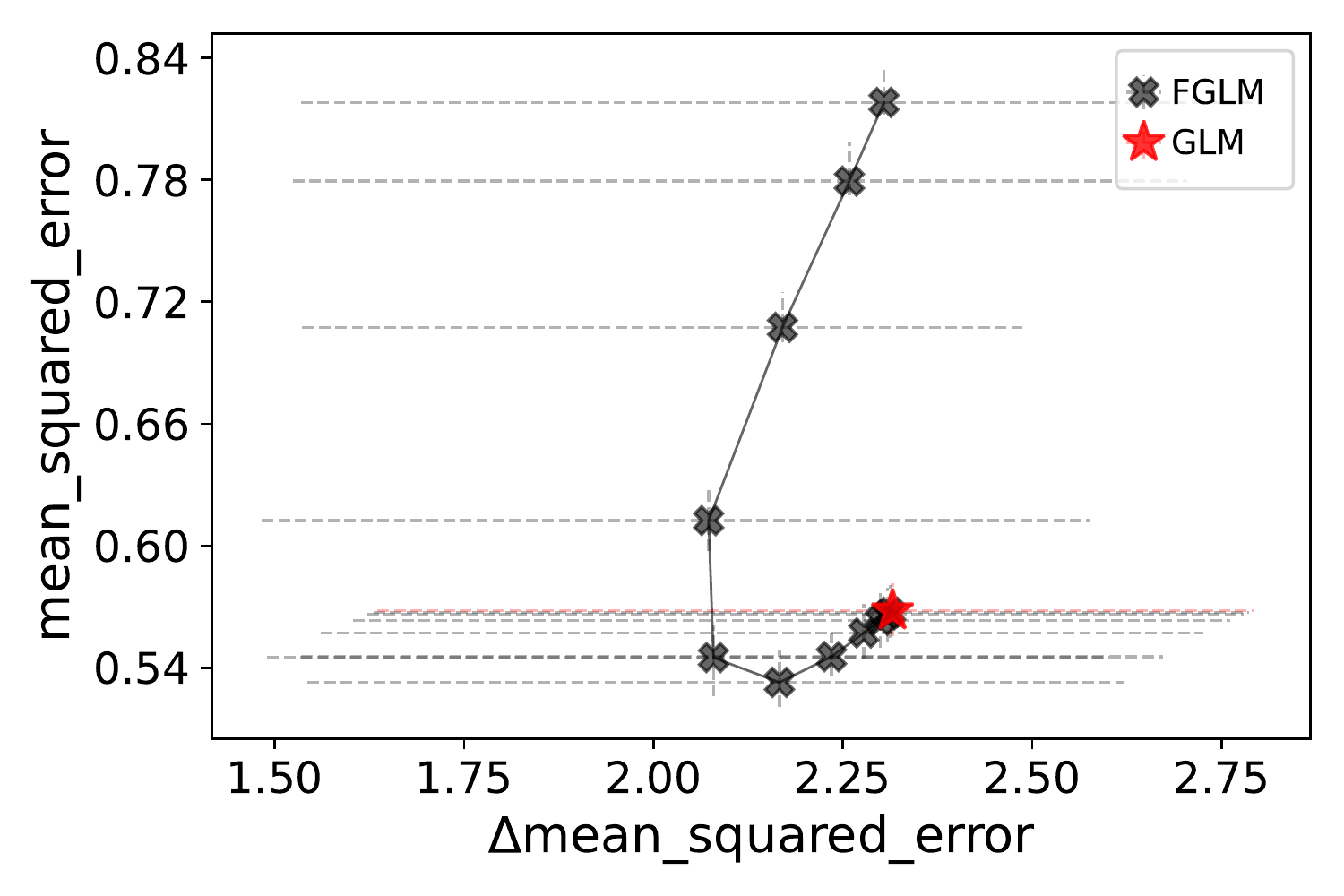}}
    \caption{Experimental results in mean squared errors (also referred to as brier scores for classification problems) from 11 real world datasets with binary (a-e) and continuous outcomes (f-i). $x$-axis is disparity of mean squared errors. Each subtitle is in the form of {Dataset}--{Outcome Type}--{Sensitive Attribute($K$)}. For both binary and continuous outcomes we use a Generalized Linear Model (GLM, red star \inlinegraphics{markers/GLM.png}), Fair Generalized Linear Model (F-GLM, black X \inlinegraphics{markers/FGLM.png}), Individual Fairness penalty (IF, green pentagon \inlinegraphics{markers/IF.png}), Group Fairness penalty (GF, dark red square \inlinegraphics{markers/GF.png}), and Bounded Group Loss (BGL, orange triangle \inlinegraphics{markers/BGL.png}). Methods for binary outcomes also include the Support Vector Machine (SVM, grey hexagon \inlinegraphics{markers/SVM.png}), Fair Constraints (FC, green diamond \inlinegraphics{markers/FC.png}), Disparate Mistreatment (DM, blue circle \inlinegraphics{markers/DM-HSIC.png}), Squared Difference penalizer (SD, dark blue diamond \inlinegraphics{markers/SD.png}), Fair Empirical Risk Minimization (FERM, plum pentagon \inlinegraphics{markers/GFERM.png}), Statistical Parity (SP, teal triangle \inlinegraphics{markers/SP.png}). Methods for continuous outcomes include the HSIC penalty (HSIC, blue circle \inlinegraphics{markers/DM-HSIC.png}), General Fair Empirical Risk Minimization (GFERM, plum pentagon \inlinegraphics{markers/GFERM.png}). See Table 1 for additional information for each method. Each dot represents mean performance across test sets  for a specific hyperparameter value $\lambda$.}
    \label{fig:results3}
\end{figure}

\begin{figure}[p]
    \centering
    \subfloat[Adult--Binary--Gender(2)]{\includegraphics[width=0.32\linewidth,keepaspectratio]{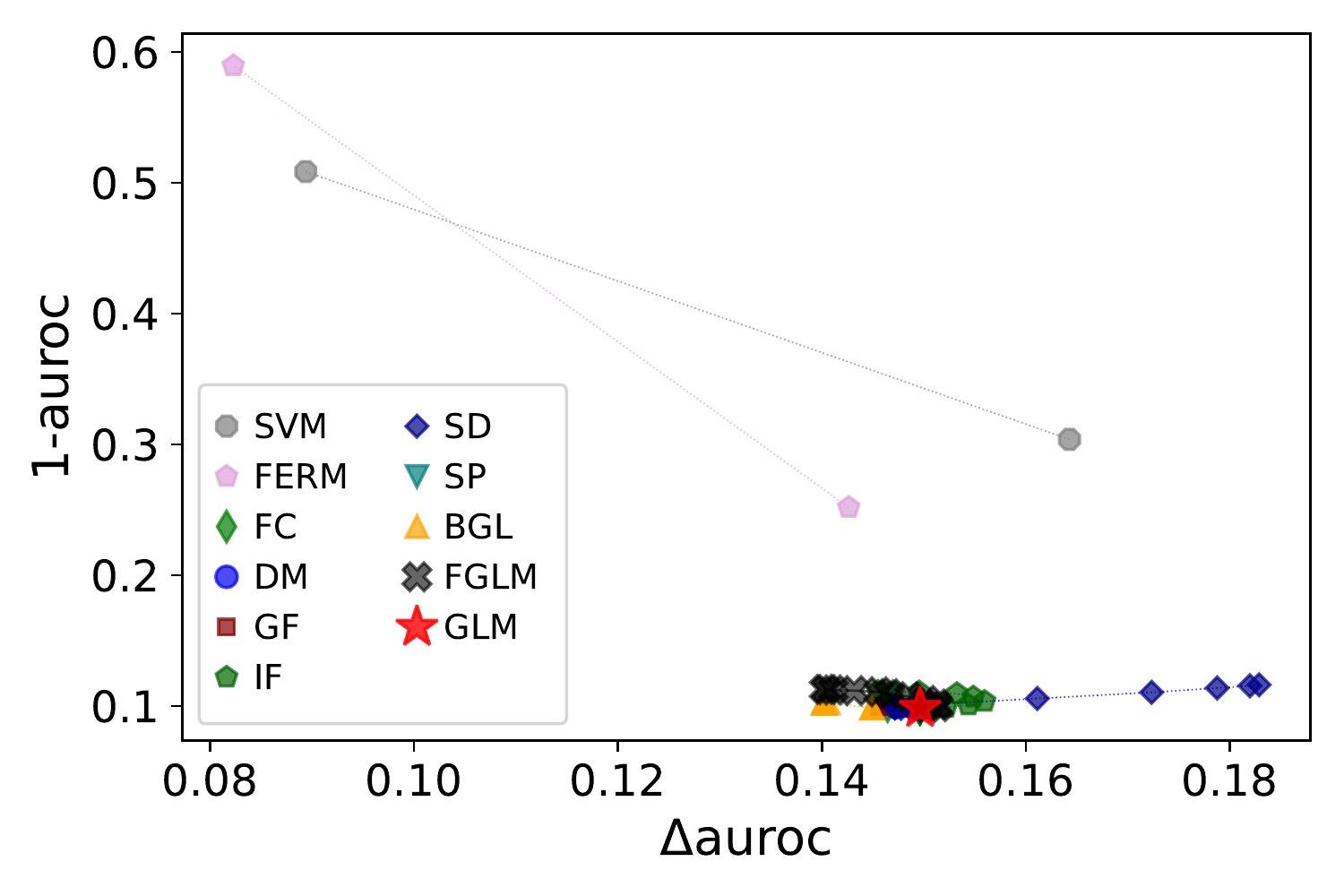}}
    \subfloat[Arrhythmia--Binary--Gender(2)]{\includegraphics[width=0.32\linewidth,keepaspectratio]{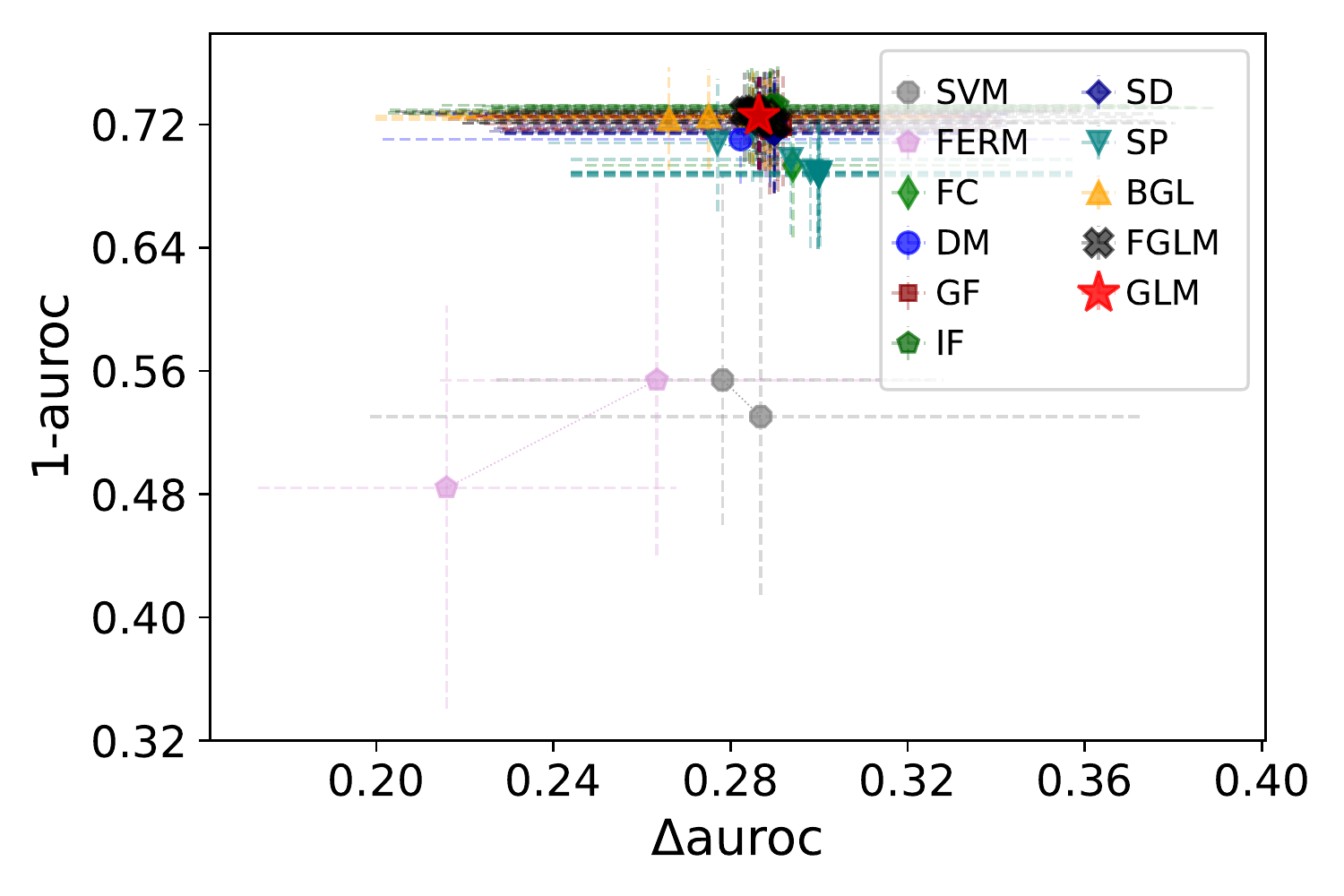}}
    \subfloat[COMPAS--Binary--Race(4)]{\includegraphics[width=0.32\linewidth,keepaspectratio]{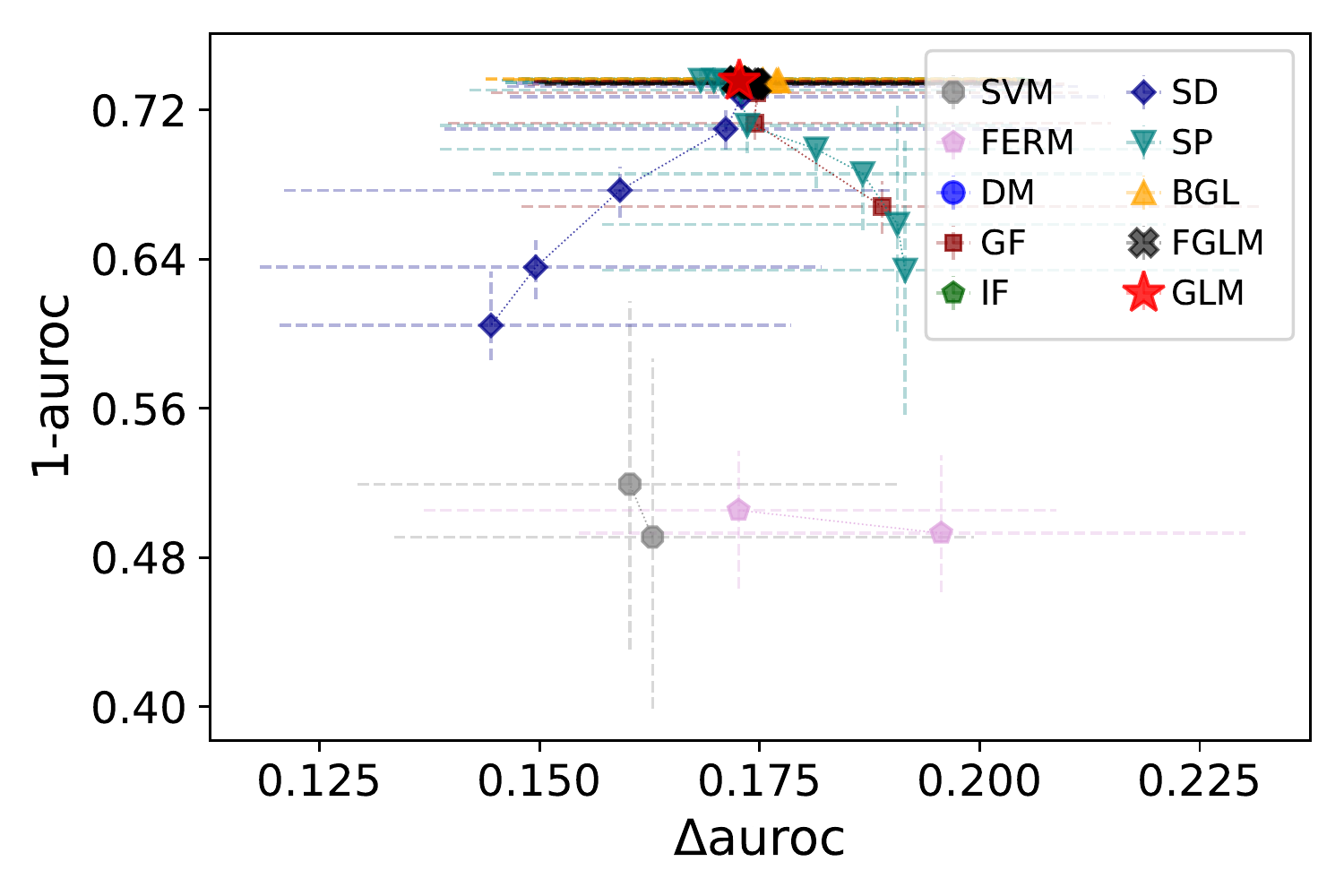}}
    \\
    \subfloat[Drug--Binary--Race(2)]{\includegraphics[width=0.32\linewidth,keepaspectratio]{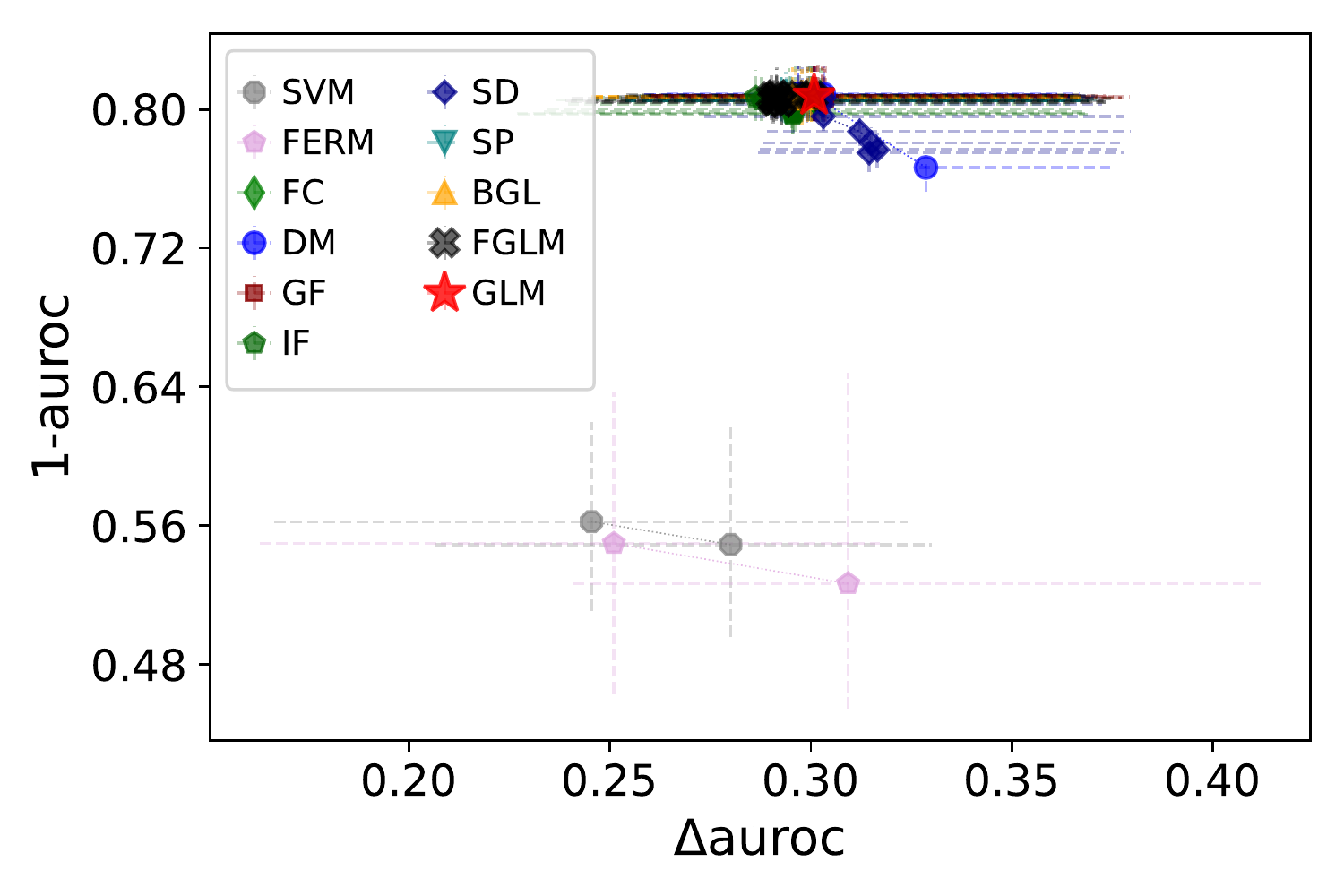}}
    \subfloat[German--Binary--Race(2)]{\includegraphics[width=0.32\linewidth,keepaspectratio]{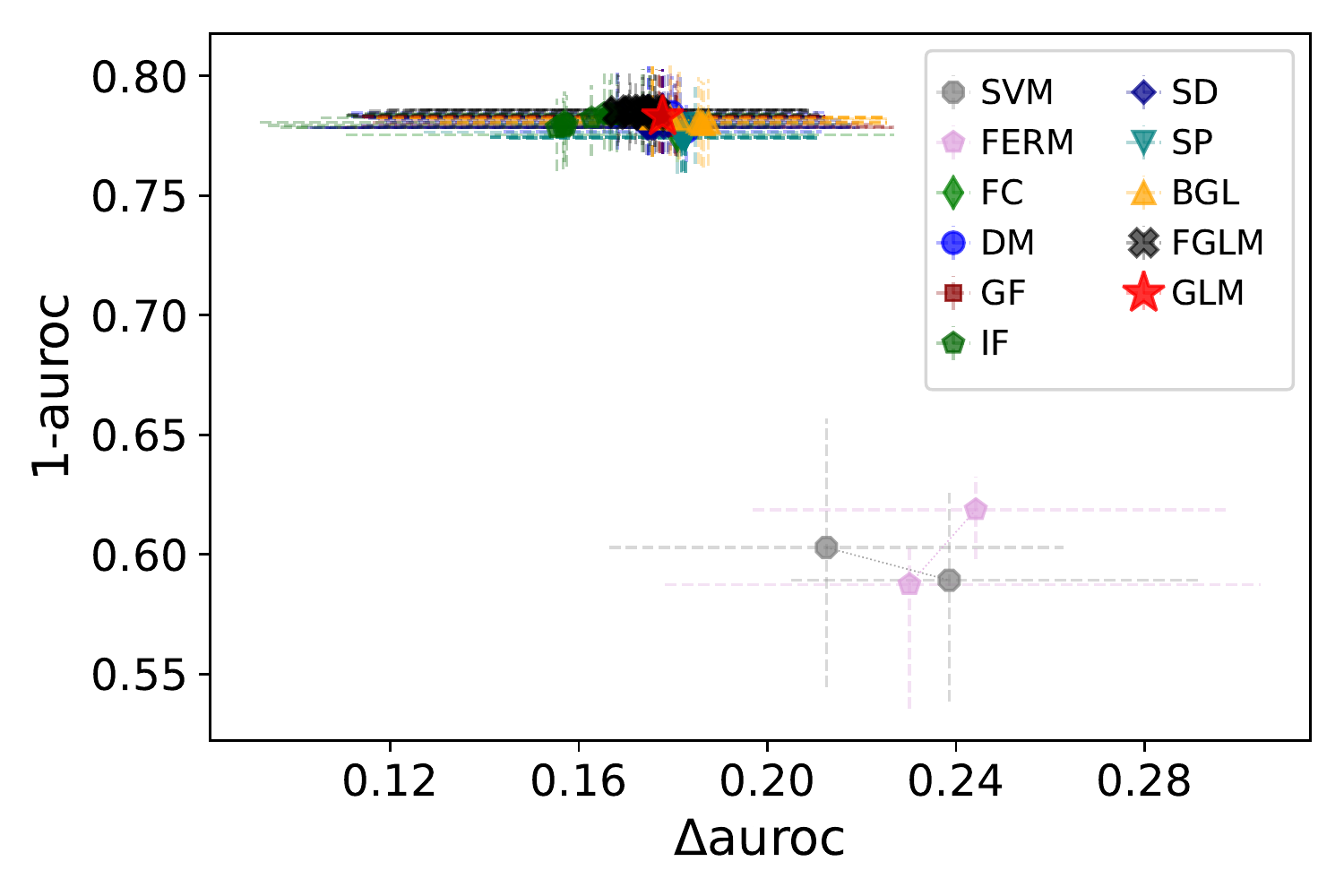}}
    \subfloat[Crime-Cont-Race(3)]{\includegraphics[width=0.32\linewidth,keepaspectratio]{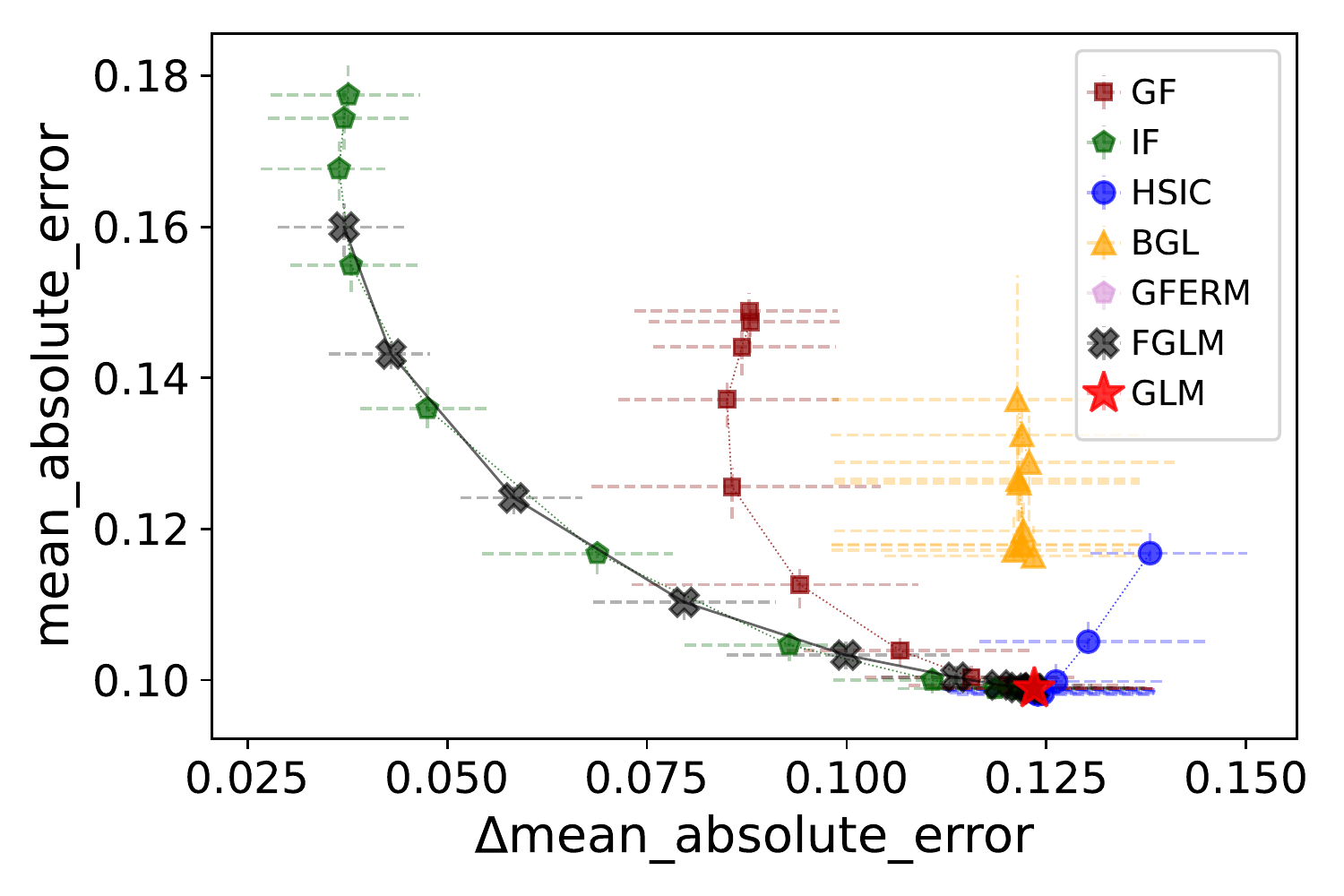}}
    \\
    \subfloat[LSAC--Cont--Race(5)]{\includegraphics[width=0.32\linewidth,keepaspectratio]{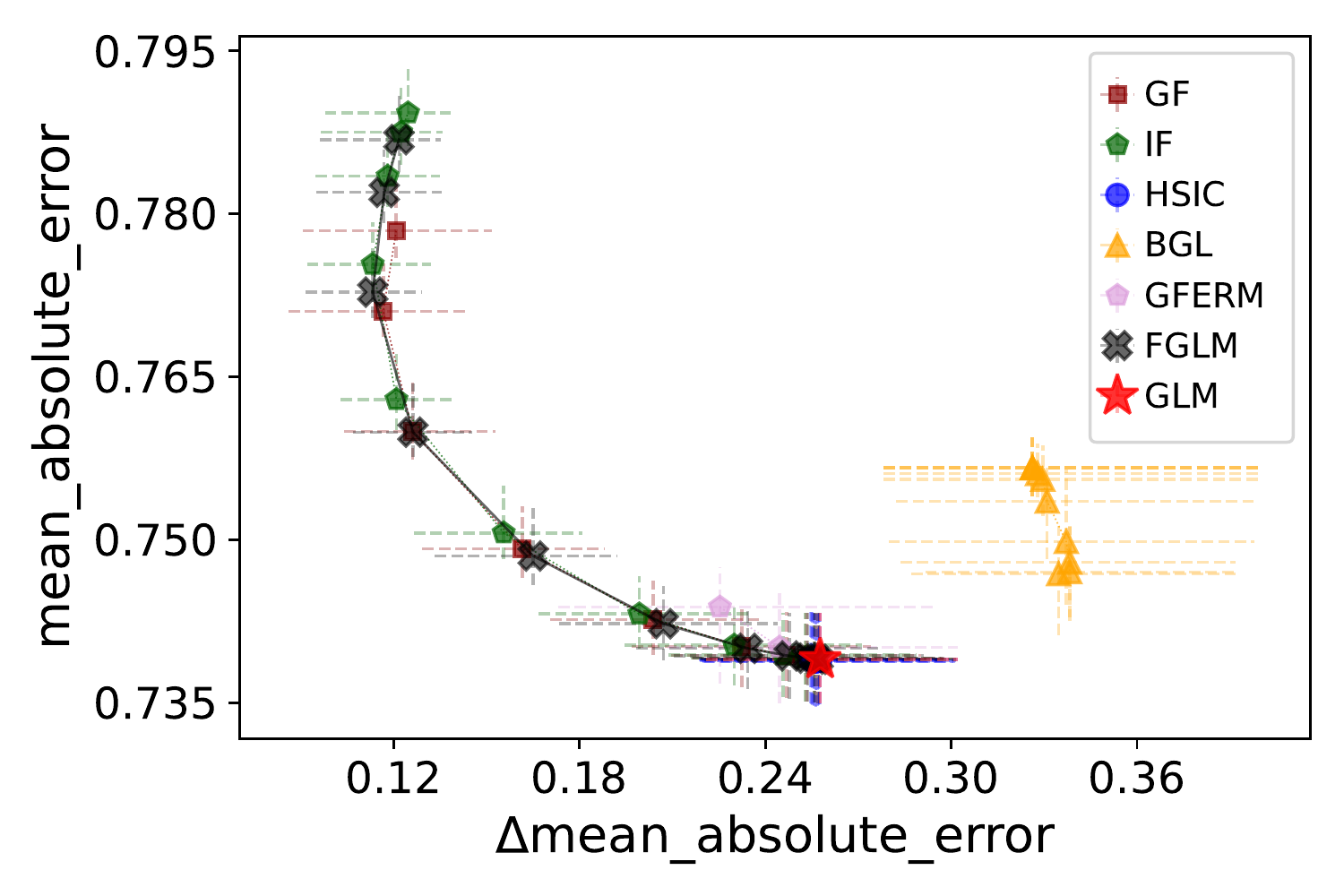}}
    \subfloat[Parkinsons--Cont--Gender(2)]{\includegraphics[width=0.32\linewidth,keepaspectratio]{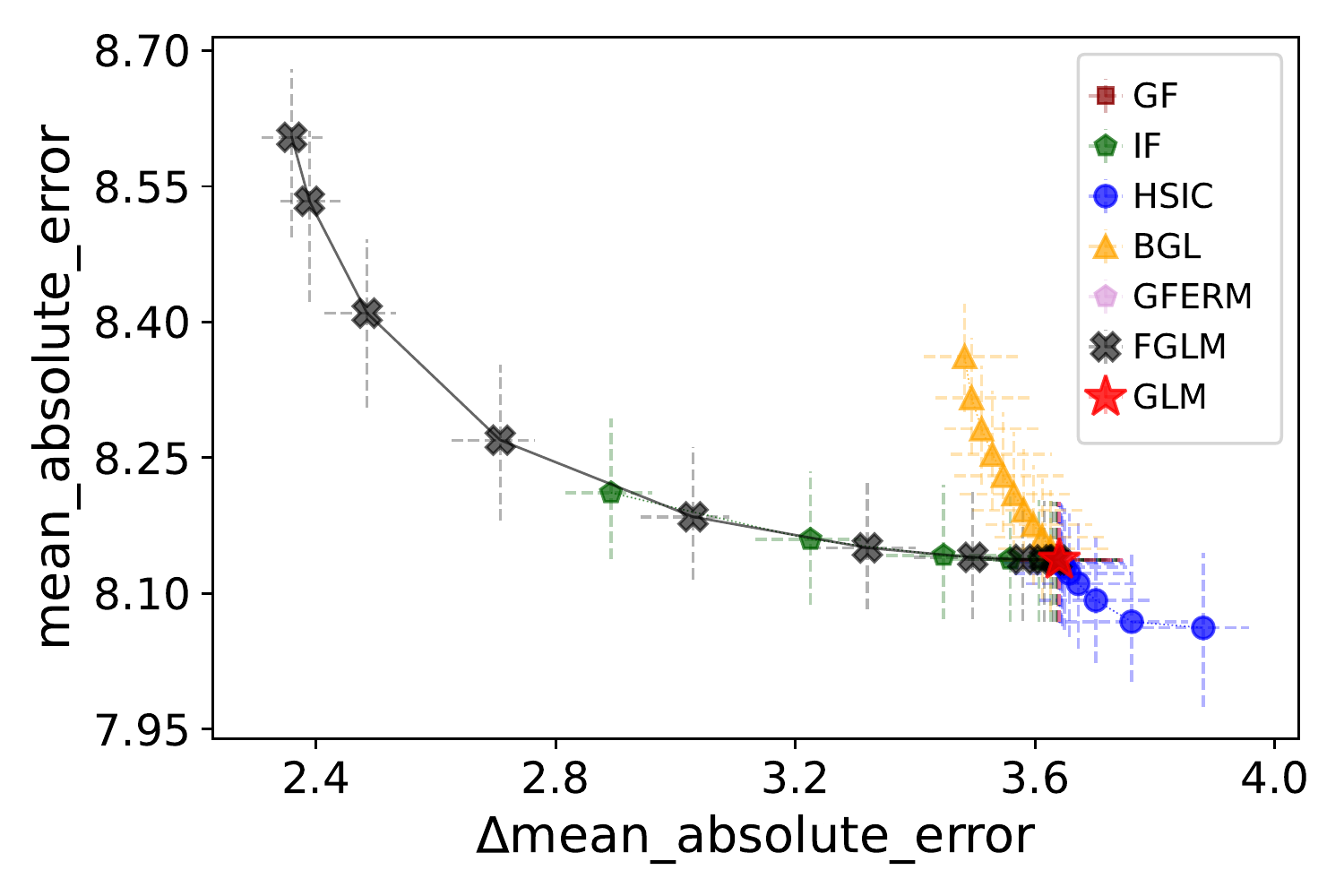}}
    \subfloat[Student--Cont--Gender(2)]{\includegraphics[width=0.32\linewidth,keepaspectratio]{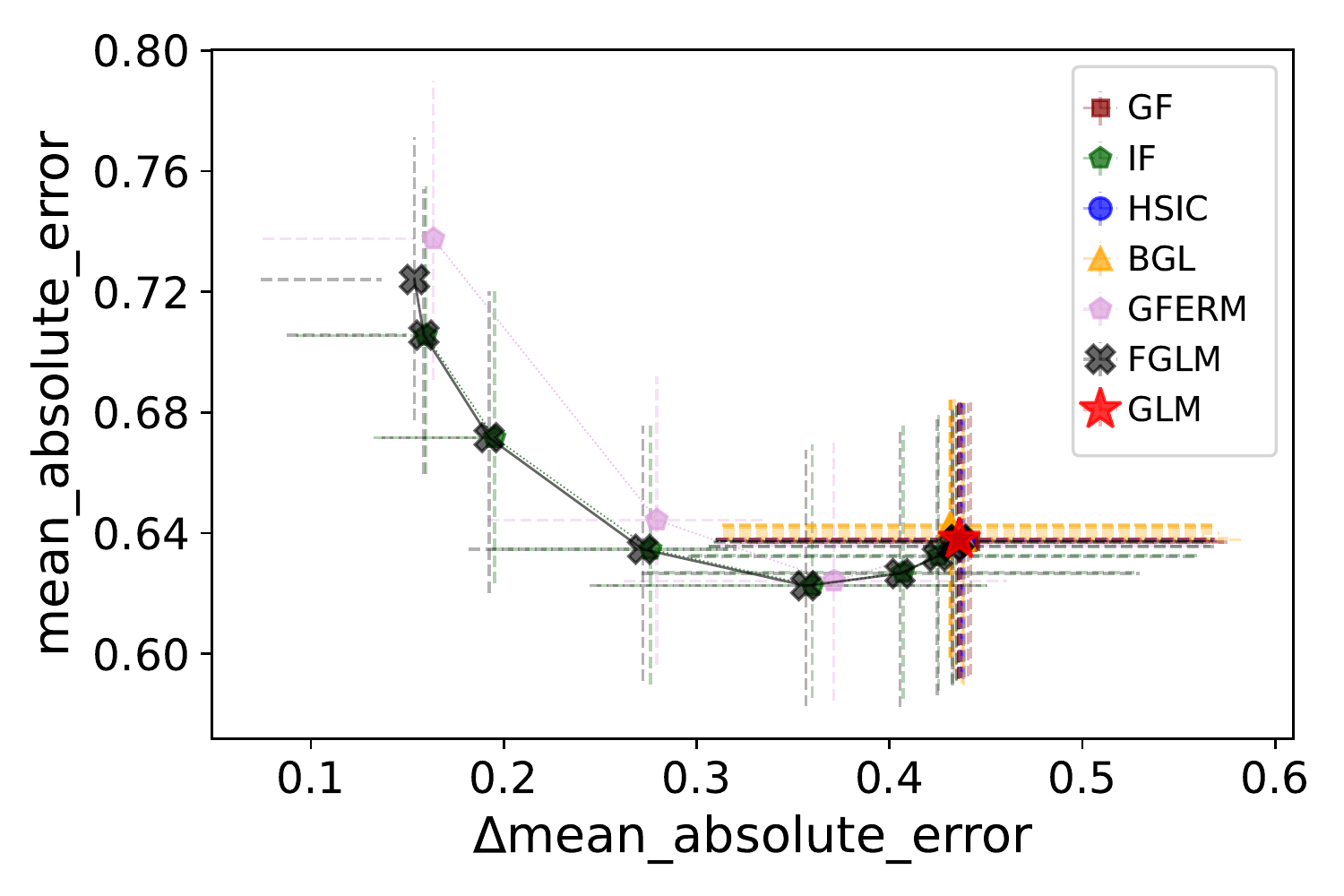}}
     \\
     \subfloat[Drug--Multi--Race(2)]{\includegraphics[width=0.32\linewidth,keepaspectratio]{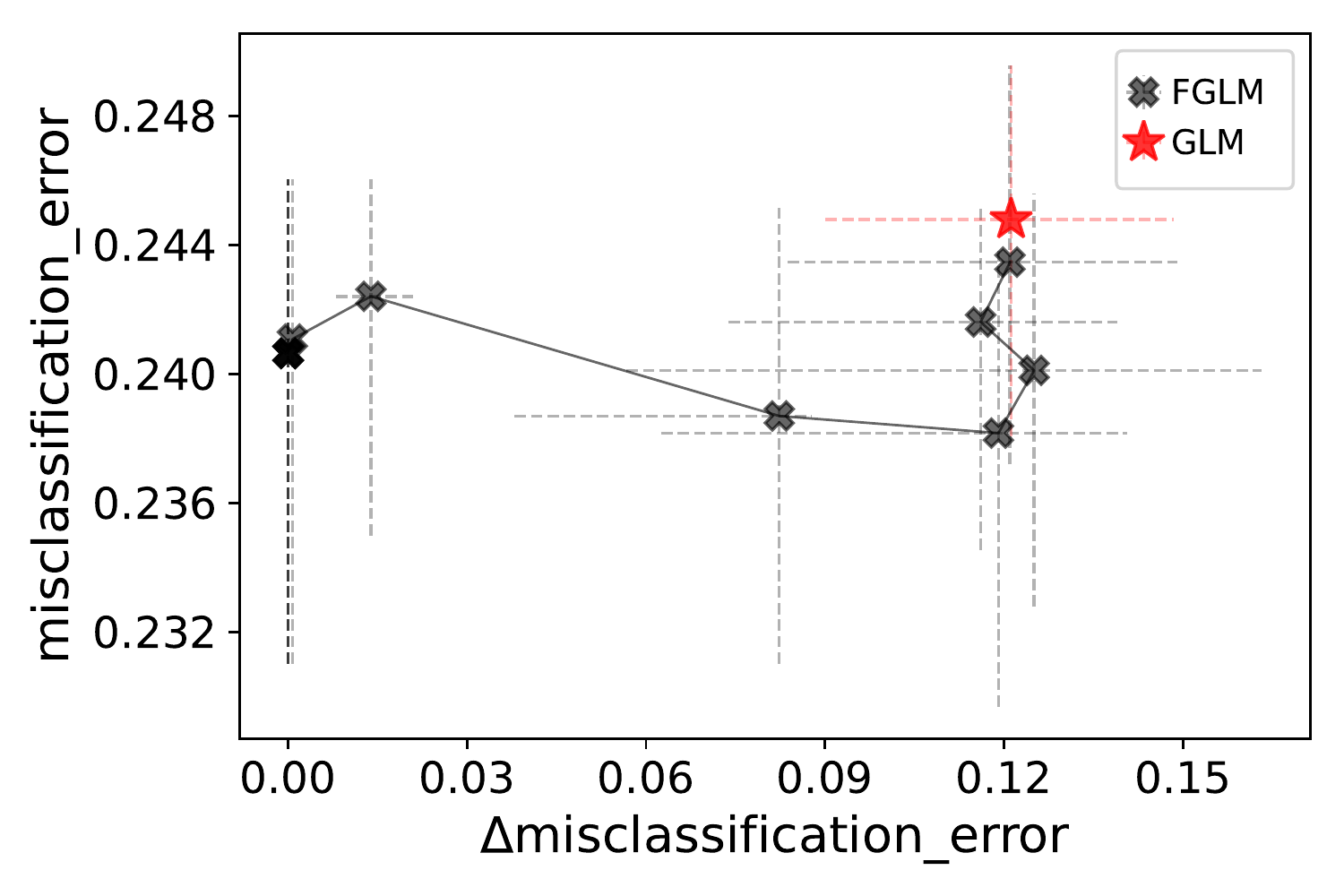}}
     \subfloat[Obesity--Multi--Gender(2)]{\includegraphics[width=0.32\linewidth,keepaspectratio]{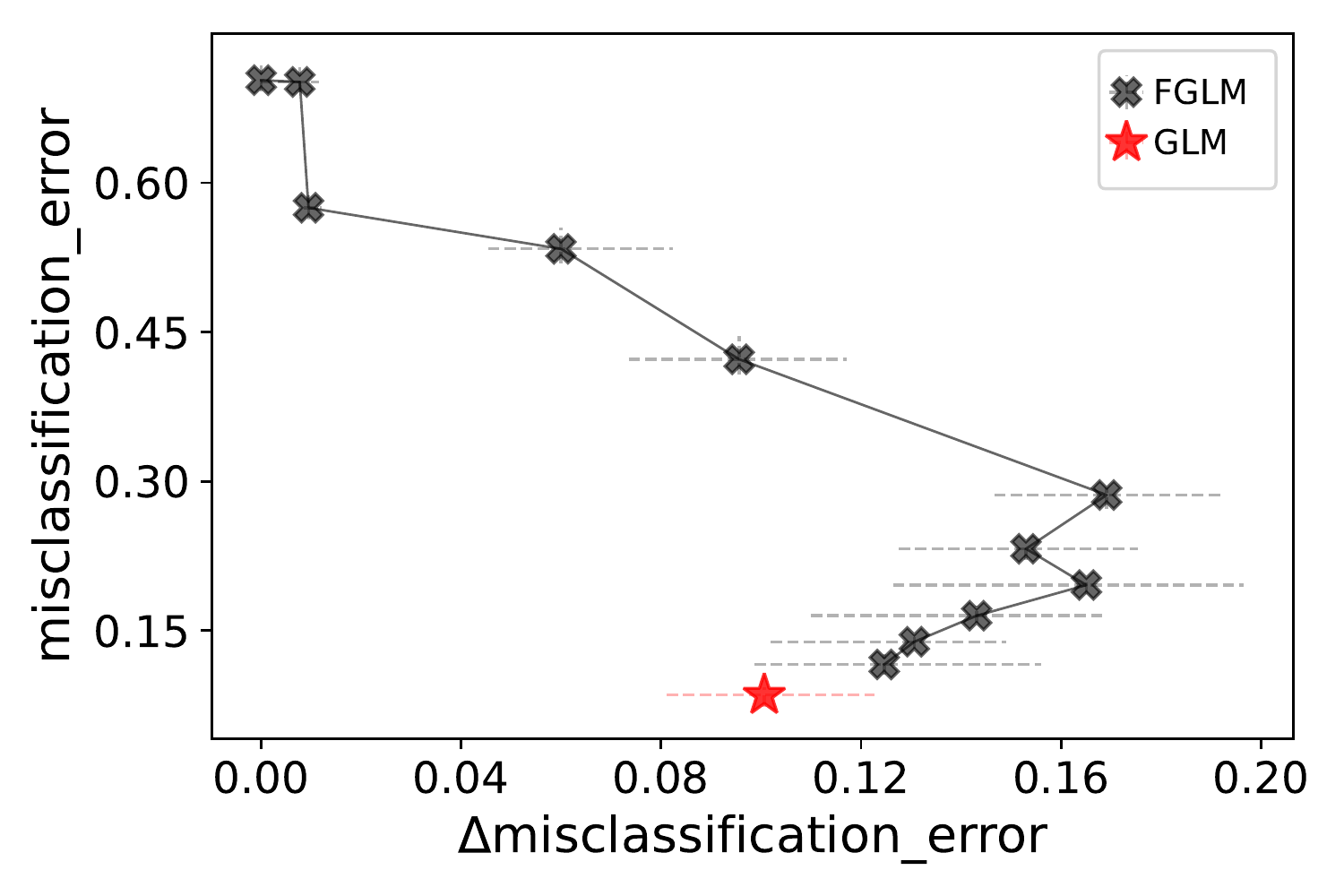}}
     \subfloat[HRS--Count--Race(4)]{\includegraphics[width=0.32\linewidth,keepaspectratio]{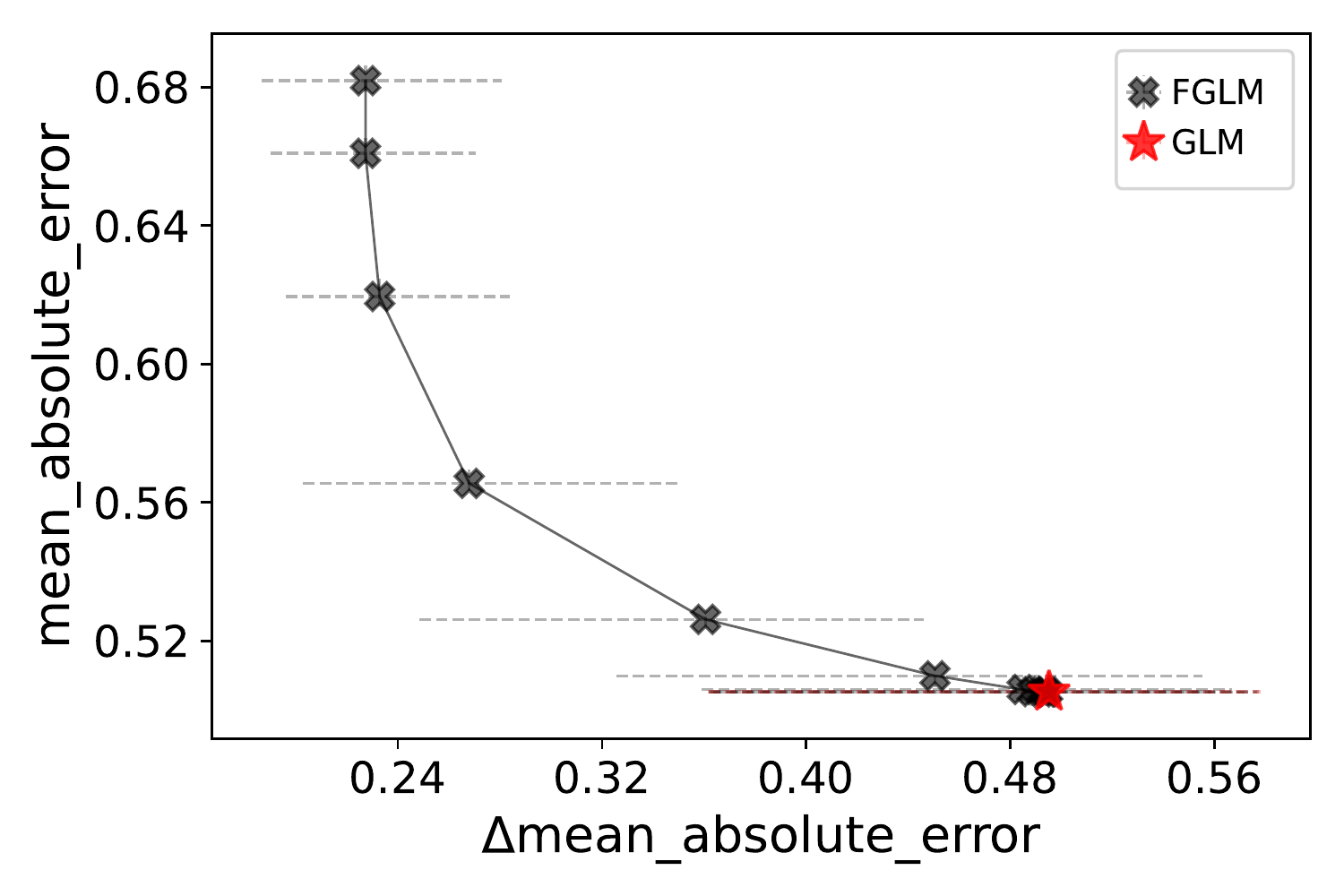}}
    \caption{Experimental results \emph{in some other metrics (AUROC/MAE/misclassification rate)} from 11 real world datasets with binary (a-e) and continuous outcomes (f-i). The $x$-axis is the disparity of each metric. Each subtitle is in the form of {Dataset}--{Outcome Type}--{Sensitive Attribute($K$)}. For both binary and continuous outcomes we use a Generalized Linear Model (GLM, red star \inlinegraphics{markers/GLM.png}), Fair Generalized Linear Model (F-GLM, black X \inlinegraphics{markers/FGLM.png}), Individual Fairness penalty (IF, green pentagon \inlinegraphics{markers/IF.png}), Group Fairness penalty (GF, dark red square \inlinegraphics{markers/GF.png}), and Bounded Group Loss (BGL, orange triangle \inlinegraphics{markers/BGL.png}). Methods for binary outcomes also include the Support Vector Machine (SVM, grey hexagon \inlinegraphics{markers/SVM.png}), Fair Constraints (FC, green diamond \inlinegraphics{markers/FC.png}), Disparate Mistreatment (DM, blue circle \inlinegraphics{markers/DM-HSIC.png}), Squared Difference penalizer (SD, dark blue diamond \inlinegraphics{markers/SD.png}), Fair Empirical Risk Minimization (FERM, plum pentagon \inlinegraphics{markers/GFERM.png}), Statistical Parity (SP, teal triangle \inlinegraphics{markers/SP.png}). Methods for continuous outcomes include the HSIC penalty (HSIC, blue circle \inlinegraphics{markers/DM-HSIC.png}), General Fair Empirical Risk Minimization (GFERM, plum pentagon \inlinegraphics{markers/GFERM.png}). See Table 1 for additional information for each method. Each dot represents mean performance across test sets  for a specific hyperparameter value $\lambda$.}
    \label{fig:results4}
\end{figure}

\bibliographystylesupp{icml2022}
\bibliographysupp{supp.bib}

\end{document}